%% file: main.tex
\documentclass[sigconf]{acmart}
\input{macros.tex}

\AtBeginDocument{%
  \providecommand\BibTeX{{%
    \normalfont B\kern-0.5em{\scshape i\kern-0.25em b}\kern-0.8em\TeX}}}

\setcopyright{none}
\renewcommand\footnotetextcopyrightpermission[1]{} % removes footnote with conference information in first column
\title{\sys: Data Cleaning for Conditional Independence Violations using Optimal Transport}
\renewcommand{\shortauthors}{Trovato and Tobin, et al.}

\author{Alireza Pirhadi}
\affiliation{%
  \institution{University of Western Ontario}
  \city{London}
  \country{Canada}
}
\email{apirhadi@uwo.ca}

\author{Mohammad Hossein Moslemi}
\affiliation{%
  \institution{University of Western Ontario}
  \city{London}
  \country{Canada}
}
\email{mmoslem3@uwo.ca}

\author{Alexander Cloninger}
\affiliation{%
  \institution{University of California San Diego}
  \city{San Diego}
  \state{CA}
  \country{USA}
}
\email{acloninger@ucsd.edu}

\author{Mostafa Milani}
\affiliation{%
  \institution{University of Western Ontario}
  \city{London}
  \country{Canada}
}
\email{mostafa.milani@uwo.ca}

\author{Babak Salimi}
\affiliation{%
  \institution{University of California San Diego}
  \city{San Diego}
  \state{CA}
  \country{USA}
}
\email{bsalimi@ucsd.edu}

\begin{document}

\pagenumbering{Roman} % Use uppercase Roman numerals

\newpage
\pagenumbering{arabic}

\input{abstract}

\maketitle

\input{intro}

\input{background}
\input{problem_definition}
\input{methods}

\input{optimizations}
\input{experiments}
\input{related_work}

\input{conclusion}

\bibliographystyle{ACM-Reference-Format}
\bibliography{ref}

\input{appx}

\end{document}

%% file: macros.tex
\usepackage{tikz}
\usetikzlibrary{positioning}

\usepackage{booktabs}  % For better horizontal rules
\usepackage{siunitx}   % For better number formatting
\usepackage{xcolor}    % For adding color
\usepackage{tikz}
\usetikzlibrary{positioning, shadows, decorations.pathreplacing, backgrounds}
\usepackage{booktabs}
\usepackage{siunitx}
\usepackage{colortbl}
\usetikzlibrary{positioning, arrows.meta}

\newcommand{\Tp}{\Pi}

\newcommand{\push}{\mathcal{T}}
\usepackage{tikz}
\usetikzlibrary{patterns}
\usepackage{amsmath}
\usepackage[multiple]{footmisc}
\usepackage{booktabs}
\usepackage{amsthm}
\usepackage{dsfont}

\usepackage[utf8]{inputenc}
\usepackage{balance}
\usepackage{siunitx}
\usepackage{multirow}
\usepackage{graphicx}
\usepackage{listings}
\usepackage{commath}
\usepackage{epstopdf}
\usepackage{color}
\usepackage{url}
\usepackage{caption}
\usepackage{alltt}
\usepackage{bbm}
\usepackage{mathrsfs}
\usepackage{float}
\usepackage{subcaption}
\usepackage{graphicx,fancyvrb}
\usepackage[linesnumbered,ruled,vlined]{algorithm2e}
\usepackage{algpseudocode}

\usepackage{mathtools}
\usepackage{capt-of}
\usepackage{colortbl}
\usepackage{xcolor}
\usepackage{xfrac}
\usepackage{footnote}
 \usepackage{enumitem}
\usepackage{array}
\usepackage{arydshln}
\usepackage{cleveref}
\definecolor{mygreen}{rgb}{0,0.6,0}
\definecolor{myred}{rgb}{0.6,0,0}
\definecolor{mygray}{rgb}{0.5,0.5,0.5}
\definecolor{mymauve}{rgb}{0.58,0,0.82}
\definecolor{myblue}{rgb}{0,0,1}

\usepackage{listings}
\lstnewenvironment{VerbatimText}[1][]{
    
    \lstset{fancyvrb=true,basicstyle=\footnotesize,captionpos=b,xleftmargin=2em,#1}
}{}

\PassOptionsToPackage{hyphens}{url}\usepackage{hyperref}

\newcommand{\Dom}{\mathcal}
\newcommand{\bDom}{\mathbfcal}

\DeclareMathAlphabet\mathbfcal{OMS}{cmsy}{b}{n}

\newcommand{\KLD}{D_{{\mathrm {KL}}}}

\DeclareMathOperator*{\argmin}{argmin}

\definecolor{ReviewerOneColor}{RGB}{240, 110, 0}   % Orange
\definecolor{ReviewerTwoColor}{RGB}{0, 139, 109}   % Dark Cyan
\definecolor{ReviewerThreeColor}{RGB}{45, 85, 205}
\definecolor{MetaReviewerColor}{RGB}{142, 69, 133} % Plum
\newcommand{\reviewerone}[1]{{#1}\vspace{2mm}}

\newcommand{\reviewerthree}[1]{{#1}\vspace{2mm}}

\newcommand{\nit}[1]{\textit{#1}}
\newcommand{\costColor}[1]{\purple{#1}}
\newcommand{\Simplex}[1]{\Delta_{#1}}
\newcommand{\prQ}{Q}
\newcommand{\coupleM}{\pi}
\newcommand{\pushT}{T}

\newcommand{\purple}[1]{\textcolor{purple}{#1}}
\newcommand{\sourceColor}[1]{\textcolor{black}{#1}}

\newcommand{\nComment}[1]{\Comment{{#1}}}
\newcommand{\Domain}[1]{\mathcal{{#1}}}
\newcommand{\pr}{P}

\newcommand{\st}[1]{\mb{#1}}
\newcommand{\vt}[1]{\textbf{#1}}
\newcommand{\mtx}[1]{\textbf{#1}}
\newcommand{\dom}[1]{\Dom #1}

\newcommand{\otmonge}{\ot_{\textit{Monge}}}
\newcommand{\ot}{\mathrm{OT}}

\newcommand{\KL}{\KLD}
\newcommand{\sinkhornCoeff}{\rho}
\newcommand{\car}{\texttt{Car}\xspace}

\newcommand{\boston}{\texttt{Boston}\xspace}

\newcommand{\adult}{\texttt{Adult}\xspace}
\newcommand{\compas}{\texttt{COMPAS}\xspace}

\newcommand{\relaxedOTCoeff}{\ensuremath{\lambda}}
\newcommand{\dist}{\textit{dist}}
\newcommand{\sys}{\textsc{OTClean}\xspace}
\newcommand{\mainAlg}{\textsc{FastOTClean}\xspace}

\newcommand{\mc}[1]{\mathcal{#1}}

\newcommand{\ignore}[1]{}

\definecolor{black}{rgb}{0,0,0}
\definecolor{grey}{rgb}{0.8,0.8,0.8}
\definecolor{red}{rgb}{1,0,0}
\definecolor{green}{rgb}{0,1,0}
\definecolor{darkgreen}{rgb}{0,0.5,0}
\definecolor{darkpurple}{rgb}{0.5,0,0.5}
\definecolor{darkdarkpurple}{rgb}{0.3,0,0.3}
\definecolor{blue}{rgb}{0,0,1}
\definecolor{shadegreen}{rgb}{0.95,1,0.95}
\definecolor{shadeblue}{rgb}{0.95,0.95,1}
\definecolor{shadered}{rgb}{1,0.85,0.85}
\definecolor{shadegrey}{rgb}{0.85,0.85,0.85}
\definecolor{oddRowGrey}{rgb}{0.80,0.80,0.80}
\definecolor{evenRowGrey}{rgb}{0.85,0.85,0.85}
\definecolor{lightpurple}{rgb}{0.88,1.0,1.0}

\newcommand{\babak}[1]{{\texttt{\color{red} Babak: [{#1}]}}}

\newcommand{\indep}{\mbox{$\perp\!\!\!\perp$}}

\newcommand{\db}{{D}}

\newcommand{\RNum}[1]{\uppercase\expandafter{\romannumeral #1\relax}}
\newcommand{\mb}[1]{{\mathbf{#1}}}
\newtheorem{defn}{Definition}[section]

\newtheorem{col}[defn]{Colollary}

\newcommand{\proj}[1]{{\Pi}}
\newcommand{\sel}[1]{{\sigma}}

\newcommand{\cut}[1]{}
\newcommand{\eat}[1]{}

\usepackage[multiple]{footmisc}

\SetKwInput{KwInput}{Input}
\SetKwInput{KwOutput}{Output}

%% file: abstract.tex
\begin{abstract}
Ensuring Conditional Independence (CI) constraints is pivotal for the development of fair and trustworthy machine learning models. In this paper, we introduce \sys, a framework that harnesses optimal transport theory for data repair under CI constraints. Optimal transport theory provides a rigorous framework for measuring the discrepancy between probability distributions, thereby ensuring control over data utility. We formulate the data repair problem concerning CIs as a Quadratically Constrained Linear Program (QCLP) and propose an alternating method for its solution. However, this approach faces scalability issues due to the computational cost associated with computing optimal transport distances, such as the Wasserstein distance. To overcome these scalability challenges, we reframe our problem as a regularized optimization problem, enabling us to develop an iterative algorithm inspired by Sinkhorn's matrix scaling algorithm, which efficiently addresses high-dimensional and large-scale data. Through extensive experiments, we demonstrate the efficacy and efficiency of our proposed methods, showcasing their practical utility in real-world data cleaning and preprocessing tasks. Furthermore, we provide comparisons with traditional approaches, highlighting the superiority of our techniques in terms of preserving data utility while ensuring adherence to the desired CI constraints.
\end{abstract}

%% file: intro.tex
\section{Introduction} \label{sec:intro}

{\em Conditional Independence (CI)} plays a pivotal role in probability and statistics. At its core, a CI statement, represented as $(X \indep Y \mid Z)$, implies that when $Z$ is known, the knowledge of $X$ doesn't provide any further insight into $Y$, and vice versa. To illustrate, consider rainfall ($Z$) influencing both the wetness of grass ($X$) and the decision to use an umbrella ($Y$). If we're already aware that it rained, then determining that the grass is wet doesn't shed any additional light on a person's choice to carry an umbrella.
CI is foundational in numerous areas. It underpins causal reasoning and graphical models, serving as a cornerstone for efficient probabilistic inference~\cite{koller2009probabilistic,pearl2009causal}. In the realm of machine learning (ML), CI's significance spans across feature selection~\cite{koller1996toward}, algorithmic fairness~\cite{bozdag2013bias, torralba2011unbiased, hooker2021moving, hajian2016algorithmic, salimi2019interventional}, representation learning~\cite{pogodin2022efficient}, model interpretability~\cite{guyon2003introduction, arjovsky2019invariant, galhotra2021explaining}, transfer learning~\cite{rojas2018invariant}, and domain adaptation~\cite{magliacane2018domain}.

{\em Conditional Independence (CI)} in statistics can be analogized with integrity constraints in databases~\cite{wong2000implication}. Specifically, in the context of databases, dependencies such as Functional Dependencies (FDs), Conditional Functional Dependencies (CFDs), and Multivalued Dependencies (MVDs) encapsulate critical semantic and structural constraints. These constraints are imperative for maintaining data integrity in relational databases and play a pivotal role in tasks like data quality management and data cleaning~\cite{bertossi2006consistent,bohannon2006conditional,fan2022foundations}. In a parallel vein, CI represents key statistical constraints that are indispensable for ensuring the robustness and validity of datasets in domains like ML and statistical inference. To elucidate this analogy further, consider the following example.

\begin{example} \label{ex:intro} \em In this example, we underscore the significance of  maintaining and enforcing CI constraints in data pipelines as essential steps in constructing fair and and reliable ML models, illustrated within the contexts of medical diagnosis and job applications.
\vspace{-0.2cm}
\paragraph{\bf Medical Diagnosis} Consider a dataset used for predicting patient recovery from respiratory infections, consists of attributes such as patient demographics, including their ZIP code, health measurements, the bacterial {strain} causing the infection, the prescribed {antibiotic}, and the {recovery outcome}. Based on domain knowledge, one would expect that the {patient's ZIP code} should be independent of the {recovery outcome} given all causal factors that affect the patient's recovery, i.e., $(\text{ZIP code} \ \indep \ \text{Recovery} \mid \text{Causal factors})$. However, existing biases, such as certain ZIP codes having better healthcare access or particular residents' health behaviors, can introduce spurious associations. Additionally, data quality issues, including incorrect ZIP code entries or inaccurately recorded recovery outcomes, or even systematic data quality issues on other attributes that are distributed non-randomly for patients with different ZIP codes, can also violate this expected independence. Training a model on this dataset may lead to a model that picks up spurious correlations between recovery outcomes and ZIP codes rather than the actual causal factors, affecting the model's performance during deployment. Furthermore, simply dropping ZIP code and not using it for training ML models does not resolve the issue if the constraint is violated due to data quality issues on the selected features. In that case, the performance of the model during deployment becomes different for different subpopulations with different ZIP codes, leading to potential geographic biases.

\vspace{-0.2cm}
\paragraph{\bf Job Application} Consider a dataset used for making hiring decisions. This dataset consists of attributes from applicants' CVs and insights from interviews, encompassing variables such as hobby, hometown, previous companies worked at, university attended, project experiences, and other qualifications. In an ideal scenario, factors considered extraneous, like hobby, university attended, and hometown, should be independent of the hiring decision when conditioned on the applicant's qualifications, i.e., ($ \text{Extraneous Factors}  \ \indep \ \text{Hiring Decision} \mid \text{Qualifications}$) However, this constraint can be violated in the dataset due to various reasons. Biases may emerge if, for example, a significant proportion of successful candidates in the dataset share hobbies perceived as technical or come from specific renowned hometowns. Data quality issues, such as inconsistent categorization of qualifications or historical biases in hiring practices, further compound the issue. These extraneous factors not only divert the model's focus from genuine qualifications but can also inadvertently introduce biases. When these factors correlate with sensitive attributes, such as race and gender, the resulting model may become profoundly unfair.
\end{example}

%\paragraph{\bf Job Application} Consider a dataset used for making hiring decisions. This dataset consists of attributes from applicants' CVs and insights from interviews, encompassing variables such as hobbies, hometown; previous companies worked at, university attended, project experiences, and other qualifications. In an ideal scenario, factors considered extraneous, like a hobby, university attended, and hometown, should be independent of the hiring decision when conditioned on the applicant's qualifications, i.e., ($ \text{Extraneous Factors}  \ \indep \ \text{Hiring Decision} \mid \text{Qualifications}$) However, this constraint can be violated in the dataset due to various reasons. Biases may emerge if, for example, a significant proportion of successful candidates in the dataset share hobbies perceived as technical or come from specific renowned hometowns. Data quality issues, such as inconsistent categorization of qualifications or historical biases in hiring practices, further compound the issue. These extraneous factors not only divert the model's focus from genuine qualifications but can also inadvertently introduce biases. When these factors correlate with sensitive attributes, such as race and gender, the resulting model may become profoundly unfair.
%\end{example}

In this paper, \textbf{we address the problem of repairing a dataset with respect to CI constraints.} Given a dataset that violates a CI constraint due to data biases and data quality issues, our goal is to clean the data to ensure adherence to CI constraints while preserving data utility. Much research has been dedicated to computing optimal repairs for data dependencies, particularly functional dependencies and conditional functional dependencies~\cite{livshits2020computing, kolahi2009approximating, bohannon2006conditional, DBLP:conf/icdt/LivshitsK21}. However, the challenge of repairs concerning CI remains relatively unexplored. A significant contribution in this area is the work by Salimi et al.~\cite{salimi2019interventional}. Their study links CI to Multi-valued dependencies (MVDs) and provides methods to compute optimal repairs by {\bf minimizing the number of tuple deletion and insertion} to ensure consistency with an MVD~\cite{salimi2019interventional}.

A significant challenge in data cleaning for ML is how to ensure that these operations do not distort the inherent statistical properties of datasets and preserve data utility. This challenge becomes especially more noticeable when considering that, in this context, {\bf the significance of individual data tuples is secondary to the underlying distribution they collectively represent}~\cite{dasu2012statistical}. Achieving the goal of preserving these statistical properties requires a method to quantify the distance between the distributions of the original and repaired data. Traditional criteria in databases, such as subset minimality and minimum cardinality repair, often fall short in effectively addressing this requirement~\cite{bertossi2006consistent}. While various methods exist for measuring the distance between probability distributions, including information theoretic measures like Kullback-Leibler (KL) and Jensen-Shannon (JS) divergences~\cite{cover1999elements}, {\bf Optimal Transport (OT) metrics, such as the Wasserstein (or Earth Mover's) distance, have demonstrated their superiority in various ML tasks}~\cite{arjovsky2017wasserstein, frogner2015learning}. 

OT provides a metric for comparing probability distributions by determining the most efficient way to convert one distribution into another. This transformation is facilitated through the use of a  {\bf transport plan}, which is a probabilistic mapping that specifies how much mass is moved from each data point in one distribution to its corresponding point in the second distribution. This mapping is optimized according to a designated cost function. One distinctive feature of OT is its capability to transform a domain-specific metric between individual data points into a comprehensive metric between entire distributions~\cite{arjovsky2017wasserstein}. This adaptability empowers OT to preserve the topological and structural properties of the data that cannot be captured and maintained using other divergences and distances between distributions.

In our paper, we introduce \sys, a novel framework that leverages OT theory for data cleaning to enforce CI constraints. \sys addresses datasets that violate CI constraints by learning a {\em probabilistic data cleaner}. This cleaner probabilistically updates attribute values to ensure adherence to CI constraints. It finds an optimal repair, aiming to satisfy the CI constraint while minimizing the OT distance from the original dataset, which indicates minimal alteration to the data. This approach is versatile, allowing for user-defined metrics to tailor cleaning to specific needs and preserving data integrity, which is crucial for subsequent applications. Additionally, \sys's probabilistic mapping operates at the tuple level, making it well-suited for streaming environments and scenarios that require model retraining on newly acquired data.
 %It can directly repair new data points, avoiding the need for comprehensive dataset cleaning.
%In this paper, we present \sys, {\bf a novel framework that utilizes OT theory for cleaning data to enforce CI constraints}. Given a dataset that violates a CI constraint and a user-defined metric for quantifying the cost of modifying attribute values for individual data points, \sys learns a {\bf probabilistic data cleaner} to ensure CI constraint adherence. This cleaner functions by probabilistically updating the attribute values of each data point, transforming the initially flawed dataset into a cleaned one that satisfies the constraint. Remarkably, the resulting cleaned dataset represents an optimal solution: among all potential dataset that meet the CI constraint, it possesses the minimum OT distance when compared to the original database, signifying the lowest cost incurred during the data cleaning process. By incorporating user-defined metrics, \sys offers a versatile and adaptable data cleaning approach tailored to specific tasks while preserving the semantic information and structural integrity of the data—essential for downstream applications. \sys provides us with a probabilistic mapping capable of operating at the tuple level, making it well-suited for deployment in streaming environments and scenarios where models are retrained on new data by directly repairing new data points, eliminating the need to clean the entire dataset at once.

A primary hurdle in employing OT in ML is its considerable computational cost. Specifically, for discrete data, OT necessitates solving a linear program. Techniques like the network simplex or interior point methods are frequently applied, but their computational intensity is significant for high-dimensional data. In fact, their cost scales as \(O(d^3 \log(d))\) when comparing histograms of dimension \(d\)~\cite{pele2009fast}. {\bf We demonstrate that using OT, the problem of repairing data under CI constraints can be formulated as a Quadratically Constrained Linear Program (QCLP)}~\cite{van1966programming,boyd2004convex}. Although this problem can be tackled using established optimization techniques, it is important to note that solving a QCLP is generally NP-hard, presenting challenges in terms of scalability and computational feasibility for high-dimensional datasets.

To address the scalability challenges, we propose the use of approximate algorithms for solving our repair problem efficiently. At the core of our approach is the Sinkhorn distance~\cite{cuturi2013sinkhorn}, an approximate OT metric that introduces entropy regularization, penalizing transport plans based on their entropy. This regularization intuitively smoothens the OT problem, making it more manageable. Importantly, it allows us to leverage Sinkhorn's matrix scaling algorithm~\cite{sinkhorn1964relationship}, which operates at speeds several orders of magnitude faster than conventional methods. Expanding on this, {\bf we formulate our repair problem as a regularized optimization problem that employs a relaxed version of OT along with entropic regularization}. This optimization problem remains non-convex; however, {\bf we have developed an alternating algorithm with guaranteed convergence.} Remarkably, our approach exhibits a substantial improvement in efficiency compared to the QCLP formulation, making it scalable to high-dimensional data.

To assess the effectiveness of our approach, we apply it to two distinct domains: algorithmic fairness~\cite{salimi2019interventional}, where CI constraints play a crucial role, and data cleaning, where the utilization of CI as a statistical constraint has proven to be beneficial~\cite{yan2020scoded}. Our experiments reveal that our techniques outperform the current state-of-the-art database repair methods that involve CI~\cite{salimi2019interventional}. In the realm of algorithmic fairness, {\bf our approach not only yields fairer algorithms but also maintains superior performance compared to baseline methods}. As for data cleaning, our findings demonstrate that {\bf enforcing CI constraints results in more accurate data representations, thereby helping prevent ML models from relying on spurious correlations.}  Furthermore, we have shown that our methods can complement existing data cleaning techniques and address their limitations by effectively removing spurious correlations.

%% file: background.tex
\section{Background} \label{sec:background}

The notation used is summarized in Table~\ref{tab:symbols}. We use uppercase letters ($X$, $Y$, $Z$, $V$) to denote variables and lowercase letters ($x$, $y$, $z$, $v$) to represent their potential values. When referring to sets of variables or values, we use boldface notation ($\mathbf{X}$ or $\mathbf{x}$). The {\em support} or {\em domain} of a variable $\mathbf{V}$ is given by $\mathcal{V}$. We use $d_{\dom V}$ to refer to $|{\Dom V}|$, i.e., the size of ${\Dom V}$'s support. For any discrete random variable $X$, its probability distribution is represented by $P_X(x)$; in some contexts, we might simply use $P$, indicating the probability of $X$ assuming the value $x$. It's essential to note that such a probability distribution $P$ can be equivalently seen as a point in the {\em probability Simplex} $\Simplex{\mb V} = \{ \mb X \in \mathbb{R}^{d_{\Dom V}} \mid \forall v \in \mathcal{V}, \mb X_{v} \geq 0 \text{ and } \sum_{v \in \mathcal{V}} \mb X_{v} = 1 \}$, where, $\mb X_v$ is the probability assigned to value $v$. Intuitively, $\Simplex{\mb V}$  defines the set of all possible probability distributions over the finite domain $\mathcal{V}$.

Given a probability distribution $\pr \in \Simplex{\st{V}}$ over a set of variables $\mb V$, and considering non-empty and disjoint subsets $\mb X, \mb Y, \mb Z$ within $\mb V$, the distribution $P$ is said to be {\em consistent} with a {\em conditional independence (CI) constraint} $(\sigma: \mb Y \indep \mb X \mid \mb Z)$, denoted as $\pr \models \sigma$, if and only if, for all values $x \in \mc{X}$, $y \in \mc{Y}$, and $z \in \mc{Z}$, the condition 
$\pr_{\mb X,\mb Y \mid \mb Z}(x,y \mid z) = \pr_{\mb X\mid \mb Z}(x \mid z) \cdot \pr_{\mb Y \mid \mb Z}(y \mid z)
$ is satisfied. If the entire set $\st{V}$ is precisely the union of the subsets $\mb X, \mb Y,$ and $\mb Z$, i.e., $\st{V} = \mb X \cup \mb Y \cup \mb Z$, then the constraint $\sigma$ is termed as {\em saturated}.

When $\pr$ is {\em inconsistent} with the constraint $\sigma: Y \indep X \mid Z$, the {\em degree of inconsistency} of $\pr$, denoted $\delta_\sigma(\pr)$, 
can be quantified using the {\em conditional mutual information (CMI)}, denoted as $I(X; Y \mid Z)$, which measures the amount of information about $Y$ obtained by knowing $X$, given $Z$. Formally, 

%\vspace{-0.2cm}
{
\begin{align}
   I(X; Y \mid Z) &= \sum_{x \in \mc{X}, y \in \mc{Y}, z \in \mc{Z}} \pr(x,y,z) \log \left( \frac{\pr_{X,Y \mid Z}(x,y \mid z)}{\pr_{X\mid Z}(x \mid z) \pr_{Y \mid Z}(y \mid z)} \right) \nonumber \\ & = \KLD[\pr(X,Y,Z)\mid \pr(X, Z)\pr(Y\mid Z)]  \nonumber 
\end{align}
}
%\vspace{-0.3cm}

\noindent where $\KLD$ is the Kullback–Leibler divergence\footnote{The Kullback–Leibler divergence between two distribution $Q(X)$ and $P(X)$ is defined as: 
$ D_{\mathrm{KL}}(P\parallel Q) = \sum_{x \in \mathcal{X}} P(x) \log \left( \frac{P(x)}{Q(x)} \right) $.
}.

\noindent The probability distribution $\pr$ is consistent with the constraint  $\sigma: Y \indep X \mid Z$ if and only if $I(X; Y \mid Z) = 0$.

Given a dataset $\db = \{ \mb v_i \}_{i=1}^n $ consisting of i.i.d. samples drawn from a distribution $\pr \in \Simplex{\st{V}} $, each sample $\mb v_i $ corresponds to an element in the domain $\mc V $. The \textit{empirical distribution} $\pr^{\db} $ of the dataset $\db $ is defined as: $
\pr_{\mb V}^{\db}(\mb v) = \frac{1}{n} \sum_{i=1}^{n} \mathbb{I}(\mb v_i = \mb v),
$
where $\mathbb{I} $ is the indicator function that returns 1 if its argument is true and 0 otherwise. For each value $\mb v $ in the domain $\mc V $, $\pr_{\mb V}^{\db}(\mb v) $ computes the fraction of times $\mb v $ appears in the dataset $\db $. This empirical distribution provides an estimate of the true underlying distribution $\pr $ from which the samples in $\db $ were drawn. Given a conditional independence constraint $\sigma: \mb Y \indep \mb X \mid \mb Z $, we say $\db $ is consistent with $\sigma $ if the empirical distribution $\pr^{\db} $ associated with $\db $ is consistent with it. This is also denoted as $D \models \sigma $.

\begin{table}
  \centering
  \caption{Summary of notation and symbols.}
  \vspace{-3mm}
  \label{tab:symbols}
  \setlength\tabcolsep{2 pt}
 \begin{tabular}{l l}
  \toprule
    Symbol              & Description \\
    \midrule
    $X,Y,Z,V$   & Variables\\
    $\mb X, \mb Y, \mb Z, \mb V$  & Sets of variables \\ 
    $\Dom X$ & Domain of a variable $X$\\
    $d_{\Dom X}$ & Size of the domain of a variable $X$\\ 
    $x \in  \Dom X $  & Their values \\      
    $\pr$ & Probability distributions\\
    $\Simplex{\mb V}$ & A probability simplex over a domain of variables $\mb V$  \\
    $\mb p \in \Simplex{\mb V}$ & A probability vector \\
    $\coupleM$ & Transport plan\\
    $\sigma: (X \indep Y \mid Z)$ & A CI constraint\\
    $\delta_\sigma(\pr)$ & Degree of inconsistency of $\pr$ to a CI constraint $\sigma$  \\
    $c,\mtx{C}$ & Cost function and cost matrix\\
    \bottomrule
    \end{tabular}
\end{table}

\subsection{Background on Optimal Transport}\label{sec:ot}

This section provides an overview of optimal transport, serving as the foundational theory for \sys. We further delve into Sinkhorn regularization and the concept of relaxed optimal transport, which underpin the approximate repair methods introduced in Section~\ref{sec:BCDEOT}.
\paragraph{\bf Monge problem:} {\em The Optimal Transport (OT) problem} seeks the most efficient way of transferring mass from a probability distribution $\pr$ to another while preserving the total mass. The OT problem's classical formulation is {\em the Monge problem} where the objective is to identify a {\em transport map} $\pushT$ that pushes a distribution $\pr\in\Simplex{\st{X}}$ forward to a distribution $\prQ\in\Simplex{\st{Y}}$ while minimizing the total cost of transporting mass. Formally, $\prQ$, known as {\em the pushforward} of $\pr$ under the transport map $\pushT$, is a new distribution defined as $\prQ(A) = \pr(\pushT^{-1}(A))$ for any $A \subseteq \mathcal{Y}$. In other words, the pushforward $\prQ$ characterizes the distribution of the images of $\pr$ under the map $\pushT$. The Monge problem can be formally defined as follows: Given two distributions $\pr$ and $\prQ$ with discrete supports $\Domain{X}$ and $\Domain{Y}$, respectively, and a cost function $c: \mathcal{X} \times \mathcal{Y} \rightarrow \mathbb{R}_{\geq 0}$, the goal is to find a transport map $\pushT: \mathcal{X} \rightarrow \mathcal{Y}$ that pushes forward $\pr$ to $\prQ$, such that the total cost of transporting mass is minimized, i.e.,

{
\begin{equation}
\otmonge(\pr, \prQ) = \argmin_{\pushT: \mathcal{X} \rightarrow \mathcal{Y}} \sum_{\vt{x}_i \in \mathcal{X}} c(\vt{x}_i, \pushT(\vt{x}_i)),
\end{equation}
}

where $\pushT$ is a transport map and $\pushT_{\#}\pr = \prQ$.

%\babak{Intuition behind coupling.}

\paragraph{{\bf Kantorovich Formulation}} \reviewerone{The deterministic transport approach in Monge's problem might not always admit a solution. Specifically, there may be cases where finding a pushforward between two distinct probability distributions is not feasible. {To overcome this limitation, Kantorovich introduced a more flexible formulation by considering probabilistic transport methods. Unlike the deterministic approach, which requires a direct one-to-one mapping between elements, probabilistic transport allows for a more versatile mapping where elements from one distribution can be mapped to multiple elements in another distribution, reflecting real-world scenarios where such distributions cannot always be perfectly aligned.} This approach is operationalized through the concept of \emph{transport plans} or \emph{couplings}. Here, a coupling {refers to} a joint distribution, denoted as \( \coupleM \), over the product space \( \Domain{X} \times \Domain{Y} \). This coupling ensures that its marginals match the given distributions \( \pr \) and \( \prQ \), meaning \( \pr = \coupleM(\st{X}) \) and \( \prQ = \coupleM(\st{Y}) \). Denote \( \Pi(\pr, \prQ) \) as the space of all possible couplings. {In this context,} the \emph{primal Kantorovich formulation} of the OT problem is defined as follows:}

\vspace{-0.2cm}
{
\begin{equation}
\ot(\pr, \prQ) = \argmin_{\coupleM \in \Pi(\pr, \prQ)} \sum_{\vt{x}_i \in \mathcal{X}}\sum_{\vt{y}_j \in \mathcal{Y}} c(\vt{x}_i, \vt{y}_j) \coupleM(\vt{x}_i, \vt{y}_j).
\label{eq:otprimal}
\end{equation} 
}

\noindent \reviewerone{The goal of the OT plan \(\coupleM\) is to minimize the overall transport cost, as expressed in Equation~\ref{eq:otprimal}, while adhering to the probabilistic nature of the transport. When the cost \(c\) represents the Euclidean distance, the OT distance is recognized as the \emph{Wasserstein distance}.}

\begin{figure}[h]
\centering
\includegraphics[width=1.0\linewidth]{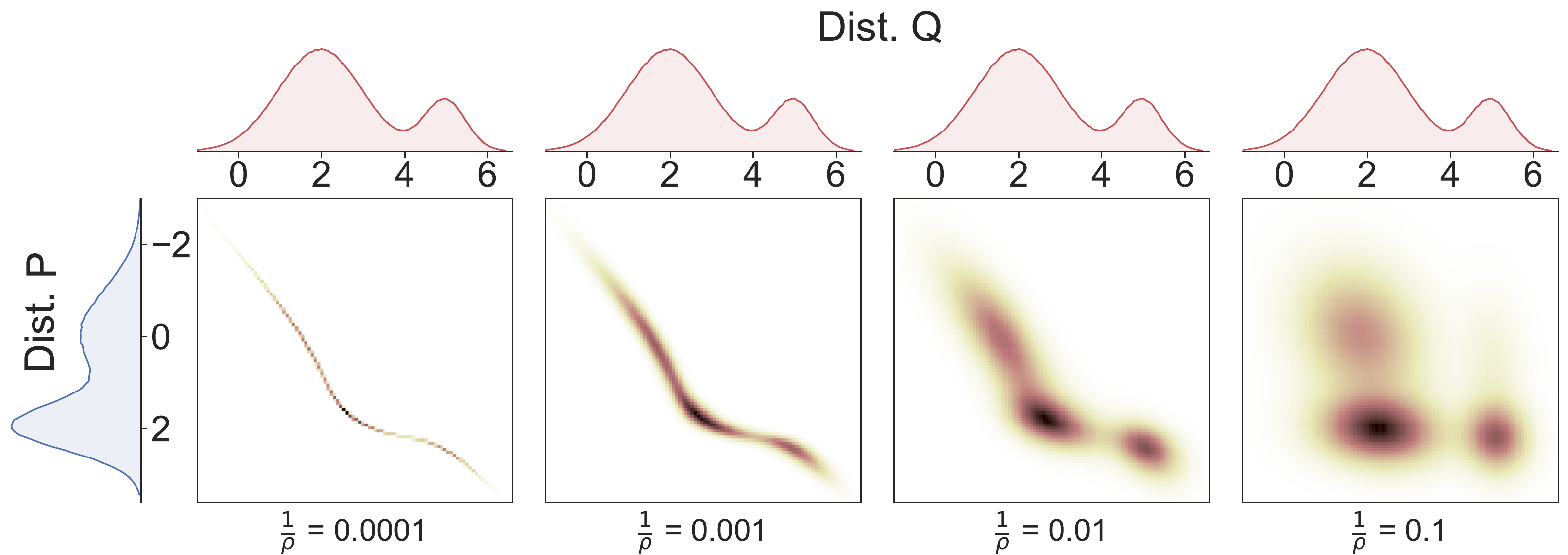}
    \caption{\reviewerone{The coefficient $1/\rho$ in regularized OT impacts the mapping between distributions $P$ and $Q$: higher coefficients (on the right) lead to smoother mappings and spread mass more evenly between $P$ and $Q$.}}
    \label{fig:sinkhorn}
\end{figure}

\paragraph{\bf Entropic Regularization:} \reviewerone{OT problems, as described by Equation~\ref{eq:otprimal}, essentially involve solving a linear program. The computational complexity of solving such a linear program \(O(n^3 \log n)\) using the network simplex, where \(n\) represents the number of variables or constraints~\cite{pele2009fast}. This complexity can become a significant challenge, especially for high-dimensional datasets. To mitigate this computational burden, entropic regularization has been introduced as an effective strategy~\cite{cuturi2013sinkhorn}. By incorporating an entropy term into the optimal transport formulation, the problem is transformed into a nonlinear but {\em smooth} optimization problem, which can be solved more efficiently. This adjustment not only reduces the complexity of the problem but also enables its solution using linear-time algorithms. In the case of entropic regularization, the added entropy term effectively spreads out the transport plan, preventing the concentration of mass in a few narrow pathways. This spreading leads to a more evenly distributed plan, reducing the presence of sharp peaks and troughs in the optimization landscape. As a result, the optimization problem becomes more regular, with a smoother surface that is easier to navigate using optimization algorithms.}

In more formal terms, the entropic OT is defined by:
{
\begin{align}
\argmin_{\coupleM\in \Pi(\pr, \prQ)} \sum_{\vt{x}_i \in \mathcal{X}}\sum_{\vt{y}_j \in \mathcal{Y}} c(\vt{x}_i,\vt{y}_j) \coupleM(\vt{x}_i,\vt{y}_j) - \frac{1}{\sinkhornCoeff} H(\coupleM).
\label{eq:entropic-otlp}
\end{align}
}
\noindent where \(H(\coupleM)\) is the entropic regularizer:
{
\begin{align*}
H(\coupleM) = -\sum_{\vt{x}_i \in \mathcal{X}}\sum_{\vt{y}_j \in \mathcal{Y}} \coupleM(\vt{x}_i,\vt{y}_j) \log(\coupleM(\vt{x}_i,\vt{y}_j))
\end{align*}
}
\noindent and \(1/\sinkhornCoeff\) is the {\em entropic regularization parameter}. A smaller value means that we emphasize the accuracy of the transport plan, while a larger value leans towards computational efficiency.

Importantly, the OT plan $\coupleM^*$, which solves the constrained optimization problem defined in~\eqref{eq:entropic-otlp}, manifests as a diagonal scaling of the matrix $\mtx{K} := e^{-\frac{\mtx{C}}{\sinkhornCoeff}}$. Specifically, it has been shown that the solution to~\eqref{eq:entropic-otlp} is unique and takes the form \(\coupleM^* = \text{diag}(\mathbf{u}) \cdot \mtx{K} \cdot \text{diag}(\mathbf{v})\), with \(\mathbf{u}\) and \(\mathbf{v}\) acting as scaling vectors. These scaling vectors are identified through an iterative process, which ensures that the resultant transport plan complies with marginal probability constraints.  The Sinkhorn Algorithm, crucial for this process, iteratively adjusts \(\mathbf{u}\) and \(\mathbf{v}\) to ensure that the resultant transport matrix, $\coupleM^*$, adheres to the given marginal constraints. \reviewerthree{Lines~4 and 5 of Algorithm~\ref{alg:sinkhorn} represent these adjustments. Specifically, \(\vt{u}\) and \(\vt{v}\) are updated iteratively to balance the rows and columns of $\mtx{K}$, ensuring that the marginals of the scaled coupling matrix $\coupleM$ closely match $\vt{p}$ and $\vt{q}$.}

%\babak{is this a correct color?}

%In lines 4 and 5 of Algorithm~\ref{alg:sinkhorn}, the algorithm adjusts the vectors $\vt{u}$ and $\vt{v}$ to balance the matrix $\mtx{K}$ such that its row sums match the distribution $\vt{p}$ and its column sums match the distribution $\vt{q}$. This is achieved through element-wise division, denoted by ...

{ 
\begin{algorithm}
\caption{Sinkhorn Algorithm}\label{alg:sinkhorn}
\KwInput{Probability distributions $P,Q$ and cost function $c$}
\KwOutput{A transport plan between $P$ and $Q$} $\vt{p}:=\nit{vector}(P);\vt{q}:=\nit{vector}(Q);\mtx{C}:=\nit{matrix}(c);$\label{ln:init-pqc}\\
$\vt{u}:=\mathbbm{1}_{d_\dom{X}};\vt{v}:=\mathbbm{1}_{d_\dom{Y}};\mtx{K}:= e^{-\frac{\mtx{C}}{\sinkhornCoeff}};$\label{ln:init-k} \nComment{Initialization}\\
    \While(\nComment{Sinkhorn iterations}){$\vt{u}$ and $\vt{v}$ are not converged}{ 
        $\vt{u} := \vt{p} \oslash (\mtx{K}\cdot \vt{v});$ \nComment{\(\oslash\): Element-wise division}\\
        $\vt{v} := \vt{q} \oslash (\mtx{K}\cdot \vt{u});$   \label{ln:sinkhorn-it-1}\\
    }
    $\coupleM := \textit{diag}(\vt{u})\cdot \mtx{K} \cdot \textit{diag}(\vt{v});$\label{ln:plan}\\
    \Return{$\coupleM;$}
\end{algorithm}
}

%The regularization coefficient in Sinkhorn influences the mapping between distributions: lower coefficients (on the left) create sharp mappings with direct correspondences, while higher coefficients (on the right) lead to smooth mappings. These smoother mappings spread mass more evenly between source and target.

\begin{example}\label{ex:reg-ot}\em \reviewerone{Figure~\ref{fig:sinkhorn} presents the optimal transport between two Gaussian mixture model distributions, $P$ and $Q$. Each distribution is a mixture of two Gaussians, providing a basis for examining the effects of entropic regularization on transport plans. The leftmost graph in Figure~\ref{fig:sinkhorn} shows the original OT plan without entropic regularization. The optimal plan is more deterministic and sharp in mapping elements between the distributions. As we introduce and increase the entropic regularization coefficient, the subsequent transport plans become more spread out. This spread is visually observable in Figure~\ref{fig:sinkhorn}, where higher coefficients lead to transport plans that are less focused and more distributed across the space. This effect illustrates the principle of entropic regularization: a lower coefficient results in a transport plan that closely aligns specific elements of the distributions, whereas a higher coefficient allows for a broader, more generalized mapping. The intuition behind these transport plans can be understood by considering how the elements of one distribution, say ranging between $-2$ and $3$ in $P$, might be transported to another distribution $Q$ with values ranging between $0$ and $6$. Without regularization, the transport plan seeks to map these elements in a direct and specific manner. However, with entropic regularization, the mapping allows for the mass from one value in $P$ to be spread across the target distribution and to be transported to many values in $P$, thereby avoiding overly precise mappings that might not generalize well across different scenarios. This approach is particularly useful when dealing with high-dimensional data, where overly specific mappings can lead to overfitting and reduced model robustness.}
\end{example}

\paragraph{\bf Relaxed Optimal Transport:}
Relaxed OT, introduced in~\cite{frogner2015learning}, provides a loss function for supervised learning grounded in OT principles. Rather than relying on hard marginal constraints typical of entropic regularized OT, it adopts softer penalties, using regularization based on the Kullback-Leibler (KL) divergence. This approach leads to:

\vspace{-0.2cm}
{
\begin{align}
\argmin_{\coupleM \in \mc J} \sum_{\vt{x}_i \in \mathcal{X}} &\sum_{\vt{y}_j \in \mathcal{Y}} c(\vt{x}_i,\vt{y}_j) \coupleM(\vt{x}_i,\vt{y}_j) - \frac{1}{\sinkhornCoeff} H(\coupleM)\; +\nonumber\\
&\relaxedOTCoeff (\KL(\coupleM({Y}),\prQ) + \KL(\coupleM({X}),\pr)).
\label{eq:relax}
\end{align}
}

\noindent where $\relaxedOTCoeff$ is the relaxation regularization coefficient, and $\KL$ denotes the KL divergence between two probability distributions. Contrasting this with the entropic OT outlined in Equation~\ref{eq:entropic-otlp}, the transport plan $\coupleM$ in relaxed OT can be an element of $\mc J$, which includes all possible joint probability distributions over the product space $\Simplex{\st{X}} \times \Simplex{\st{Y}}$. It has been shown in~\cite{frogner2015learning} that Sinkhorn algorithm also works for the relaxed version of the entropic OT in Equation~\ref{eq:entropic-otlp} but with different update rules for $\vt{u}$ and $\st{v}$~\cite[Proposition 4.2]{frogner2015learning}:  

{\begin{align}
\vt{u} = (\vt{p}\oslash(\mtx{K}\cdot \vt{v}))^{\frac{\sinkhornCoeff\relaxedOTCoeff}{\sinkhornCoeff\relaxedOTCoeff+1}} \hspace{0.5cm}\text{and}\hspace{0.5cm} \vt{v} =  (\vt{q}\oslash(\mtx{K}^\top\cdot \vt{u}))^{\frac{\sinkhornCoeff\relaxedOTCoeff}{\sinkhornCoeff\relaxedOTCoeff+1}}    \label{eq:update-rules}
\end{align}
}

%% file: problem_definition.tex
\section{Problem Definition}\label{sec:problem}
%https://reader.elsevier.com/reader/sd/pii/S0195669817300860?token=4FC3D18950004077BB7E95179333CFD2C1750D2BD3F6C0110F8CD62F2F64CEC7D4E6C23B3F2E5212F2D3A4EF7BF04910&originRegion=us-east-1&originCreation=20230214022052

%The concept of data repair commonly pertains to modifying a database to adhere to a predefined set of constraints. As discussed in Section~\ref{sec:intro}, in some applications involving a data distribution, the distribution is expected to conform to certain CI constraints; it becomes pertinent to 'repair' the distribution itself, irrespective of its representation, to meet the requirements of these CIs. In this section, we formalize our intended meaning behind the notion of distribution repair.

Given a database $\db$ that is inconsistent with a CI constraint $\sigma: (\mb X \indep \mb Y \mid \mb Z)$, our objective is to resolve this inconsistency by updating the attribute values of each datapoint in $\db$ to derive a repaired database $\hat{\db}$ which is consistent with $\sigma$. To ensure minimal distortion and maintain the utility of the data, we assume we are given a user-defined cost function that quantifies the cost of updating a datapoint (this cost function generalizes the minimality criteria in update-based data repair in databases~\cite{bertossi2006consistent}). Leveraging the principles of OT, our goal is to develop a data cleaner, envisioned as a transport map, that repairs $ \db $ at a minimum cost. Next, we define the problem of learning an optimal data cleaner for a CI constraint.

\begin{definition}[CI Data Cleaner] \label{df:repair} \em
Consider a database $ \db = \{ \mb v_i \}_{i=1}^n $ that violates a CI constraint $ \sigma $, i.e., $ \db \not \models \sigma $, and a user-defined cost function $ c: \mathcal{V} \times \mathcal{V} \rightarrow \mathbb{R}_{\geq 0} $ that assigns a cost to transforming or perturbing one tuple in 
$\mathcal{V}$ to another tuple in 
$\mathcal{V}$. The \textit{CI data cleaner} of $ \db $ with respect to $ \sigma $ is a transport map $ T^* : \mathcal{V} \rightarrow \mathcal{V} $ that transforms $ \db $ into a database $ \hat{\db} = T^*(\db) = \{ \hat{\mb v_i} = T^*(\mb v_i) \}_{i=1}^n $ such that $ \hat{\db} \models \sigma $ and has the minimum transportation cost, i.e., $ T^* $ is the solution to the following constrained optimization problem: {\begin{align}
    \arg\min_{T} \sum_{i=1}^{n} c(\mb v_i, T(\mb v_i)) & 
    & \text{s.t.} \quad T(\db) \models \sigma. \label{eq:detrepair}
\end{align}
}
\end{definition}

We illustrate an optimal data cleaner with an example: 

\begin{example} \label{ex:odcexexis}\em     Let's consider a database $\db_1 = \{(0, 0, 1)$, $(1, 0, 1)$, $(0, 1, 1)$, $(0, 1, 0)\}$ defined over binary variables $X$, $Y$, and $Z$. $\db_1$ violates the CI constraint $\sigma:Y \indep Z$ because the probability $P_{Y,Z}(1, 0)$ is $\frac{1}{4}$, which is not equivalent to the product of the marginal probabilities $P_{Y}(1) = \frac{2}{4}$ and $P_Z(0) = \frac{1}{4}$. Further, suppose cost is measured using Euclidean distance.  
    An optimal CI repair can be obtained using the transport map $T$, which maps $(0, 0, 1) \rightarrow (0, 0, 0)$ and other tuples to their current values. As a result, by updating one attribute value, $T$ transforms $\db_1$ into a repaired database 
    $\hat{\db_1} = \{(0, 0, 0), (1, 0, 1), (1, 1, 0), (0, 1, 1)\}$,
    which is consistent with $\sigma$. %\mostafa{Do we mention L1 norm cost? Thanks, commented this} \babak{added Euclidean distanc}
\end{example}

% Mostafa: I corrected an error in the example that was due to assuming Z, X, Y as the order in the distribution, as opposed to X,Y,Z

However, the CI data cleaner defined in Definition~\eqref{df:repair} might not lead to a minimum cost repair. This is especially true if $\db$ is a bag, which is typically the case with databases used for ML. These databases are either bags or projections onto a subset of features that yield a bag. We illustrate this with an example:

\begin{example}\label{ex:odcexnotexis}
\em Continuing with Example~\ref{ex:odcexexis}, now consider a database $\db_2 = \{(1, 0, 0), (1, 0, 1), (1, 1, 0), (1, 1, 0)\}$, which is now a bag, and is inconsistent with the constraint $Y \indep Z$. Similarly, $\hat{\db_2} = \{(1, 0, 0)$, $(1, 0, 1)$, $(1, 1, 0)$, $(1, 1, 1)\}$ is a minimum cost repair for $\db_2$, obtained by modifying only one attribute value. However, no transport map exists that can transport $\db_2$ into $\hat{\db_2}$ simply because $(1, 1, 0)$ cannot be mapped to both itself and $(1, 1, 1)$. Upon close examination, it becomes evident that no transport map can lead to a repair for $\db_2$ with cost 1. 
\end{example}

%Moreover, for an instance $\db_3 = \{(1, 0, 0), (1, 1, 1)\}$, it is evident that a repair consistent with Definition~\ref{df:repair} does not exist. 

% Mostafa: The same error was in the above example.

\paragraph{\bf Probabilistic Optimal Data Cleaner} As demonstrated in Example~\ref{ex:odcexnotexis}, the transport map defined in Definition~\ref{df:repair} does not always yield the minimum cost repair (although it can always produce a trivial repair by mapping every tuple to a single tuple, which completely distorts the distribution). Indeed, it's possible for the minimum cost repair to be outside the feasible region defined by the problem in Equation~\eqref{eq:detrepair}. Drawing from the Kantorovich relaxation of OT, we shift our approach to seeking a transport plan, or transport coupling, denoted as $\pi(\mb v', \mb v)$, as an alternative to a deterministic transport map $T$. Here, the marginal distribution $\pi(\mb v) = \pr^{\db}$ represents the empirical distribution of the database $\db$, and $\pi(\mb v')$ is the target distribution that is consistent with the CI constraint. This transport plan yields a probabilistic mapping, $\pi(\mb v' \mid \mb v)$, which probabilistically updates a data point $\mb v \in \db$ to $\mb v'$ following the mapping. The repaired database is then obtained by applying this mapping to $\db$, by sampling. In essence, Definition~\ref{df:repair} transitions into a problem where the aim is to (1) identify a transport plan $\pi(\mb v', \mb v)$ that pushforwards the distribution $\pi(\mb v) = \pr^{\db}$, i.e., the empirical distribution associated with $\db$ into one consistent with the CI constraints, and (2) among all distributions with the same support and consistent with the constraint, find the distribution $\pi(\mb v')$ with the minimum OT distance to $\pi(\mb v) = \pr^{\db}$. Formally, an {\em optimal probabilistic data cleaner} for CI constraint seeks to clean data using a probabilistic mapping $\pi( \mathbf{v}' \mid \mathbf{v})$ associated with a transport plan or probabilistic coupling $\pi( \mathbf{v}', \mathbf{v})$, obtained by solving the following optimization problem:  

\vspace{-0.3cm}
{
\begin{align}\hspace{-3mm}
    \arg\min_{\pi} \sum_{i=1}^{d_{\Dom V}} \sum_{j=1}^{d_{\Dom V}} c(\mathbf{v}_i, \mathbf{v}'_j) \pi(\mathbf{v}_i, \mathbf{v}'_j)
    & \ \ \text{s.t.} \ \ \pi(\mathbf{v}) = \pr^{\mathrm{\db}}, \; \pi(\mb v') \models \sigma. \label{def:potr}
\end{align}}
\noindent The feasible region of the optimization problem defined in Equation~\ref{def:potr} consists of all possible probability distributions that satisfy the constraint, hence including a distribution associated with a minimal cost repair. Therefore, one can find a mapping that transforms the empirical distribution of $\db$ into a consistent distribution with the minimum cost. Moreover, the optimal probabilistic mapping, derived from solving Equation~\ref{def:potr}, provides an approach for probabilistic data cleaning. For large datasets, samples drawn from this probabilistic cleaner will lead to a dataset \( \hat{\db} \) whose empirical distribution \( \pr^{\hat{\db}} \) closely aligns with the target distribution \( \pr(\mathbf{v'}) \), in line with the law of large numbers. Consequently, the resulting dataset is approximately consistent with the constraint. In ML applications, this level of approximation is generally adequate.

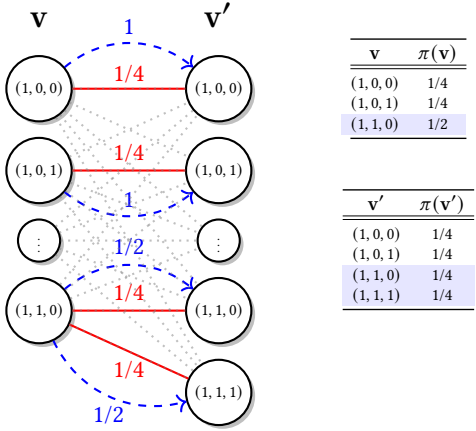
\begin{figure}
\begin{tikzpicture}[node distance=1.5cm, auto, thick, scale=0.7, every node/.style={scale=0.7}]

\tikzstyle{vertex}=[circle, fill=white, draw, minimum size=1cm, drop shadow={shadow xshift=0.05cm, shadow yshift=-0.05cm, opacity=0.5}]

\node[vertex] (a) {$(1, 0, 0)$};
\node[vertex] (b) [below=0.2cm of a] {$(1, 0, 1)$};
\node[vertex] (dots1) [below=0.2cm of b, scale=0.8] {$\vdots$};
\node[vertex] (c) [below=0.2cm of dots1] {$(1, 1, 0)$};

\node[vertex] (a') [right=1.5cm of a] {$(1, 0, 0)$};
\node[vertex] (b') [below=0.2cm of a'] {$(1, 0, 1)$};
\node[vertex] (dots2) [below=0.2cm of b', scale=0.8] {$\vdots$};
\node[vertex] (c') [below=0.2cm of dots2] {$(1, 1, 0)$};
\node[vertex] (d') [below=0.2cm of c'] {$(1, 1, 1)$};

% Draw the 1 probabilities
\draw[-,red] (a) -- (a') node[midway,above,scale=1.5,yshift=-1pt] {1/4};

\draw[->,blue, dashed, bend left=40] (a) to node[midway,above,scale=1.5,yshift=0pt] {1} (a');

\draw[-,red] (b) -- (b') node[midway,above,scale=1.5,yshift=0pt] {1/4};

\draw[->,dashed, blue, bend left=-40] (b) to node[midway,above,scale=1.5,yshift=0pt] {1} (b');

\draw[-,red] (c) -- (c') node[midway,above,scale=1.5,yshift=0pt] {1/4};

\draw[->, dashed,blue, bend left=40] (c) to node[midway,above,scale=1.5,yshift=0pt] {1/2} (c');

\draw[-,red] (c) -- (d') node[midway,below,scale=1.5,yshift=0pt] {1/4};
\draw[->, dashed, blue, bend left=-40] (c) to node[midway,below,scale=1.5,yshift=0pt] {1/2} (d');

% Draw the 0 probabilities with labels
\foreach \source in {dots1, a, b, c} {
    \foreach \dest in {dots2, a', b', c', d'} {
        \ifx\source\dest\else% do not label the arrows we already labeled
            \draw[-, dotted, gray, opacity=0.5] (\source) -- (\dest);
        \fi
    }
}

% Connect the dots with gray lines to all other nodes

% Adding labels and braces
\node[scale=2, above=0.2cm of a] (v) {$\mathbf{v}$};
\node[scale=2, above=0.2cm of a'] (v'){$\mathbf{v'}$};

 \node[left=-4.5cm of a, anchor=west] (lhstable) {
\begin{tabular}{@{}ccS[table-format=1.2]@{}}
\toprule
\Large{$\mathbf{v}$} & \Large{$\pi(\mathbf{v})$} \\
\hline
\midrule
$(1, 0, 0)$ & 1/4 \\
$(1, 0, 1)$ & 1/4 \\
\rowcolor{blue!10}
$(1, 1, 0)$ & 1/2 \\
\bottomrule
\end{tabular}
};

% Bottom Table
\node[below=0.5cm of lhstable, anchor=north] (rhstable) {
\begin{tabular}{cc}
\toprule
\Large $\mathbf{v}'$ &  \Large $\pi(\mathbf{v}')$ \\
\hline
\midrule
$(1, 0, 0)$ & 1/4 \\
$(1, 0, 1)$ & 1/4 \\
\rowcolor{blue!10}
$(1, 1, 0)$ & 1/4 \\
\rowcolor{blue!10}
$(1, 1, 1)$ & 1/4 \\
\bottomrule
\end{tabular}
}; 
\end{tikzpicture}
\caption{Graphical representation of the plan $\pi(\mb v, \mb v')$ for $\db_2$. Nodes represent elements in $\mc V$. Labeled red edges indicate joint probabilities $\pi(\mb v, \mb v')$, while dashed directed edges depict the probabilistic mapping $\pi(\mb v \mid \mb v')$. Only nodes and edges with non-zero probabilities are shown for clarity.}
\label{fig:tpexample}
\end{figure}

%\draw[decorate, decoration={brace, amplitude=2pt, raise=-5pt}, yshift=-5pt] 
%  (-1.2,-3.3,-12) -- (0.2,-3.3,-12) node [midway, above, yshift=-15pt] {};

% LHS Table for marginal distribution

\begin{example}\label{ex:odcexnotexis2} \em
Consider $\db_2 = \{(1, 0, 0)$, $(1, 0, 1)$, $(1, 1, 0)$, $(1, 1, 0)\}$ from Example~\ref{ex:odcexnotexis}. The probabilistic mapping $\pi(\mb v, \mb v')$ is graphically represented in Figure~\ref{fig:tpexample}, which depicts the bipartite graph constructed from the elements of the domain $\mc V$. Labeled red edges illustrate the joint probabilities $\pi(\mb v, \mb v')$, while dashed directed edges showcase the corresponding probabilistic mapping $\pi(\mb v \mid \mb v')$. The graph only includes nodes and edges for which $\pi(\mb v, \mb v')$ and $\pi(\mb v \mid \mb v')$ are non-zero to maintain clarity. It's evident that the marginal distribution $\pi(\mb v)$ displayed in Figure~\ref{fig:tpexample} matches the empirical distribution $\pr^{\db_2}$ associated to $\db_2$. Furthermore, $\pi(\mb v \mid \mb v')$ primarily maps all elements to themselves with a probability of 1. However, it transports half of the mass from $(1,1,0)$ to itself and the other half to $(1,1,1)$ to repair the constraint violation. This results in a distribution $\pi(\mb v')$ consistent with the constraint. Notably, the OT cost of this repair is $1/4$ since just $1/4$ of the mass with cost 1 transitions from $(1,1,0)$ to $(1,1,1)$. 

The mapping $\pi(\mb v \mid \mb v')$ can be employed to clean $\db_2$ probabilistically. Due to the limited sample size, this doesn't guarantee consistency. Still, for a larger database, the repaired database becomes representative of $\pi(\mb v')$ and hence becomes consistent with the constraint. To illustrate this, consider another database $\db_3$ echoing the tuples in $\db_2$, but each tuple is now replicated $n$ times. This mirrors the empirical distribution of $\db_2$ and still violates the constraint. In such a scenario, repairing $\db_3$ with $\pi(\mb v \mid \mb v')$ likely results in a consistent database. Probabilistically repairing the $2n$ instances of $(1, 1, 0)$ in $\db_3$ through the mapping $\pi(\mb v'\mid \mb v)$ can be interpreted as a sequence of $2n$ Bernoulli trials with a 1/2 probability. On average, this yields $n$ tuples of $(1, 1, 0)$ and $n$ tuples of $(1, 1, 1)$, ensuring consistency with the constraints.
\end{example}

% \babak{can you fix the spacing here}

\vspace{-3mm}
\paragraph{\bf Discussion on Complexity} Designing scalable algorithms to solve the optimization problem outlined in \eqref{def:potr} and subsequently computing optimal repairs for CI constraints presents significant challenges. A straightforward approach entails exploring the vast space of all distributions consistent with the CI, computing OT distance in relation to the empirical distribution of $\db$, and identifying the optimal solution. This method, however, is not feasible primarily due to the intractable nature of the space of consistent distributions. Furthermore, as discussed  in~\ref{sec:intro}, the computation of OT is computationally demanding. In our context, the transport plan involves \(d_{\Dom V}^2\) variables, thereby exacerbating the inherent complexity.

Although a detailed complexity analysis of the optimization problem~\ref{def:potr} is not addressed in this paper, it is worth noting that our problem is akin to the computation of minimum update-based repair (U-repair) for MVDs~\cite{bertossi2006consistent}. U-repair aims to identify a repair that necessitates the fewest attribute value modifications to enforce an MVD. Specifically, given a database \(D\) with attributes \(XYZ\) and an MVD \(X \twoheadrightarrow Y\), the decision problem is whether \(D\) has an optimal U-repair with no more than \(k\) modifications. This decision problem can be translated to our repair challenge by presuming a uniform distribution over \(D\), considering a cost function \(c(x, y, z, x', y', z')\) that enumerates the number of modifications required to obtain $(x', y', z')$ from $(x, y, z)$, and checking if \(D\) can achieve an optimal repair at a cost lesser than \(k\) given the conditional independence \(X \indep Y \mid Z\). Under the specified assumptions, it is easy to check \(D \models (X \indep Y \mid Z) \) if and only if \(D \models X \twoheadrightarrow Y\).  While there's extensive literature on the U-repair problem for Functional Dependencies~\cite{kolahi2009approximating,livshits2020computing}, to the best of our knowledge, it hasn't been studied for MVDs. %Furthermore, our repair problem, as described in~\cite{salimi2019interventional}, is also related to the rank-one matrix approximation problem, which is known to be NP-hard.

\ignore{

%\vspace{-3mm}
%\begin{align} 
%     \Tilde{T} = \argmin_{T\in \push}\dist (P,\Tilde{P})
%  \hspace{1cm} \text{such that\ } \Tilde{P} \models \Sigma   \label{eq:repairproblem}
%\end{align}

%A repaired database $\Tilde{D}$ is  drawn from $\popul$ to a {\em repaired database} $\drepair=\{\rdatap_i=\trans(\datap_i)\}_{i=1}^n$ that corresponds to a {\em repaired data distribution} $\modifpopul$, such that the distribution $\pr_{\rdatap \sim \modifpopul }(\rdatap)$ is {\em minimally different} from 

%Specifically, the problem of computing optimal repairs to a database $\dtrain$ wrt a set of CI constraints $\Sigma$ is how to transform the database $\dtrain$ to another database $\drepair$ via an optimal data cleaner $\optrans$ obtained by solving the following optimization problem: 
%
%\setlength{\abovedisplayskip}{0pt}{\setlist{nosep}\vspace{-.1cm} \begin{align}      \optrans = \argmin_{\trans}\dist\!\left(\pr_{\rdatap \sim \modifpopul }(\rdatap), \pr_{\datap \sim \popul }(\datap) \right)&&& \text{s.t.\ } \pr_{\rdatap \sim \modifpopul }(\rdatap) \models \Sigma   \label{eq:repairproblem}\end{align} }
%

\noindent In Equation~\ref{eq:repairproblem}, $\push$ is the set of possible pushforwards from the discrete space $\bDom V$ to itself, and $\dist$ is a function that measures the distance between two distributions, e.g., the family of {\em f-divergence} that generalizes the relative entropy (aka the Kullback-Leibler (KL) divergence)~\cite{csiszar1967information}, and  {\em optimal transport} measures % measures WHAT? 
such as the {\em Wasserstein distance}~\cite{villani2009optimal}. 

A repair $\Tilde{D}$ for $D$ w.r.t. $\Sigma$ is a database obtained from applying the pushforward $\Tilde{T}$ on $D$:

\vspace{-3mm}
\begin{align}\label{eq:repair}
\Tilde{D}=\{\Tilde{v}_j|\exists v_i \in D \wedge \Tilde{v}_j=\Tilde{T}(v_i)\}    
\end{align}

\noindent The idea of the repair in Equation~\ref{eq:repair} is as follows. Assuming the pushforward $\Tilde{T}$ converts the distribution $\pr$ to $\Tilde{P}$ that satisfies $\Sigma$, the repair $\Tilde{D}$, which is obtained from mapping data in $D$ using $\Tilde{T}$ will also respect the CI constraints.  

%We define a repair of $D$ as the database obtained from applying the pushforward $\Tilde{T}$ on $D$:

%Given $D$, we assume a known probability distribution $P_D$ as an estimation of $\Delta$ obtained from $D$. 

The {\em degree of inconsistency} of a distribution $\pr$ from a constraint $(\mb X \indep \mb Y \mid \mb Z)$ can be quantified by measuring the distance between $\pr_{X,Y,Z}(\mb x,  \mb y, \mb z )$ and  $\pr_{Y,Z}(\mb y, \mb z )\pr_{X\mid Z}(\mb x, \mb z )$, which is zero iff $(\mb X \indep \mb Y \mid_{\pr} \mb Z)$; for KL divergence, it is in correspondence with conditional mutual information~\cite{yeung1997framework}.

\babak{We should check whether the problem of finding optimal cardinality repair for saturated MVDs is in correspondent to our repair for a particular cost function.}
}

%% file: methods.tex
\section{Efficient Computation of Probabilistic Optimal Data Cleaner} \label{sec:methods}

In this section, we introduce efficient methods for computing the optimal data cleaner for CI constraints as described in~\eqref{def:potr}. In Section~\ref{sec:QCLP}, we formulate the problem as a Quadratically Constrained Linear Program (QCLP). This formulation allows for the derivation of an exact solution using existing efficient algorithms designed for QCLP. Subsequently, in Section~\ref{sec:BCDEOT}, we present an approximate version of the optimization problem in~\eqref{def:potr}. This approach facilitates the development of scalable and efficient solutions using iterative algorithms, particularly those based on Sinkhorn's matrix scaling.% algorithm.%, as mentioned in Section~\ref{sec:intro}.

\ignore{\begin{figure*}[h]
    \centering    \includegraphics[width=0.5\textwidth]{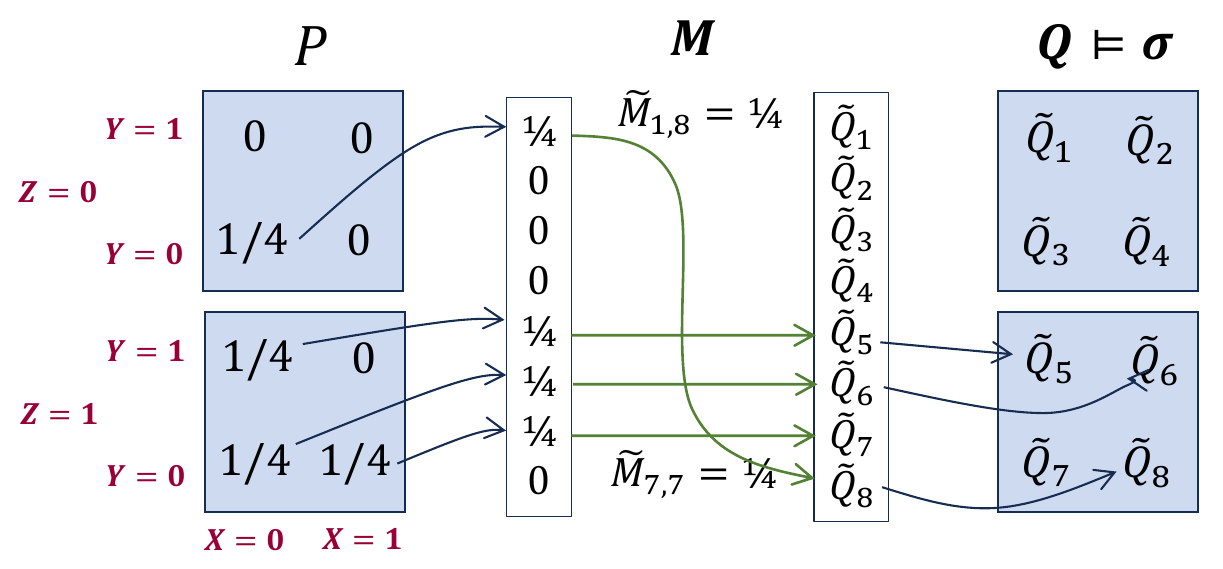}
    \caption{The QCLP formulation of the optimal data cleaner}
    \label{fig:qclp}
\end{figure*}}

\subsection{QCLP Formulation} \label{sec:QCLP} We present a QCLP designed to find an optimal data cleaner, as outlined in Section~\ref{sec:problem}. This program takes three inputs: a database $D$, a CI constraint $\sigma$, and a cost function $c$. We assume that $\sigma$ is a saturated CI constraint (i.e., it contains all attributes of $D$~cf.~\ref{sec:background}), with discussions on extending to unsaturated CI in Section~\ref{sec:unsat}.

To formulate the QCLP, we first describe the decision variables in the program, followed by an explanation of the constraints and the objective function. For clarity and better understanding, we use $D_2$ from Example~\ref{ex:odcexnotexis2} to demonstrate the QCLP formulation.

\begin{figure}
\begin{tikzpicture}[node distance=2cm, auto, thick, scale=0.7, every node/.style={scale=0.7}]

\node[align=center, anchor=west] at (-7,0)(pmatrix) {
    {\huge $\coupleM(\vt{v},\vt{v}')$}\\[10pt]
    \begin{tabular}{c@{}c}
      $\sourceColor{\pr^{D_2}(\vt{v})=\coupleM(\vt{v})\quad}$ & 
      ${\LARGE \begin{bmatrix}
        \tilde{\coupleM}_{1,1} & \cdots & \tilde{\coupleM}_{1,7} & \tilde{\coupleM}_{1,8} \\
        \tilde{\coupleM}_{2,1} & \cdots & \tilde{\coupleM}_{2,7} & \tilde{\coupleM}_{2,8} \\
        \tilde{\coupleM}_{3,1} & \cdots & \tilde{\coupleM}_{3,7} & \tilde{\coupleM}_{3,8}
      \end{bmatrix}}$ \\[5pt]
      & $\tilde{\prQ}(\vt{v}')=\coupleM(\vt{v}')$
    \end{tabular}
};

\node[anchor=north, right=0.5cm of pmatrix] (ptable) {
\begin{tabular}{cc}
\toprule
\Large $\mathbf{v}'$ &  \Large $\tilde{\prQ}=\pi(\mathbf{v}')$ \\
\hline
\midrule
$(0, 0, 0)$ & \Large $\tilde{\prQ}(0,0,0)$ \\
$\vdots$ & $\vdots$ \\
$(1, 1, 0)$ & \Large $\tilde{\prQ}(1,1,0)$ \\
$(1, 1, 1)$ & \Large $\tilde{\prQ}(1,1,1)$ \\
\bottomrule
\end{tabular}
}; 

\node[below=0.2cm of ptable, anchor=north] (validity) {
{\Large \begin{tabular}{l c}
\textbf{Validity constraints:} & \\
\multirow{2}{*}{$\quad\tilde{\coupleM}_{1,1} \ge 0, \tilde{\coupleM}_{1,2} \ge 0, \cdots, \tilde{\coupleM}_{3,8} \ge 0$} \\
\end{tabular}}
}; 

\node[below=0.0cm of pmatrix, anchor=north] (objectivef) {
{\Large \begin{tabular}{l c}
\textbf{Objective:} & \\
\multirow{2}{*}{$\quad\min_{\tilde{\coupleM}} (\costColor{1 \times} \tilde{\coupleM}_{1,1}+\costColor{2\times} \tilde{\coupleM}_{1,2}+$} \\
\multirow{2}{*}{$\hspace{1.3cm}\costColor{2\times} \tilde{\coupleM}_{2,1}+...+\costColor{1 \times}\tilde{\coupleM}_{3,8})$} \\
\end{tabular}}
}; 

\node[below=0.4cm of validity, anchor=north] (marginal) {
{\Large \begin{tabular}{l c}
\textbf{Marginal constraints:} & \\
\multirow{2}{*}{$\quad\tilde{\coupleM}_{1,1}+\tilde{\coupleM}_{1,2}+\cdots+\tilde{\coupleM}_{1,8}\sourceColor{=\frac{1}{4}}$} \\
\multirow{2}{*}{$\quad\tilde{\coupleM}_{2,1}+\tilde{\coupleM}_{2,2}+\cdots+\tilde{\coupleM}_{1,8}\sourceColor{=\frac{1}{4}}$} \\
\multirow{2}{*}{$\quad\tilde{\coupleM}_{3,1}+\tilde{\coupleM}_{3,2}+\cdots+\tilde{\coupleM}_{3,8}\sourceColor{=\frac{1}{2}}$} \\
\end{tabular}}
}; 

\node[below=0.4cm of objectivef, anchor=north](independenceC) {
{\large \begin{tabular}{l c}
\textbf{Independence constraints:} & \\
\multirow{2}{*}{$\tilde{Q}_{Y,Z}(0,0)=\tilde{Q}_{Y}(0)\times \tilde{Q}_{Z}(0)$} \\
\multirow{2}{*}{$\tilde{Q}_{Y,Z}(0,1)=\tilde{Q}_{Y}(0)\times \tilde{Q}_{Z}(1)$} \\
\multirow{2}{*}{$\tilde{Q}_{Y,Z}(1,0)=\tilde{Q}_{Y}(1)\times \tilde{Q}_{Z}(0)$} \\
\multirow{2}{*}{$\tilde{Q}_{Y,Z}(1,1)=\tilde{Q}_{Y}(1)\times \tilde{Q}_{Z}(1)$} \\
\end{tabular}}
}; 

\end{tikzpicture}
\caption{The QCLP for Example~\ref{ex:QCLP}. The top left is the transport plan defined by the decision variables. The top right is \( \tilde{\prQ} \) definitions. The rest are the objective and constraints. %\babak{Fix the cost and constraints}
}
\label{fig:QCLPexample}
\end{figure}

%\babak{Make sure $d_{\Dom V}$  is defined in background} Here, \( n \) is the number of tuples in  \( D \). \babak{First you say its \( |\Dom V| \) and then you suggest we can change this? Please reword} Meanwhile, \( d_{\Dom V} \), equal to \( |\Dom V| \), represents the total number of potential records we might consider. 
\paragraph{\bf Decision Variables} In the QCLP, decision variables are represented as \( \tilde{\coupleM}_{i,j} \), where both \( i \) and \( j \) span from 1 up to \( d_{\Dom V} \) (reflecting the size of the support of \( {\Dom V} \)). These variables are the transport plan's probabilities representing the optimal data cleaning strategy. Since this plan has non-zero probabilities exclusively for the values present in \( D \)'s active domain, \( i \)'s range can be limited to the size of \( D \)'s active domain. The following example clarifies this.

\begin{example} \label{ex:QCLP} \em
In the QCLP for the optimal cleaner of \(D_2\) from Example~\ref{ex:odcexnotexis2}, the transport plan is defined by an \(8 \times 8\) variable matrix. However, given that \(D_2\) contains only three records, we use a \(3 \times 8\) decision variable matrix, with the remaining rows of the initial matrix being zero. These decision variables indicate possible modifications to the three records in \(D_2\), enabling them to align with any of the eight potential records in \(\hat{D}_2\). The QCLP considers all eight potential records in \(\hat{D}_2\), each associated with its distinct variable.\end{example}

\paragraph{\bf Constraints} The QCLP incorporates three types of constraints to encode the conditions in our data cleaner formulation in \eqref{def:potr}:

\begin{itemize}[leftmargin=10pt]
\item \textit{Validity Constraints:} These constraints, together with marginal constraints, ensure that $\tilde{\coupleM}$ makes a valid transport plan. Specifically, the decision variables must be non-negative real values:
{
\begin{align}
\tilde{\coupleM}_{i,j}\ge 0\;\;\forall i\in[1,d_{\dom{V}}],j \in [1,d_\mathcal{V}]\label{eq:validity}
\end{align}
}
\item \textit{Marginal Constraints:} These constraints are included to guarantee that the marginals of the transport plan, as described by $\tilde{\coupleM}$, align with $P^D$ (the empirical distribution of $D$):%\babak{$P^D$ and D are the same }:
{
\begin{align}
\sum_{j=1}^{d_\mathcal{V}} \tilde{\coupleM}_{i,j}\sourceColor{=\pr^D(\vt{v}_i)}\;\;\;\forall i \in [1,d_\mathcal{V}] \label{eq:marginal}
\end{align}
}
\item \textit{Independence Constraints:} These constraints are formulated to ensure that the probability distribution $\coupleM(\vt{v}')$ satisfies the CI constraint $\sigma: (X \indep Y \mid Z)$. To express these constraints, we introduce $\tilde{\prQ}$ as the marginal probability distribution obtained from the decision variables $\tilde{\coupleM}$. The independence constraints express the equation $\tilde{\prQ}_{X,Z}(x',z') \times \tilde{\prQ}_{Y,Z}(y',z') = \tilde{\prQ}(x',y',z') \times \tilde{\prQ}_{Z}(z')$ and guarantee the marginal probability distribution satisfies $\sigma$. We use the notation $\tilde{\prQ}$ instead of $\prQ$ to emphasize that the decision variables in $\tilde{\coupleM}$ specify the marginal probability distribution.
\end{itemize}

\paragraph{\bf Objective} The objective of the QCLP is to minimize the transport cost, which is represented as follows:
{
\begin{align}
\min_{\tilde{\coupleM}} \sum_{i=1}^{d_\mathcal{V}} \sum_{j=1}^{d_\mathcal{V}} {c(\vt{v}_i,\vt{v}j) \times} \tilde{\coupleM}_{i,j}\label{eq:objective}
\end{align}
}
In this expression, the transport cost is calculated by summing the product of the cost function $c(\vt{v}_i,\vt{v}j)$ and the decision variables $\tilde{\coupleM}_{i,j}$, over all elements in the set $\mathcal{V}$.

\begin{example} \em
Expanding on Example~\ref{ex:QCLP}, Figure~\ref{fig:QCLPexample} shows the constraints and objective present in the QCLP for $D_2$. Specifically, the validity constraints ensure that 24 decision variables are non-negative. The three marginal constraints verify the alignment of the marginal probability, as defined by the transport plan, with the probabilities of the three input records in $D_2$. The independence constraints ensure that the probability distribution specified by $\tilde{Q}$ satisfies $\sigma: X\indep Y \mid Z$. For example, four independence constraints in this example guarantee $\sigma:Y\indep Z$ holds for all possible values of $Y$ and $Z$. The first independence constraint is $\tilde{Q}_{Y,Z}(0,0)=\tilde{Q}_{Y}(0)\times \tilde{Q}_{Z}(0)$, where the marginals $\tilde{Q}_{Y,Z}(0,0), \tilde{Q}_{Y}(0)$, and $\tilde{Q}_{Z}(0)$ are defined as sums of decision variables in $\tilde{\coupleM}$. The costs in the objective are the Euclidean distance between the input records and their possible repair, e.g., the cost 1 in $\costColor{1\times}\tilde{\coupleM}_{1,1}$ is the Euclidean distance between $(1,0,0)$, as the first record in $D_2$, and $(0,0,0)$, as the first possible repair. Similarly 2 in $\costColor{2\times}\tilde{\coupleM}_{1,2}$ reflects the Euclidean distance between $(1,0,0)$ and $(0,0,1)$.
\end{example}

%\babak{Should we add this as an algorithm? We should mention we linearize it. this is typical trick and it converges to a local min and add citations. }

%\mostafa{This is too simple and an algorithm env is overkill. An explanation will do IMO. I added the last paragraph with a citation.}

The above program is classified as a QCLP because, while the objective function and the validity and marginal constraints are linear with respect to the decision variables, the independence constraints are non-linear (quadratic). This is due to each side of the constraint consisting of a product of values in $\tilde{\prQ}$, that each is, in turn, a sum of the variables in $\tilde{\coupleM}$. QCLP represents a distinct subtype of Quadratically Constrained Quadratic Programs (QCQPs) or Second-Order Cone Programs (SOCPs) that feature quadratic constraints and objectives. Addressing a QCLP is a non-convex optimization problem and is NP-hard~\cite{van1966programming,boyd2004convex}. Diverse, efficient methodologies, including sequential quadratic programming, augmented Lagrangian, interior-point, and active set, have been employed to derive sub-optimal solutions for such programs~\cite{boyd2004convex}. 

We implemented an alternating algorithm to compute the optimal repair by solving the QCLP program. This method iteratively transforms the quadratic independence constraints into linear ones, similar to the Alternating Direction Method of Multipliers (ADMM)~\cite{boyd2011distributed}. The process begins with initial variable estimates for \(\tilde{\coupleM}\), ensuring the marginal distribution \(\tilde{\prQ}\) satisfies \(\sigma\). These initial values can be derived from the marginal probabilities of \(\pr^D\). In each iteration, we partition the variables in \(\tilde{\coupleM}\) into two subsets. We substitute the variables with their current estimates for the first subset, effectively linearizing the constraints. This transformation allows us to treat the second subset as variables within a linear program. In subsequent iterations, we alternate roles: treating variables of the second subset as constants and updating the first subset's values by solving a distinct linear program. This alternating process continues until the variables stabilize, indicating convergence. We have omitted the algorithm's specifics for brevity. The algorithm's convergence proof is similar to that of ADMM as presented in~\cite{boyd2011distributed}.

\subsubsection{Analysis of the QCLP Solution} \label{sec:qclp-a}

%\babak{This is incomplete. We should say the iterative algorithm solves a linear program in each iteration which is challenging due to the dimensionality. Also does it converge at the fixed point? Any thoughts on converges rate? }

%\mostafa{Linearization is mentioned before. I updated the paragraph, although the issues with high dimensionality were already in the paragraph. Feel free to edit and rephrase if needed. This is not about NP-hardness of QCLP but about why the practical solutions wont scale for us.}

%\babak{Please check}
The QCLP formulation, though convergent, encounters scalability challenges. Specifically, in each iteration, it necessitates solving an OT problem which is structured as a linear program. The computational complexity of determining the OT scales as \(O(d^3 \log(d))\) when comparing histograms of dimension \(d\)~\cite{pele2009fast}. In the following section, we introduce an alternative formulation that mitigates this scalability issue and obviates the need for solving a linear program.

%First, they tend to give sub-optimal solutions. Second, they struggle to scale for high-dimensional data, including our problem setting. In our problem, the number of variables grows exponentially with more attributes and as the database's size increases. Specifically, we have \({d_\dom{V}}^2\) decision variables, where \(d_\dom{V}\) grows exponentially with the number of attributes in \(\st{V}\). This makes solving the QCLP quite demanding. We suggest a new method based on relaxed OT to address this. This method is simpler but still thorough, addressing the main challenges of the QCLP.

\subsection{Fast Approximation via Relaxed OT using Sinkhorn Iterations}
\label{sec:BCDEOT}

In this section, we present an approximate algorithm for computing optimal repairs by casting the problem into a regularized optimization. This approach integrates the CI constraint and the constraint on marginals as regularizers, drawing inspiration from the relaxed optimal transport discussed in Section~\ref{sec:background}. Specifically, we formulate the problem of computing the optimal cleaner in \eqref{def:potr} as the following regularized optimization problem:
{
\begin{align}
\argmin_{\coupleM \in \Tp,\prQ \in \Simplex{\st{V}}} &\sum_{\vt{v}_i \in \mathcal{V}} \sum_{\vt{v}'_j \in \mathcal{V}} c(\vt{v}_i,\vt{v}'_j) \coupleM(\vt{v}_i,\vt{v}'_j) - \frac{1}{\sinkhornCoeff} H(\coupleM)\; +\nonumber\\
&\relaxedOTCoeff (\KL(\coupleM(\st{v}'),\prQ) + \KL(\coupleM(\st{v}),\pr^D)) + \mu\;\delta_\sigma(\prQ),
\label{eq:relax-opt-clean}
\end{align}
}
\noindent In the above formulation, \(P^D\) denotes the empirical distribution of the dataset \(D\). The target distribution, represented by \(Q\), functions as a decision variable, while \(\coupleM\) is the transport plan. The regularization term \(\KL(\coupleM(\st{v}'),\prQ) + \KL(\coupleM(\st{v}),\pr^D)\) penalizes the objective when there are deviations of its marginals \(\coupleM(\st{v})\) and \(\coupleM(\st{v}')\) from \(\pr^D\) and \(\prQ\), respectively. Additionally, the CI constraint, represented by \(\sigma\), is imposed on \(Q\) through the regularization term \(\delta_\sigma(\prQ)\) within the objective (recall from Section~\ref{sec:intro} that  $\delta_\sigma(Q)=\KLD[Q(X,Y,Z)\mid Q(X, Z) Q(Y\mid Z)]$). This term measures the degree of inconsistency of \(Q\) in relation to \(\sigma\) by utilizing the conditional mutual information, as discussed in Section~\ref{sec:background}.  This method is in contrast from the hard constraints used in the QCLP formulation Section~\ref{sec:QCLP}.  The hyperparameters \(\lambda\) and \(\mu\) serve as regularization coefficients, adjusting for discrepancies from the marginals and the degree of inconsistency in the target distribution \(Q\). The methodology for tuning these hyperparameters is discussed in Section~\ref{sec:experiment}.

 Intuitively, the optimization problem aims to find a distribution \(Q\) that aligns closely with the empirical distribution \(P^D\) while being consistent with the imposed constraint. The relaxed OT distance serves as a measure of this alignment, and the objective is to minimize this distance, ensuring that \(Q\) is a faithful representation of \(P^D\) that simultaneously satisfies the constraint.

%In this formulation, $\coupleM(\st{v})$ and $\coupleM(\st{v}')$ are the marginal probability distributions of the transport plan $\coupleM$ and both are in $\Simplex{\st{V}}$, and $\delta_\sigma(\prQ)$ is the CMI as described in Section~\ref{sec:background}.

%This formulation is similar to the relaxed OT approach. It features a term, \(\relaxedOTCoeff (\KL(\coupleM(\st{v}),\prQ) + \KL(\coupleM(\st{v}'),\pr^D))\), which serves as a soft constraint that matches the marginal probabilities of the plan \(\coupleM\) with both the input distribution \(\pr^D\) and the resulting distribution \(Q\). However, the formulation has a main difference with the relaxed OT in the inclusion of the CI constraint \(\sigma\) through the term \(\mu\;\delta_\sigma(\prQ)\). The introduction of this term is facilitated by the marginal distribution \(Q\), which is considered a variable in the optimization. The strength of the CI constraint \(\sigma\) is adjusted by the coefficient \(\mu\). A higher \(\mu\) ensures strict adherence to independence, whereas a lower value offers a more relaxed satisfaction.

%\babak{This is not gradient stuff. remove those. Just explain iterative optimization. Add citations...}

%\babak{The problem is non-convex in general. However, for a fixed $Q$ ... }
The inclusion of the CI constraint term makes our new formulation non-convex. We address this non-convexity with an alternating algorithm, \mainAlg. Before we detail \mainAlg in Algorithm~\ref{alg:repair}, we describe its main idea. In this algorithm, we sequentially focus on either the transport plan \(\coupleM\) or the resulting distribution \(\prQ\), optimizing one while holding the other constant. Initially, we can set \(\prQ\) to a distribution that meets the CI constraint \(\sigma\). With this fixed value, our objective becomes a convex function, which we solve using the Sinkhorn matrix scaling algorithm discussed in Section~\ref{sec:background}. When we alternate, our goal becomes minimizing the divergence between \(Q\) and \(\coupleM(\st{v}')\). In this stage, \(\prQ\) must also align with the CI constraint \(\sigma\). 

To address this problem, we adopt an alternating minimization strategy. Initiating with an initial guess for \(Q\), the algorithm first determines the optimal transport plan \(\pi (\mb v, \mb v')\) between \(P^D\) and \(Q\) through Sinkhorn iterations. In the subsequent iteration, a new \(Q\) is constructed based on the target distribution of \(\pi\), denoted \(\pi(\mb v')\). Specifically, this \(Q\) is identified to be proximate to \(\pi(\mb v')\) based on the KL divergence while also ensuring it either approximately or strictly satisfies the independence constraint. In subsequent iterations, the transport plan is recalibrated with respect to the revised \(Q\). Hence, the procedure can be viewed as a two-layered iterative process where the outer loop identifies a relaxed OT map, and the inner loop refines the target distribution of this map to enforce the constraint. The core intuition behind this approach is twofold. Firstly, the outer loop endeavors to determine a transport plan that maps the empirical distribution of data to a target distribution proximate to \(Q\), influenced by the regularization coefficient; its primary objective is to minimize the transport cost. Conversely, the inner loop evaluates the target distribution derived from the outer mapping and formulates a distribution in close alignment with it, ensuring adherence to the constraint. In essence, while the outer loop emphasizes on minimizing the transportation cost, the inner loop focuses on enforcing independence constraints.

The inner loop of this alternating algorithm, which reconstructs \(Q\) based on  \(\pi(\mb v')\) to satisfy the CI constraint, can be interpreted as a rank-one non-negative matrix factorization (as highlighted in Capuchin~\cite{salimi2019interventional}). Specifically, when dealing with conditional mutual information, the problem aligns with non-negative matrix factorization using the KL divergence objective, which is inherently non-convex but is typically addressed using alternating algorithms (for approximate enforcement of a CI constraint, one can use approximate matrix factorization techniques~\cite{finesso2004approximate}). For a specific value \(z \in {\Dom Z}\), we aim to determine matrices \(\mathbf{W}_z\) of size \(d_X \times 1\) and \(\mathbf{H}_{z}\) of size \(d_Y \times 1\). These matrices represent the joint and conditional distributions \(Q(X', Z'=z)\) and \(Q(Y' \mid Z'=z)\). They are chosen to minimize the divergence \(\KLD(\pi(X',Y',Z'=z) \mid \mathbf{W}_z \cdot \mathbf{H}_{z}^T)\). While the \(\KLD\) is convex with respect to either \(\mb W_z\) or \(\mb H_z\), it is not jointly convex for the pair \((\mb{W}_z, \mb{H}_z)\).  Established alternating methods, along with their associated update rules from the matrix factorization domain, such as those highlighted by Lee~\cite{lee2000algorithms}, can be employed. Starting with a random setup, these methods update \(\mathbf{W}_z\) and \(\mathbf{H}_z\) until they converge. The final matrices help us shape a new \(Q\) that satisfies the independence constraint.

%To minimize the KL divergence with respect to \(Q\), we initialize \(\mb W\) and \(\mb H\) using the marginals \(Q_X\) and \(Q_Y\) from our initial guess. Applying the update rules ensures that \(\mb W\) and \(\mb H\) converge, and the resulting matrices are then used to update the value of \(Q\) for the next iteration.

%To address the second part of the optimization, we employ an alternating method based on non-negative matrix factorization (NMF). For simplicity, let's consider a CI constraint \(\sigma\) defined as \(X\indep Y\). We can expand this concept for general CI constraints with conditional attribute $Z$. The strategy involves breaking down the marginal distribution \(\coupleM(\vt{v}')\) into two matrices, \(\mb W\) and \(\mb H\), using the multiplicative update rules in NMF~\cite{lee2000algorithms}, where the two matrices represent the marginal of $\coupleM(\vt{v}')$ w.r.t. $X$ and $Y$, respectively. To minimize the KL divergance with \(Q\), we initiate \(\mb W\) and \(\mb H\) using the marginals \(Q_X\) and \(Q_Y\) from the initial guess. As we apply the update rules, \(\mb W\) and \(\mb H\) adjust and are guaranteed to converge. The outcome then will serve as the value of \(Q\) for the subsequent iteration.

%\babak{We need a theorem that shows for a fixed target distribution (9) can be solved efficiently using an iterative algorithm.}

We outline the algorithm to solve the optimization problem in~\eqref{eq:relax-opt-clean}, denoted by \mainAlg, in Algorithm~\ref{alg:repair}. It begins by setting initial values for the vectors \(\vt{p}\), \(\vt{q}\), and the cost matrix \(\mtx{C}\) (see Lines~\ref{ln:init-main} to \ref{ln:init-q}). The vector \(\vt{q}\) is set up to represent probabilities in a distribution satisfying \( \sigma \), which serves as a first guess for the resulting distributions \( Q \). The vectors \(\vt{u}\) and \(\vt{v}\), and the matrix \(\mtx{K}\) are then prepared for Sinkhorn iterations (Line~\ref{ln:init-v-k}). The Sinkhorn method find a plan $\coupleM$ between our original \(\vt{p}\) and the estimate \(\vt{q}\) by updating \(\vt{u}\) and \(\vt{v}\) until they stabilize (Line~\ref{ln:sinkhorn-it}). See Section~\ref{sec:background} on checking convergence. After this, the algorithm computes the transport plan \( \coupleM \) (Line~\ref{ln:pi}) and shifts its focus to reconstructing \(\vt{q}\). The reconstruction step (Line~\ref{ln:block-end}) employed an alternating algorithm as described before to update  \(\vt{q}\).

%The update step (Line~\ref{ln:block-end}) applies NMF to each slice of \(\vt{q}\), keeping the value of \(Z\) constant in each slice, consistent with our prior discussion on NMF.

{ 
\begin{algorithm}
\caption{\mainAlg: Fast Computation of Probabilistic Data Cleaner for Conditional Independence}\label{alg:repair}
\KwInput{Database $D$, cost function $c$, and CI constraint $\sigma: X\indep Y \mid Z$}
\KwOutput{Transport plan (probabilistic data cleaner) $\coupleM$}
$\vt{p}:= \textit{vector}(P^{D});$  $\mtx{C}:= \textit{matrix}(c);$\label{ln:init-main}\\

 Randomly initialize  $\vt{q}$  \nComment{An initial guess for $Q$}\label{ln:init-q}\\ %\babak{Initialize u and update k}
$\vt{u}:=\mathbbm{1}_{d_\dom{X}};\vt{v}:=\mathbbm{1}_{d_\dom{Y}};\mtx{K}:= e^{-\frac{\mtx{C}}{\sinkhornCoeff}};$ \nComment{Sinkhorn Initialization} \label{ln:init-v-k}\\
\While(\nComment{Sinkhorn iterations}){{\em $\vt{q}$} is not converged}{
    
    \While{{\em $\vt{u}$} and {\em $\vt{v}$} are not converged}{ 
        %\textit{update $\vt{u}$ and $\vt{V}$} %using cost matrix k from c, p and q
        $\vt{u} := (\vt{p}\oslash (\mtx{K}\cdot \vt{v}))^{\frac{\sinkhornCoeff\relaxedOTCoeff}{\sinkhornCoeff\relaxedOTCoeff+1}},\vt{v} := (\vt{q}\oslash (\mtx{K}\cdot \vt{u}))^{\frac{\sinkhornCoeff\relaxedOTCoeff}{\sinkhornCoeff\relaxedOTCoeff+1}};$  \label{ln:sinkhorn-it}
    }
    $\coupleM = \textit{diag}(\vt{u})\cdot \mtx{K} \cdot \textit{diag}(\vt{v});$\label{ln:pi}\\
   % \For(\nComment{BCD for updating $\vt{q}$}){$z_k \in \dom{Z}$\label{ln:block-start}}{ 
   \For{each \(z \in {\Dom Z}\label{ln:block-start-ot}\)}{
   Initialize \(\mb W_z\), \(\mb H_z\) randomly. \\
   \While{ $\mb W_z$ and $\mb H_z$ are not converged} {
    \textbf{Update} $\mb W_z$ to minimize \(\KLD(\pi(X',Y',Z'=z) \mid \mathbf{W}_z \cdot \mathbf{H}_{z}^T)\) with $\mb H_z$ fixed \\
    \textbf{Update} $\mb H_z$ to minimize \(\KLD(\pi(X',Y',Z'=z) \mid \mathbf{W}_z \cdot \mathbf{H}_{z}^T)\) with $\mb W_z$ fixed}
}
    Construct $\vt{q}$ using $\mb W_z$s and $\mb H_z$s computed in the previous step
}
     % $\vt{q}= \texttt{getRefinedQ}(\pi)$ \nComment{An alternating procedure which construct a new $Q$ from $\pi$ that satisfies $\delta$}}.  
    \label {ln:block-end}
%}
\Return{$\coupleM$;}
\end{algorithm}
}

\subsubsection{Analysis of the algorithm} 
We prove that the algorithm converges. In Section~\ref{sec:experiment}, we empirically demonstrate the inner workings and convergence properties of this algorithm. In Section~\ref{sec:opti}, we propose efficient strategies to further optimize this algorithm.  

\begin{theorem}\label{th:correct}
For the optimization problem outlined in Equation~\eqref{eq:relax-opt-clean}, Algorithm~\ref{alg:repair} converges.
\end{theorem}
\begin{proof}
Algorithm~\ref{alg:repair} can be understood as an iterative optimization over one variable, either the transport plan \(\pi\) or the distribution \(Q\), while holding the other variable constant. When \(Q\) is fixed, optimization concerning the transport plan is smooth, differentiable, and strictly convex, ensuring that the Sinkhorn iterations converge, as established by~\cite{frogner2015learning}. Conversely, with a fixed \(\pi\), the inner problem breaks down into an objective function that remains strictly convex with respect to each matrix separately, and the adopted update rule ensures convergence to a stationary point, as elaborated in~\cite{hien2021algorithms}. This approach mirrors the Coordinate Descent method, where the objective function is convex for each individual coordinate. As per~\cite{tseng2001convergence}[theorem 5.1], this process guarantees convergence to a coordinate-wise minimum of the objective function.
\end{proof}

%% file: optimizations.tex
\section{Optimizations}
\label{sec:opti}

We applied several optimizations to improve \mainAlg that we briefly explain below and show their efficacy in Section~\ref{sec:experiment}. 

%We already mentioned some default optimizations in the previous sections. 

%\babak{Moved}These include improved initialization strategies for $\vt{q}$, $\vt{u}$, and $\vt{v}$, as well as mitigating the impact of the large domain size ($d_{\Dom V}$) on the overall runtime. These optimizations collectively contribute to the improved efficiency of the algorithm.
\vspace{-0.2cm}
\paragraph{\bf Default Optimization} We applied two straightforward yet effective optimizations: 1) Confining the transport plan's size to restrict mass movement solely within $D$'s active domain to $\Dom V$, excluding movement to the entire support. We explained this in the context of QCLP while defining decision variables in Section~\ref{sec:QCLP}. This restriction can be further narrowed down to allow mass movement within a more limited subset. 2) Rather than randomly initializing the target distribution $Q$ in \mainAlg, we initiated it with a distribution satisfying the CI constraint by applying Non-negative Matrix Factorization (NMF) to the empirical distribution of $D$, which our results demonstrated to aid faster convergence.
\vspace{-0.2cm}
\paragraph{\bf Warm Starting Sinkhorn} Convergence of the Sinkhorn iteration is a significant bottleneck in \mainAlg. We observe that our alternating algorithm, while it changes $Q$ in each iteration in which we fix the transport plan, only makes slight adjustments, implying that the transport plan should undergo minor changes in the next iteration. Therefore, instead of initializing the Sinkhorn scaling factors $\mb u$ and $\mb v$ with vectors of ones, adopting a warm starting approach by initializing them with the $\mb u$ and $\mb v$ from the previous iteration can significantly accelerate convergence. Our evaluation results indicate that this is a highly effective idea.
%\vspace{-0.2cm}
%\paragraph{\bf Parallel Computation} A computationally intensive portion of \mainAlg occurs within the for loop in Line~\ref{ln:block-start-ot} in Algorithm~\ref{alg:repair} that updates \(\vt{q}\) through an iterative alternating method for each value in \(\Dom Z\). This can become computationally expensive and hinder the overall runtime of the algorithm when \(\Dom Z\) is large. However, this step is embarrassingly parallel, meaning the update rules can be applied in parallel for different values in \(\Dom Z\).

\vspace{-0.1cm}
\input{unsaturated} \label{sec:unsat}

\ignore{Our proposed method for non-saturated case which there is an extra variable W that doesn't appear in CI constraint becomes costly when domain size of W increases. Under three assumptions on distance functions of (X,Y,Z,W) and (X,Y,Z) spaces, we can ignore the extra variable and find the repair in a smaller space and then bring it back to the actual space. Three assumptions state that for every value of w and w' these expressions must hold:
\begin{align}
\begin{split}
d_{X,Y,Z,W}((x_0,&y_0,z_0,w),(x_1,y_1,z_1,w)) =\\ &d_{X,Y,Z,W}((x_0,y_0,z_0,w'),(x_1,y_1,z_1,w'))
\end{split}
\end{align}
\begin{align}
\begin{split}
d_{X,Y,Z,W}((x_0,&y_0,z_0,w),(x_1,y_1,z_1,w')) \geq\\ &d_{X,Y,Z,W}((x_0,y_0,z_0,w),(x_1,y_1,z_1,w))
\end{split}
\end{align}
\begin{align}
\begin{split}
d_{X,Y,Z}((x_0,&y_0,z_0),(x_1,y_1,z_1)) =\\ &d_{X,Y,Z,W}((x_0,y_0,z_0,w),(x_1,y_1,z_1,w))
\end{split}
\end{align}For adjusting our approximate method to work on the space of (X,Y,Z) we just need to change the structure of our reconstructed distribution to the following form:
\begin{equation}
\forall z \in D_Z\ \ \ \ P'(x,y,z)=(p_{x,z}(z) \times p_{y|z}(z)^T)
\end{equation}
After getting the probability distribution of result on (X,Y,Z), any arbitrary distribution can be considered for W to reconstruct the probability distribution of result on (X,Y,Z,W). Proof is provided in the appendix.}

%% file: unsaturated.tex
\vspace{-0.1cm}\paragraph{\bf Unsaturated CI Constraints}\label{sec:unsat} So far, we assumed that \(\sigma: X\indep Y \mid Z\) represents a saturated CI constraint, implying \(\st{V}=\{X,Y,Z\}\). However, in many real-world scenarios, especially with high dimensional data, CI constraints may not be saturated.

For unsaturated constraints, we split \(\st{V}\), the set of attributes in the database \(D\), into two sets: \(\st{U}=\{X,Y,Z\}\) (the attributes in \(\sigma\)) and \(\st{W}=\st{V}\setminus \st{U}\) (those not in \(\sigma\)). A naive method is to compute a transport plan \(\coupleM\) of size \(d^2_{\Dom V}\), considering all attributes in \(\st{V}\), including \(\st{W}\). Adapting methods from Section~\ref{sec:methods} for this scenario is straightforward but computationally expensive with high-dimensional data.

A more efficient strategy is to run \mainAlg for the marginal distribution \(P^D_\st{U}\) instead of \(P^D\). This results in a smaller transport plan \(\coupleM_s\) of size \(d^2_{\Dom U}\) compared to \(\coupleM\). With \(\coupleM_s\), we construct \(\coupleM\) as follows: \(\coupleM(\vt{v},\vt{v}')=0\) if \(\vt{w}\neq \vt{w}'\), and \(\coupleM(\vt{v},\vt{v}')=\coupleM_s(\vt{u},\vt{u}') P_{\st{W} \mid \st{U}}(\vt{w}\mid \vt{u})\) otherwise. This ensures no additional transport cost for moving masses between different values of \(\st{W}\) as there is no mass moved for \(\vt{w}\neq \vt{w}'\). Thus, the cost associated with \(\coupleM\) is the same as \(\coupleM_s\), making it optimal if \(\coupleM_s\) is optimal. Note that this requires the cost function to satisfy some basic properties, such as the cost of \(\vt{u}\vt{w} \rightarrow \vt{u}'\vt{w}\) being equal to the cost of \(\vt{u} \rightarrow \vt{u}'\), which is satisfied by the Euclidean distance and other cost functions in our work. Additionally, the use of \(P_{\st{W} \mid \st{U}}(\vt{w}\mid \vt{u})\) ensures that \(\coupleM\) satisfies the marginal constraint \(P^D(\vt{v})=\coupleM(\vt{v})\). The resulting distribution \(Q\) from \(\coupleM\) satisfies \(\sigma\) as its marginal is \(Q_{\st{U}}\) which is known to satisfy \(\sigma\).

%We propose that the optimal solution can be expressed as \(Q_{\st{V}}(\vt{v}) = Q_{\st{U}}(\vt{u}) \times Q_{\st{W}\mid \st{U}}(\vt{w}|\vt{u})\), where \(Q_{\st{W}\mid \st{U}}\) is formulated as:%\begin{equation} Q_{\st{W}\mid \st{U}}(\vt{w}\mid \vt{u}') = \sum_{\vt{u} \in \dom{U}} \coupleM(\vt{u},\vt{u}') \times P_{\st{W} \mid \st{U}}(\vt{w}\mid \vt{u}) \label{eq:saturation} \end{equation}

%In the above equation, \(\coupleM\) denotes the reduced-size transport plan. The outcome of \eqref{eq:saturation} provides an optimal plan for \(P^D\), incorporating all variables \(\st{V}\). This is because the transport cost of $\coupleM$ and obtaining $Q_{\st{U}}$ will be the same as the transport cost of the larger transport plan for generating $Q_\st{V}$. This is because the conditional probability $P_{\st{W} \mid \st{U}}$ is used to convert the marginal $Q_{\st{U}}$ to $Q_{\st{V}}$, which means no additional cost is incurred for adding $\st{W}$ to the transport plan. \ignore{For example, if the incurred cost of moving mass from $\vt{u}=(0,0)$ to $\vt{u}'=(0,1)$ is $\coupleM(\vt{u},\vt{u}')=a$,  the same cost will be incurred for moving mass from $\vt{v}_1=(0,0,0)$ and $\vt{v}_2=(0,0,1)$ to $\vt{v}'_1=(0,1,0)$ and $\vt{v}_2=(0,1,1)$. }We have labeled this efficient method as the Saturation approach and will demonstrate its superior runtime performance compared to the ``Naive'' method in the experiments in Section~\ref{sec:sys-exp-opt}.

%% file: experiments.tex
\section{EXPERIMENTS}
\label{sec:experiment}
In our experimental evaluation of \sys, we seek to answer the following research questions: {\bf Q1} How does the end-to-end performance of \sys in terms of algorithmic fairness compare with baseline approaches? (Section~\ref{sec:fairness-exp}) {\bf Q2}  In data cleaning tasks related to CIs, how does the performance of \sys compare with the baselines? (Section~\ref{sec:exp-cleaning})
{\bf Q3} How effective is \sys in determining optimal repairs? This encompasses evaluating its convergence behavior, runtime performance, and efficacy of the optimizations. (Section~\ref{sec:sys-exp})

{ 
\begin{table}
    \centering
     \resizebox{0.35\textwidth}{!}{
    \begin{tabular}{*{5}{l}} 
    \toprule
        \textbf{Dataset} &
        \textbf{\#tuples} & 
        \textbf{\#attr.} & 
        \textbf{avg. dom} & 
        \textbf{init. CMI} \\
    \midrule
        \adult &
        48,842   & 
        14   & 
        5.42 & 0.18770 \\
        \compas &
        10,000   & 
        12   & 
        2.4 & 0.05484 \\
        \car &
        1,728    & 
        6   & 
        3.67 & 0.03617 \\
        \boston &
        506    & 
        14   & 
        4.5 & 0.05983 \\
    \bottomrule
    \end{tabular}
     }
    \caption{Datasets characteristics}

    \label{tab:datasets}
\end{table}
}

\paragraph{\bf Datasets} We used four datasets. The \adult and \compas datasets highlight the fairness aspect of \sys's application, while the datasets \car and \boston showcase the efficacy of \sys in data cleaning tasks. Table~\ref{tab:datasets} provides an overview of these datasets.

\paragraph{\adult~{\em \cite{adultdataset}}} 
In the \adult dataset, or ``Census Income,'', each entry captures details like age, work class, education level, marital status, occupation, relationship status, race, gender, weekly working hours, and country of origin. The dataset's main objective is to predict if an individual earns over \$50K annually.

\paragraph{\compas~{\em \cite{compasdataset}}} 
The \compas dataset from the Broward County Sheriff's Office in Florida predicts the likelihood of an individual re-offending. Key attributes include age, gender, race, criminal history, risk scores, charge degree, and jail history. \compas is essential for studies focusing on the fairness implications of predictive policing.

\paragraph{\car~{\em \cite{cardataset}}} 
The Car Evaluation dataset evaluates cars based on attributes like buying price, maintenance cost, number of doors, person capacity, and safety. Cars are classified based on their overall condition into unacceptable, acceptable, good, or very good.

\paragraph{\boston~{\em \cite{bostondataset}}} 
The Boston Housing dataset provides insights into the housing market in Boston, Massachusetts. It covers attributes like crime rate, residential zoning, average room count, distance to employment centers, and median home value. It's frequently used for regression analysis in predicting housing prices.

\paragraph{\bf Baselines} We use baselines \ignore{in both fairness and data-cleaning applications} that we briefly review here. 

\paragraph{Algorithmic fairness} In the realm of algorithmic fairness, the objective is to guarantee that decision-making algorithms operate equitably, avoiding discrimination based on sensitive attributes like race or gender. While there are myriad definitions of fairness in the literature, this study primarily focuses on {\em interventional fairness}, as articulated in~\cite{salimi2019interventional}. This particular notion underscores the importance of enforcing conditional independence within data. Consider a sensitive attribute \(S\). Without loss of generality, let's assume \(S\) is binary where \(S=1\) denotes the protected (or sensitive) group and \(S=0\) the unprotected group. Further, consider a ML model with output \(\hat{Y}\) trained on a set of features \(\mb X\).  The notion of interventional fairness divides \(\mb X\) into two sets:  {\em admissible variables} \(\mb A\) and {\em inadmissible variables} \(\mb N\). Admissible variables are those where the effect of the sensitive attribute on the outcome, mediated by these variables, is considered fair. In~\cite{salimi2019interventional}, the extent to which a ML model deviates from this fairness standard is quantified using {\em the Ratio of Observational Discrimination (ROD)}, defined as:

{
\[ 
\text{ROD} = \frac{1}{|dom(A)|} \sum_{a \in A} \frac{P(\hat{Y}=1|S=0,a)P(\hat{Y}=0|S=1,a)}{P(\hat{Y}=0|S=0,a)P(\hat{Y}=1|S=1,a)}
\]
}
A ROD value of 1 signals the absence of any bias and is in correspondence to the conditional independence $(\hat{Y} \indep S \mid \mb A)$. In this paper, we employ the logarithm of the ROD for our analyses. A logarithmic ROD value of 0 is indicative of the absence of discrimination, while progressively higher values of the log ROD signify increasing levels of bias. The approach detailed in~\cite{salimi2019interventional} reduces the challenge of training a fair ML model to the task of enforcing a CI constraint on the training data. They introduced several methods in this context, which we adopt as baselines for our evaluations. Their methods fall into two categories: Methods based on matrix factorization and MaxSat methods. From the first category, the ``Cap(MF)'' factorizes each joint probability distribution of \(P^D\) for a fixed value of \(\st{Z}\) by minimizing Euclidean norm, while ``Cap(IC)'' does the factorization by using marginals of the initial distribution. They also propose a problem reduction of repairing w.r.t a CI constraint to solving a general CNF formula, and they solve it using their MaxSat method ``Cap(MS)''. We also included a naive baseline referred to as 'Dropped,' where the model is trained solely on admissible variables, which is sufficient for enforcing intervention fairness, as demonstrated in~\cite{salimi2019interventional}. 

\paragraph{Data Cleaning} \reviewerone{In our data cleaning evaluation, we assess the performance of \sys\ and compare it with various imputation and data cleaning methods. We consider five baselines for handling missing values: 1) Most frequent (MF) fills missing values with the most frequent values within the attribute, 2) k-nearest neighbors (kNN) identifies the most frequent values among neighboring data points for imputation, 3) GAIN uses Generative Adversarial Networks~\cite{yoon2018gain}, and 4) Hyperimputation is a method that integrates multiple imputation techniques, blending traditional iterative imputation with deep learning~\cite{jarrett2022hyperimpute}. We selected kNN and MF as basic, widely-used baselines. We compared \sys with GAIN since it is a leading imputation method and Hyperimpute since it is known for its ability to surpass various imputation techniques. We also use two baselines in scenarios with attribute noise: 1) using the dirty dataset as a simple baseline, and 2) Baran~\cite{mahdavi2020baran} as an advanced data cleaning method that utilizes comprehensive context information, including the value, co-occurring values, and attribute type, to generate correction candidates with high precision.}

%In addition to these competing methods, we also use the evaluation framework in~\cite{dasu2012statistical} to compare the performance of \sys as a cleaning strategy, which we explain in detail in Section~\ref{sec:distortion}.

\vspace{-0.3cm}
\subsection{Tuning \sys} \label{sec:tuning}
%\babak{Plerase revise: we want to say our algorithm need two parametr. refresh their mind what those parameters are and a clear strategy how we want to tune them.}

\paragraph{Cost function} We employ two cost functions in our experiments. The first function calculates the cost as the Euclidean distance between two records after normalizing their attributes by dividing them by their standard deviation. The second function utilizes a distance learned through MLKR (Metric Learning for Kernel Regression~\cite{weinberger2007metric}), a supervised metric learning technique that minimizes the leave-one-out regression error. We chose MLKR because it is widely used for distance learning and designed explicitly for supervised tasks like those in our settings. We label the results from the first cost function as \sys-C1, while the cost function using the learned distance is labeled as \sys-C2.

% The second function utilizes a distance learned through MLKR (Maximum Likelihood Kernel Ridge), a metric learning technique that fine-tunes a kernel ridge regression model to find an effective distance for a classification task. We chose MLKR because it is widely used for distance learning while giving better results in supervised tasks like those in our settings.

% MLKR metric learning?

%\babak{Explain how we choice the distant function}
\vspace{-0.3cm}
\paragraph{Regularization Coeffients} Two tuning parameters of \mainAlg are \( \relaxedOTCoeff \) and \( \frac{1}{\sinkhornCoeff} \). As \( \relaxedOTCoeff \) and \( \sinkhornCoeff \) grow, our formulation of \sys gets closer to the OT distance, and \mainAlg gives better results. However, as their values grow, the cost of running \mainAlg increases due to slower convergence. To find parameter values that balance runtime and fast convergence, we perform a grid search for each dataset to tune \sys. \sys has another parameter, \( \mu \), that quantifies the dissatisfaction of the CI constraint. %removed this intentianaly However, since our algorithm invariably satisfies the CI constraint, it does not impact \mainAlg.

%\sys is impacted by two parameters: \( \sinkhornCoeff \) and \( \relaxedOTCoeff \). We identified their optimal values through a systematic grid search. We observed that as \( \relaxedOTCoeff \) increases, the OT distance improves until reaching a plateau. Adjusting \( \rho \) refines the distance, but only up to a specific limit. Consistently, these parameter settings performed well across various datasets. It is noteworthy that \( \mu \) in Eq.~\ref{eq:relax-opt-clean} has no bearing on the function of \mainAlg, as the algorithm's transport plans neutralize \( \mu \)'s influence. Such a property might be of significance for approximation algorithms.

\subsection{Algorithmic Fairness} 
\label{sec:fairness-exp}

We evaluate the effectiveness of {\sys} within the domain of algorithmic fairness. 
To harness \sys\ for training interventionally fair algorithms, we utilize our probabilistic data cleaning approach to modify the data, ensuring its consistency with the CI constraint (\(S \indep \mb N \mid \mb A\)). This enforced independence ensures the sensitive attribute does not influence the inadmissible variable, except through \(\mb A\). If this independence is maintained, any valuable predictive information encapsulated within the inadmissible variables \(\mb N\) cannot be sourced from the sensitive attribute. The flexibility of our approach, underpinned by optimal transport, allows us to craft specific cost functions for probabilistic data cleaning to preserve as much predictive capability as possible. Specifically, we designed a cost function to modify the inadmissible variables and keep sensitive attributes and admissible variables unchanged, ensuring that while fairness is achieved, all relevant predictive information within \(\mb A\) is retained. Additionally, it ensures that any remaining predictive value within \(\mb N\) is not derived from the sensitive attribute \(S\).

We applied \sys\ to establish a probabilistic data cleaner for the training data. This cleaner was subsequently used to pre-process the dataset. The subsequent sections present evaluation results on the \adult and \compas datasets. Our evaluation metrics include cross-validated AUC and the mean ROD averaged over iterations derived from cross-validation outcomes. Besides ROD, we also assess other fairness measures, such as equality of odds and demographic parity. Notably, our approach incidentally enhances these fairness metrics as well. We also report other popular fairness measures, such as equality of odds— which requires that classifiers have equal false positive and false negative rates across protected groups—and demographic parity, which ensures that the decision outcome is independent of the protected attribute. 

Figure~\ref{fig:fairness} showcases our evaluation results for the \compas and \adult datasets. In the \adult dataset, the sensitive attribute is ``sex'', ``marital-status'' is inadmissible, and the admissible attributes include ``occupation'', ``education-num'', ``hours-per-week'', and ``age''. For \compas, we treat ``race'' as sensitive, ``age-cat'' and ``priors-count" as inadmissible, and ``charge-degree'' as admissible. Notably, \sys\ demonstrates superiority over the baseline, achieving models that are at least as fair, if not fairer, and exhibit an elevated AUC. This improvement can be attributed to our optimal transport-based approach, which empowers our method to retain considerable predictive value while rigorously enforcing fairness constraints. Furthermore, Figure~\ref{fig:fairness-measures} shows \sys's reasonable performance on other fairness notions, specifically Equality of Opportunity (EO) and Demographic Parity (DP). On both datasets, our methodology consistently surpasses the baseline in these respects. (Note: the result of ``Cap(MS)" is not plotted in Figure~\ref{fig:compas} as it achieved a constant AUC of 0.5 in all cross-validation iterations.)

\begin{figure}[h]
    \centering
    \begin{subfigure}{0.22\textwidth}
        \centering
        \includegraphics[width=\linewidth]{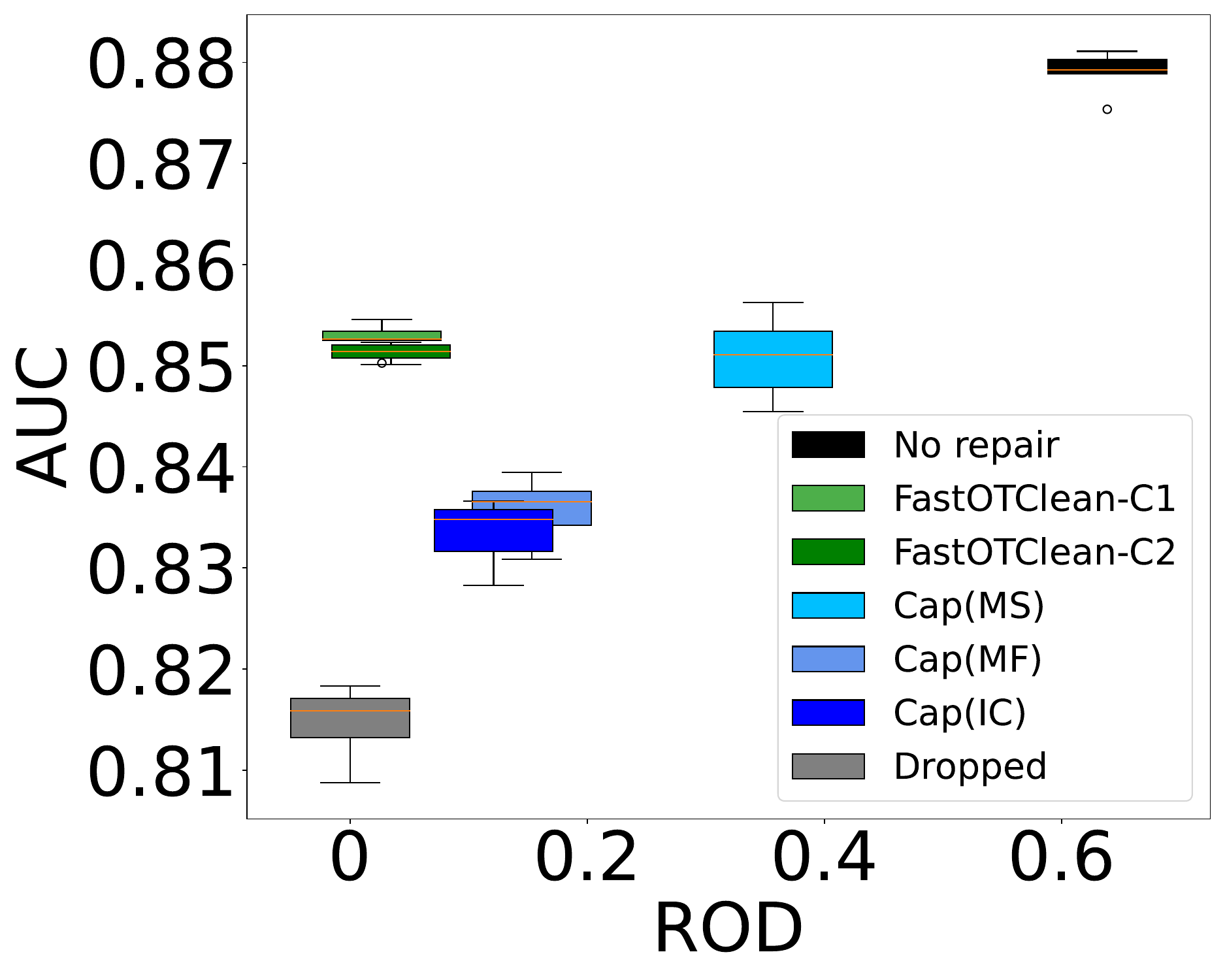}
        \caption{\adult}
        \label{fig:adult}
    \end{subfigure}
    \hfill
    \begin{subfigure}{0.22\textwidth}
        \centering
        \includegraphics[width=\linewidth]{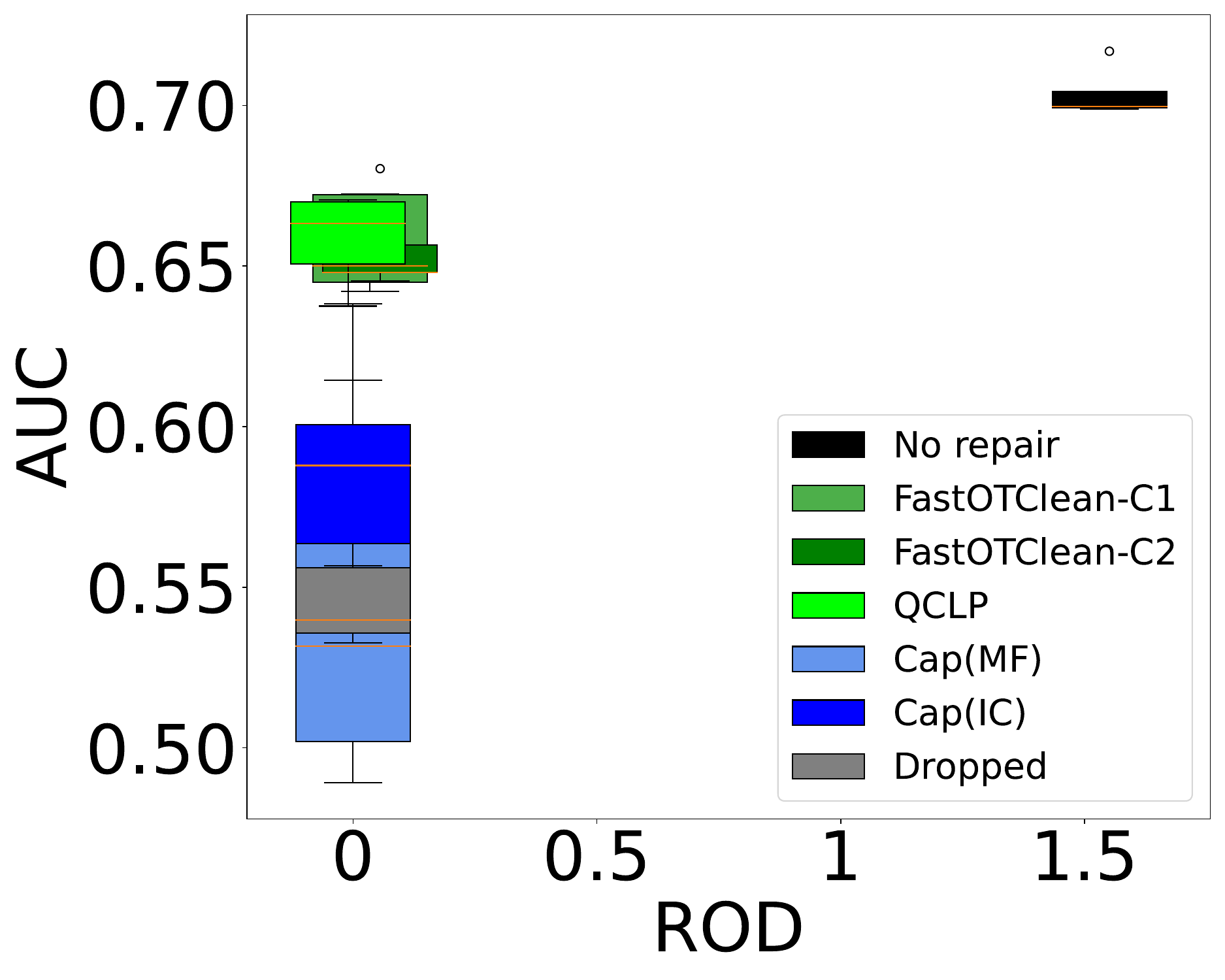}
        \caption{\compas}
        \label{fig:compas}
    \end{subfigure}
    \vspace{-3mm}
    \caption{\reviewerthree{Comparison of \sys's performance with the baselines showing higher AUC and lower ROD (bias)}}
    \label{fig:fairness}
\end{figure}

\begin{figure}[h]
    \centering
    \begin{subfigure}{0.22\textwidth}
        \centering
        \includegraphics[width=\linewidth]{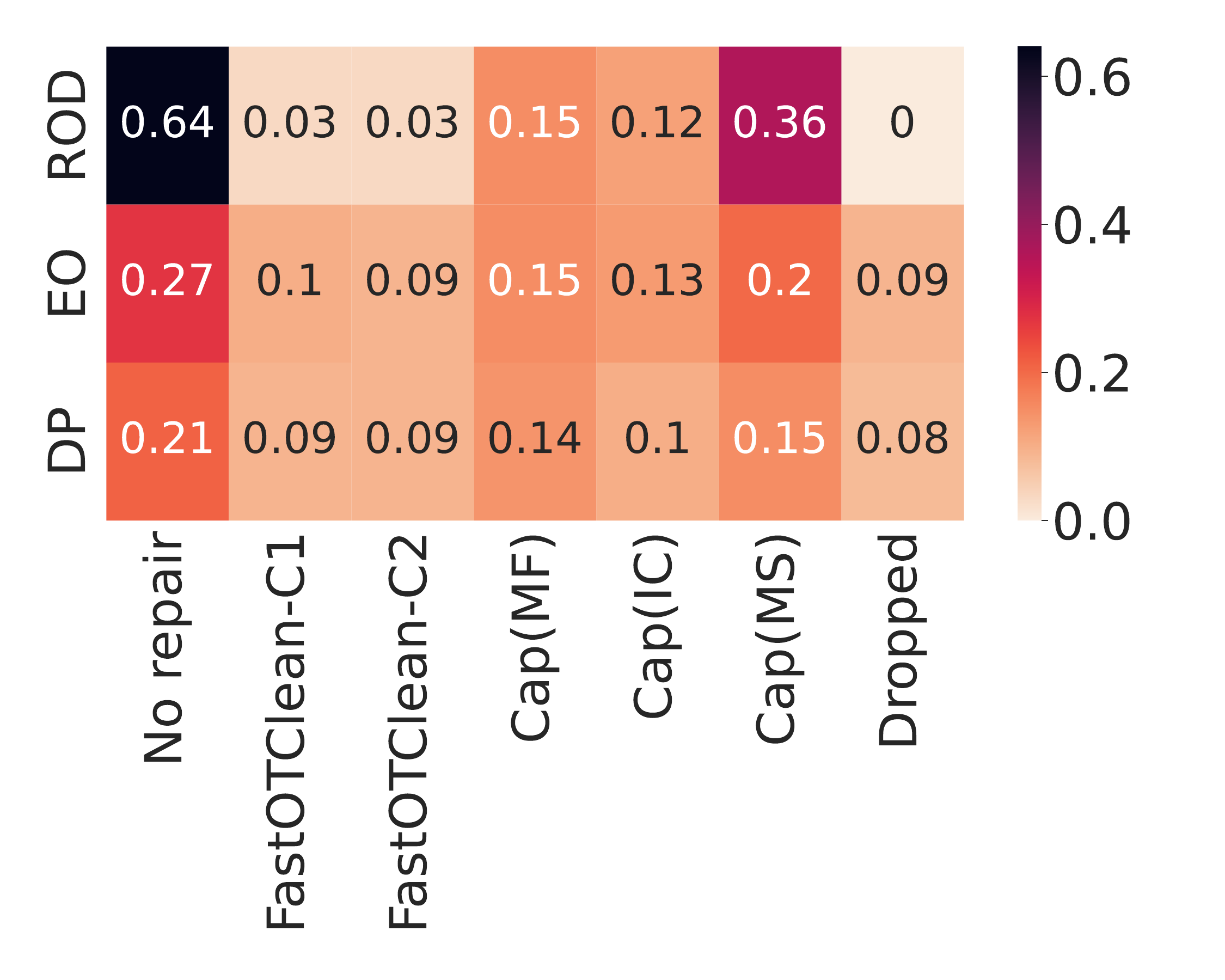}
        \caption{\adult}
        \label{fig:adult-fiarness}
    \end{subfigure}
    \hfill
    \begin{subfigure}{0.22\textwidth}
        \centering
        \includegraphics[width=\linewidth]{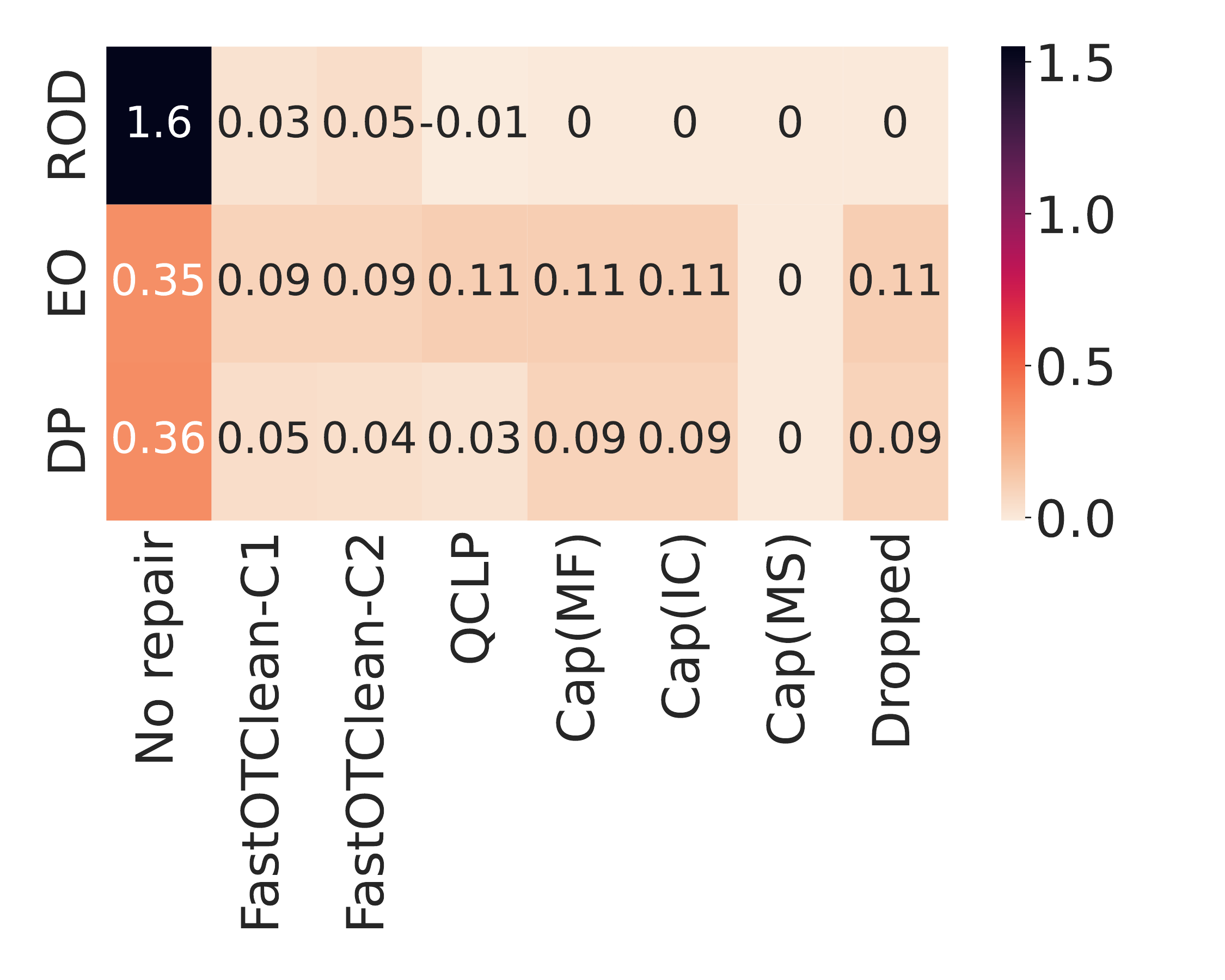}
        \caption{\compas}
        \label{fig:compas-fairness}
    \end{subfigure}
    \vspace{-3mm}
    \caption{\reviewerthree{Fairness metrics in \sys, indicating lower biases (ROD, EO, and DP) compared to baseline methods}}
    \label{fig:fairness-measures}
\end{figure}

% \begin{table}
%     \centering
%     % \resizebox{0.45\textwidth}{!}{
%     \begin{tabular}{*{5}{l}} 
%     \toprule
%         \textbf{Method} &
%         \textbf{ROD} & 
%         \textbf{EO} & 
%         \textbf{DP} \\
%     \midrule
%         No repair &
%         0.636   & 
%         0.26   & 
%         0.20 \\
%         \sys-C1 &
%         0.009   & 
%         0.09   & 
%         0.09 \\
%         \sys-C2 &
%         0.029   & 
%         0.09   & 
%         0.09 \\
%         Cap(MF) &
%         0.148   & 
%         0.14   & 
%         0.14 \\
%         Cap(IC) &
%         0.113   & 
%         0.12   & 
%         0.10 \\
%         Dropped &
%         0   & 
%         0.09  & 
%         0.08 \\
%     \bottomrule
%     \end{tabular}
%     % }
%     \caption{methods fairness measures.}
%     \vspace{-2mm}
%     \label{tab:fairness}
% \end{table}

%In the case of not changing sensitive and outcome attributes, DP doesn't change a lot, and the reason is that if our ML classifier has good accuracy, then the distribution of Y and $\hat{Y}$ would be close, and thus $P(\hat{Y}=1|S=0)-P(\hat{Y}=1|S=1)$ would be close to $P(Y=1|S=0)-P(Y=1|S=1)$. The DP won't be improved since the distribution of Y and S has not changed.

%\subsection{Repairing Test Data}

\subsection{Data Cleaning}\label{sec:exp-cleaning}

%\babak{Make sure to mention we generated semi synthetic data in which we synthetically injected error in training data but use the clean test data to evaluate generalizability to unseen test data.}
To evaluate the performance of \sys in data cleaning, we conducted experiments using semi-synthetic datasets that featured two types of dirty data: {\em attribute noise} and {\em missing values}. These datasets were derived from the \car and \boston datasets. We used these datasets to train ML models for predicting the labels ``class'' (indicating the car's condition in the \car dataset) and ``medv'' (representing median house price in the \boston dataset), respectively. In each case, we introduced noise errors and missing values into the training data, while the original clean data served as the test set for assessing model generalization.  For \car, we considered the CI constraint (\( \text{doors} \indep \text{class} \,|\, \text{the remaining attributes} \)). This constraint implies that the number of car doors should not significantly impact the class label when considering other factors such as buying price and safety. For the \boston dataset, we examined the constraint (\( \text{B} \indep \text{medv} \,|\, \text{the remaining attributes} \)), which suggests that the ``B'' attribute (indicating the percentage of blacks per town) should not influence the ``medv'' label. Initially, these constraints approximately held in the original datasets. To introduce attribute noise, we deliberately added non-random noise that led to violations of the CI constraints. Additionally, we injected two types of missingness: missing at random (MAR) and missing not at random (MNAR).

%We excluded missing completely at random (MCAR) because it is simpler to address with basic methods, such as mode imputation or mean substitution. Moreover, MCAR values may not introduce any new unwanted attribute correlations that \sys aims to eliminate and are outside our analysis scope. 

%% This is not needed
%Typically, domain experts can provide such CI constraints beforehand, or they can be discovered using techniques like constraint disc overy algorithms in data profiling. 

%\babak{I can not follow this.}

\reviewerthree{We chose to use a semi-synthetic dataset, where we added errors to real-world data, to create both ``dirty'' datasets and their accurate ground truths. This was essential because it is difficult to find real datasets with both genuine errors and ground truth. A limitation of this approach is that the injected error patterns may not exactly replicate those in actual datasets. However, our cleaning system is designed to be effective regardless of the specific error types. It primarily targets fixing spurious correlations and reducing the impact of any differences in error patterns on our goals.}

\reviewerone{To create a dependency between two attributes through attribute noise, we introduce random noise into one based on the values of the other. For adding missing data, our approach depends on the type. In MAR scenarios, where an attribute's missingness is influenced by another attribute, we decide to add missing values based on the other attribute's values in the same record. In MNAR cases, where an attribute's missingness is affected by its own value and other attributes, we randomly select records and determine missingness based on these factors. This method systematically creates relationships between attributes, effectively incorporating noise and addressing different missing data situations.}

To assess the efficacy of \sys, we utilized the ``Dirty'' datasets to train various ML models, including logistic regression, random forest, SVM, and MLP, and reported results for the best-performing model. When dealing with missing values, we employed two imputation methods: most frequent values (MF) and kNN, as explained previously. The dirty model is labeled with the imputation method used for training the dataset. In all experiments, the models were tested on ground truth data (the data before adding noise or missing values), and the models trained on the ground truth were denoted as ``Clean.'' Additionally, we applied \sys\ to enforce the corresponding CI constraint before training the ML models. This step aimed to remove spurious correlations induced by violations of CI, which could lead to poor performance of the ML model.

\begin{figure}[h]
    \centering
    \begin{subfigure}{0.22\textwidth}
        \centering
    \includegraphics[width=\linewidth]{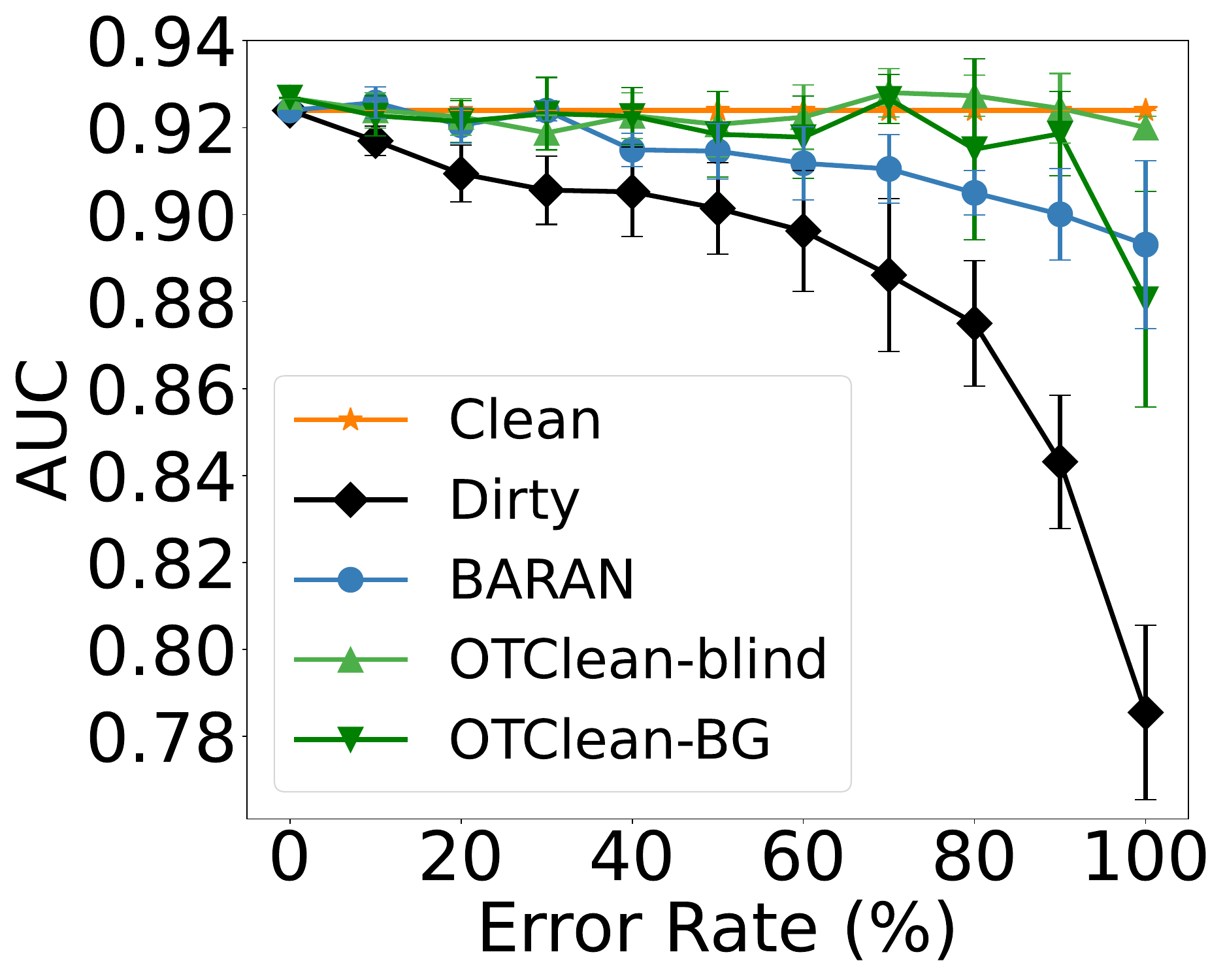}
        \caption{\car}
        \label{fig:car-auc-noise}
    \end{subfigure}
    \hfill
    \begin{subfigure}{0.22\textwidth}
        \centering
        \includegraphics[width=\linewidth]{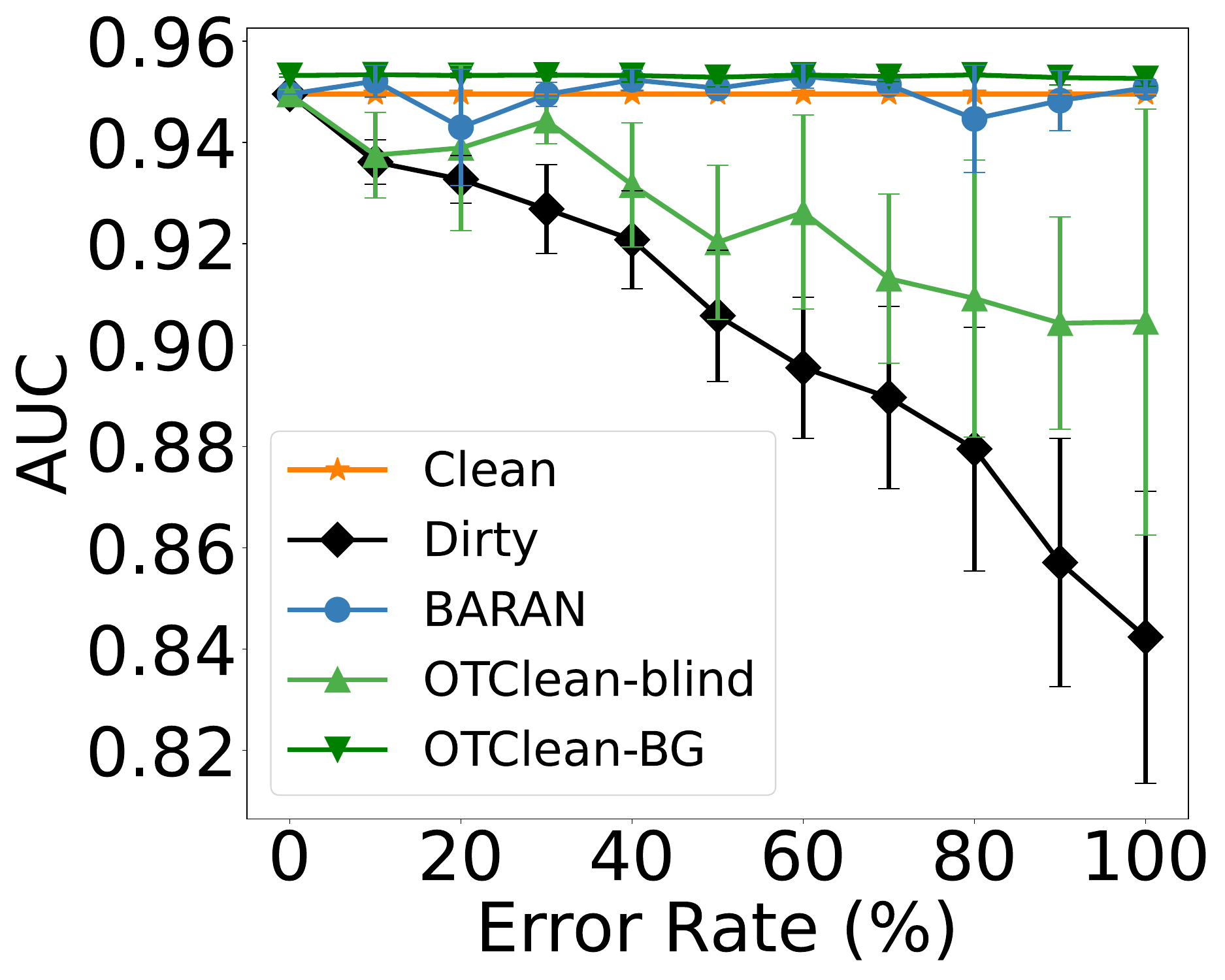}
        \caption{\boston}
        \label{fig:boston-auc-noise}
    \end{subfigure}
    \begin{subfigure}{0.22\textwidth}
        \centering
    \includegraphics[width=\linewidth]{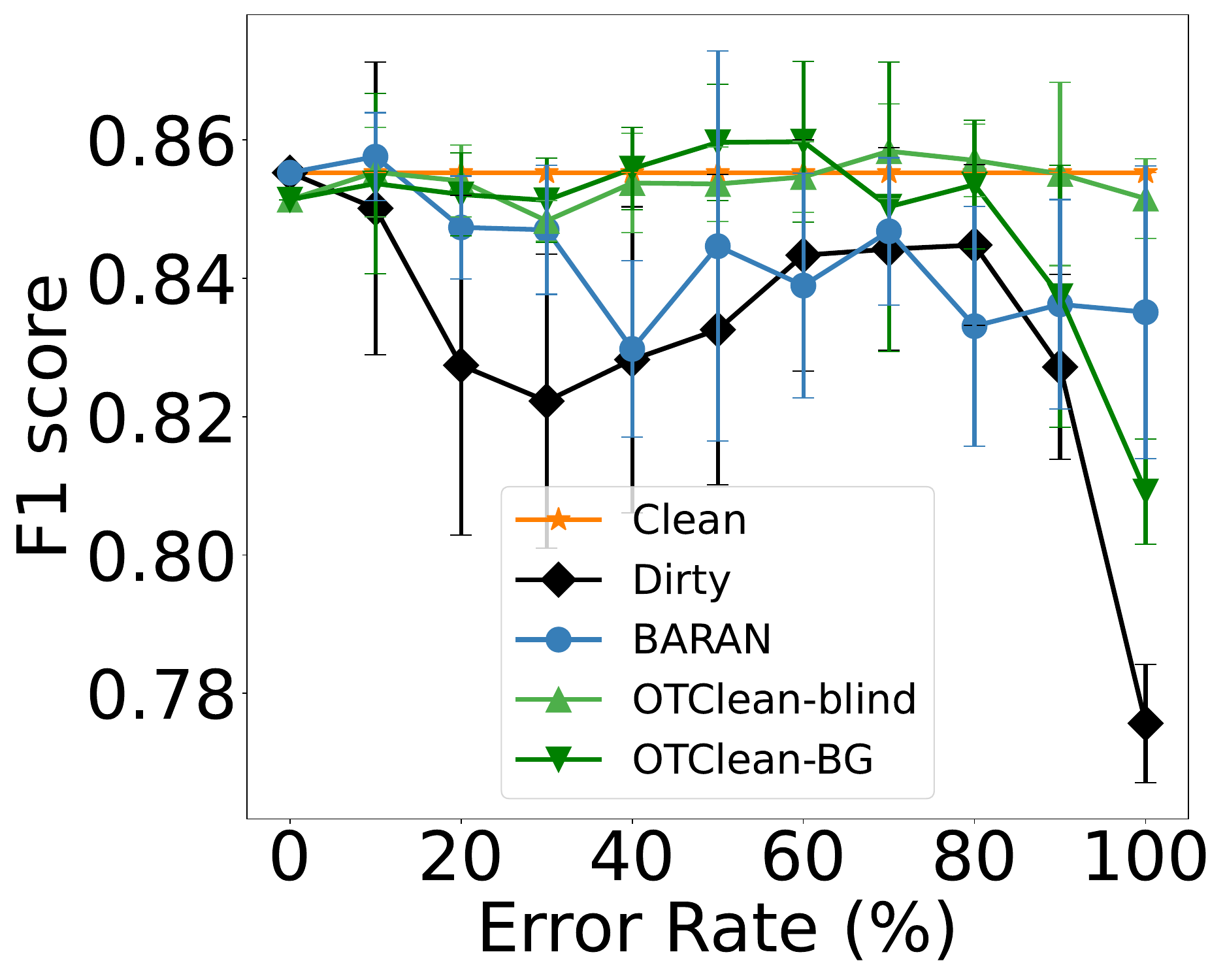}
        \caption{\car}
        \label{fig:car-f1-noise}
    \end{subfigure}
    \hfill
    \begin{subfigure}{0.22\textwidth}
        \centering
        \includegraphics[width=\linewidth]{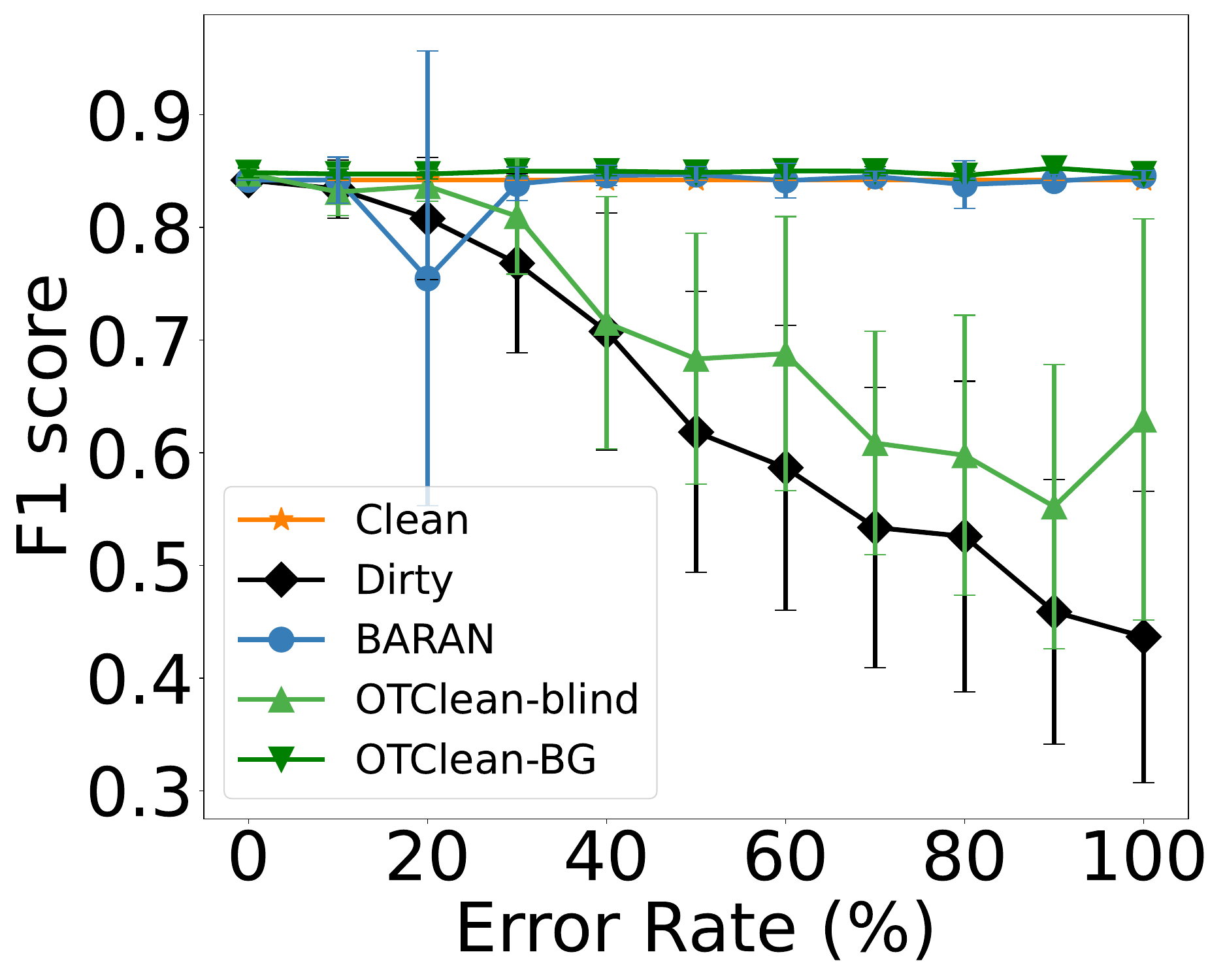}
        \caption{\boston}
        \label{fig:boston-f1-noise}
    \end{subfigure}
    \vspace{-3mm}
    \caption{Attribute noise}
    \label{fig:noise}
\end{figure}

\paragraph{\bf Attribute Noise}  
%\babak{I think you don't have to detail the mechanism you use to generate data. This could go to the extended version. We can just say we injected error such as it lead to violation and then give a general idea why this error make sense in practice}
%To introduce attribute noise in \car, we inject noise into the non-predictive attribute 'door,' where the noise is proportional to the label 'class.' Similarly, in \boston, the noise is proportional to the label 'medv' and is added to the non-predictive attribute 'B.' Both scenarios resemble real-world instances of data distortion due to sampling biases, historical biases, noise, or errors. For example, in \car, a dataset may be collected from a car dealership that maintains 2-door cars in better conditions, leading to a higher likelihood of these cars being labeled as 'good' or 'very good' in the 'class' attribute. Similarly, historical biases, such as a trend where houses in neighborhoods with a higher percentage of black residents tend to have lower prices, can contribute to the kind of data biases we simulate in our \boston experiment.

\reviewerone{Figure~\ref{fig:noise} shows our results for cleaning data with attribute noise. We compared the performance, in terms of AUC and F1-score, of models using ``Clean'' data, ``Dirty'' data, and data cleaned by \sys and Baran. Our cleaning algorithm only applies the CI constraint and does not need prior information about the noise type. However, it can also use knowledge about which attribute is noisy for repair. We tested \sys in two ways: ``blind'', without knowing the noisy attribute, and with background knowledge (BG), where the noisy attribute is identified. The figures show how accuracy changes with different levels of noise. As noise increases, the model trained on dirty data performs worse. In contrast, the model trained on \sys-cleaned data in both scenarios closely matches the ground truth model's behavior. This is because the dirty data model might learn false patterns not present in clean test data. However, using \sys to apply the CI constraint helps the model focus on the correct data patterns. While \sys improves accuracy in both the blind and BG-informed settings, using background knowledge generally leads to better performance than the blind approach and Baran.}

\begin{figure*}[h]\centering
    \begin{subfigure}{0.22\textwidth}
        \centering
        \includegraphics[width=\linewidth]{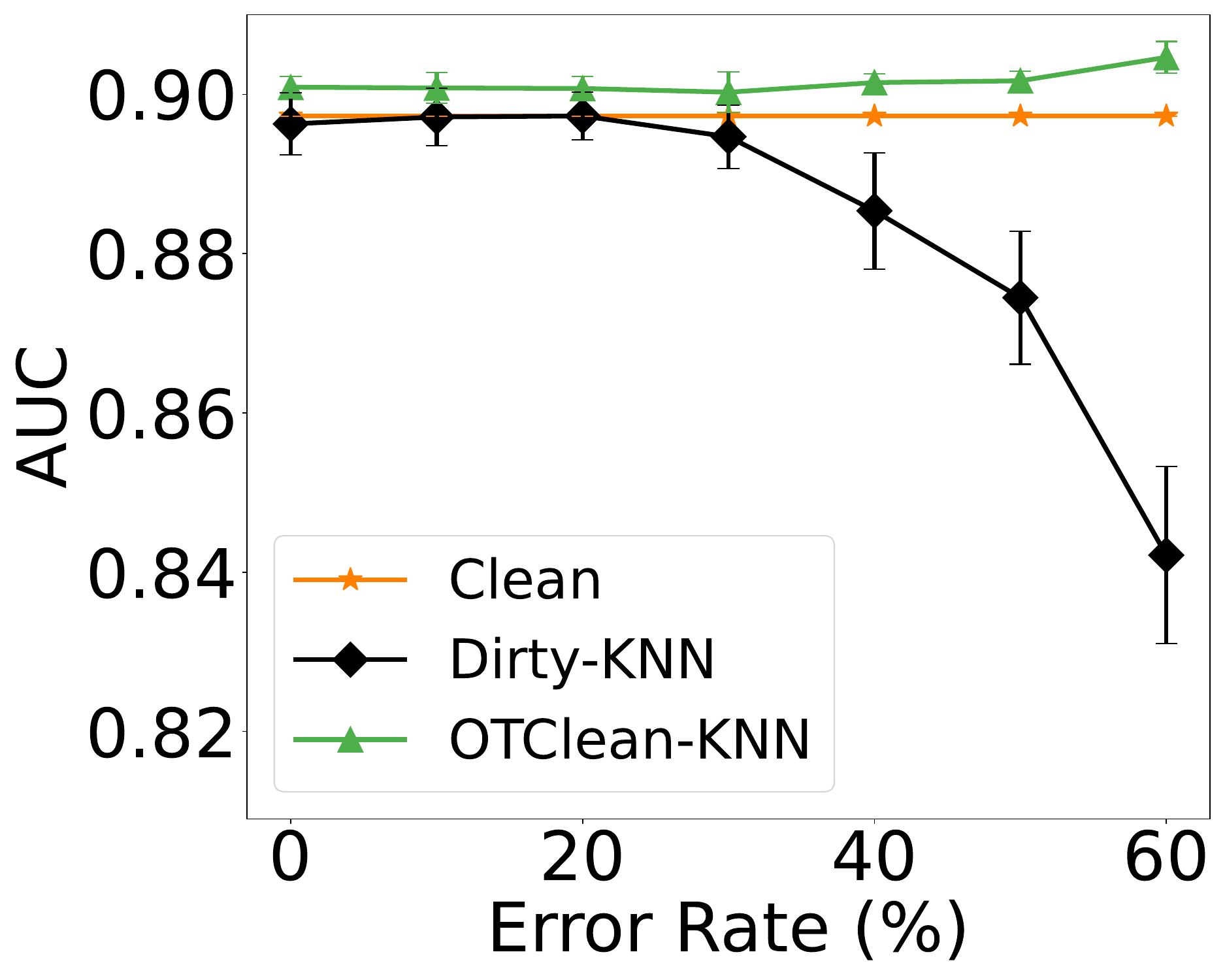}
        \caption{kNN Imputation}
        \label{fig:knn-mar-boston}
    \end{subfigure}
    \hfill
    \begin{subfigure}{0.22\textwidth}
        \centering
        \includegraphics[width=\linewidth]{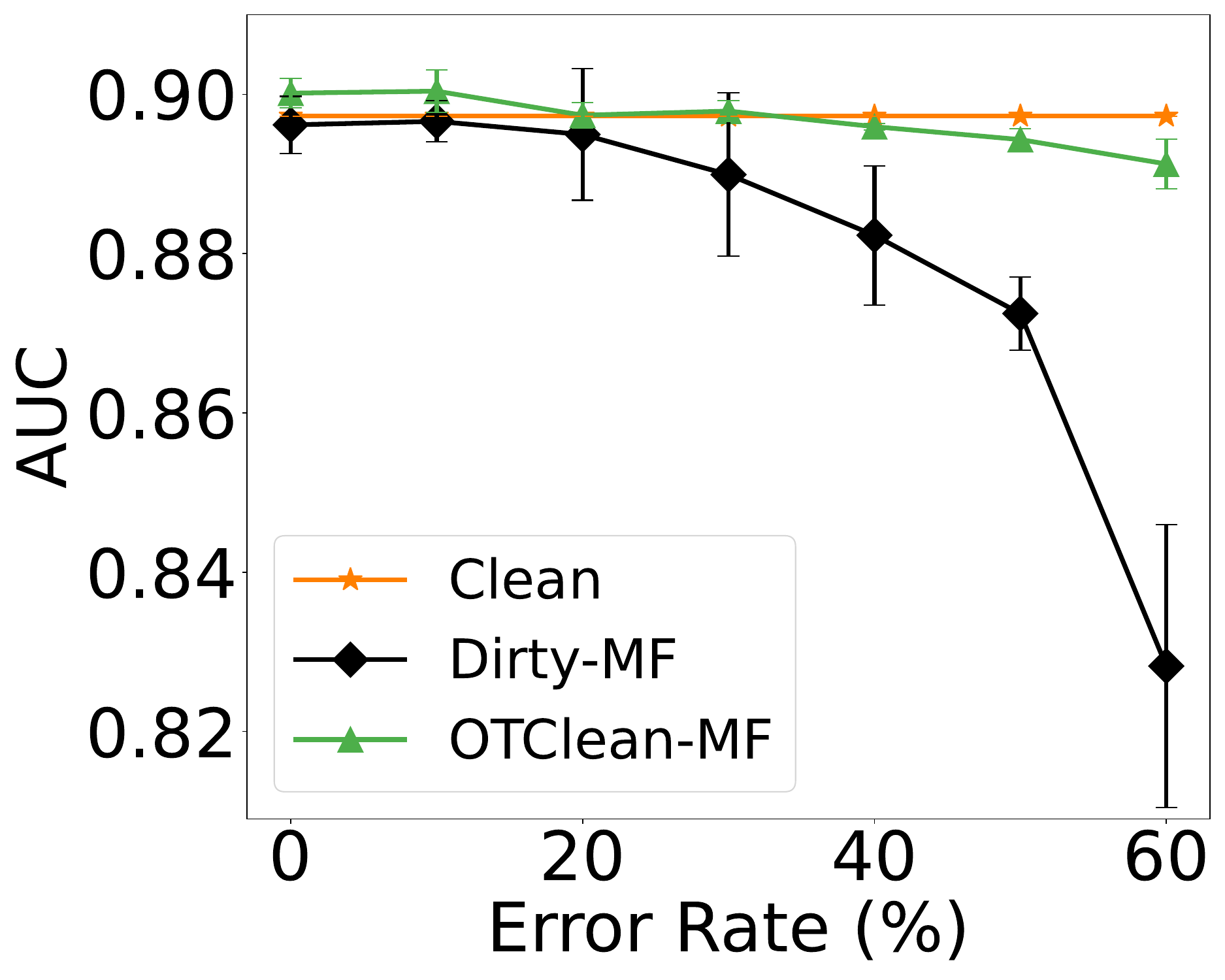}
        \caption{Most Frequent}
        \label{fig:mf-mar-boston}
    \end{subfigure}
    \hfill
    \begin{subfigure}{0.22\textwidth}
        \centering
        \includegraphics[width=\linewidth]{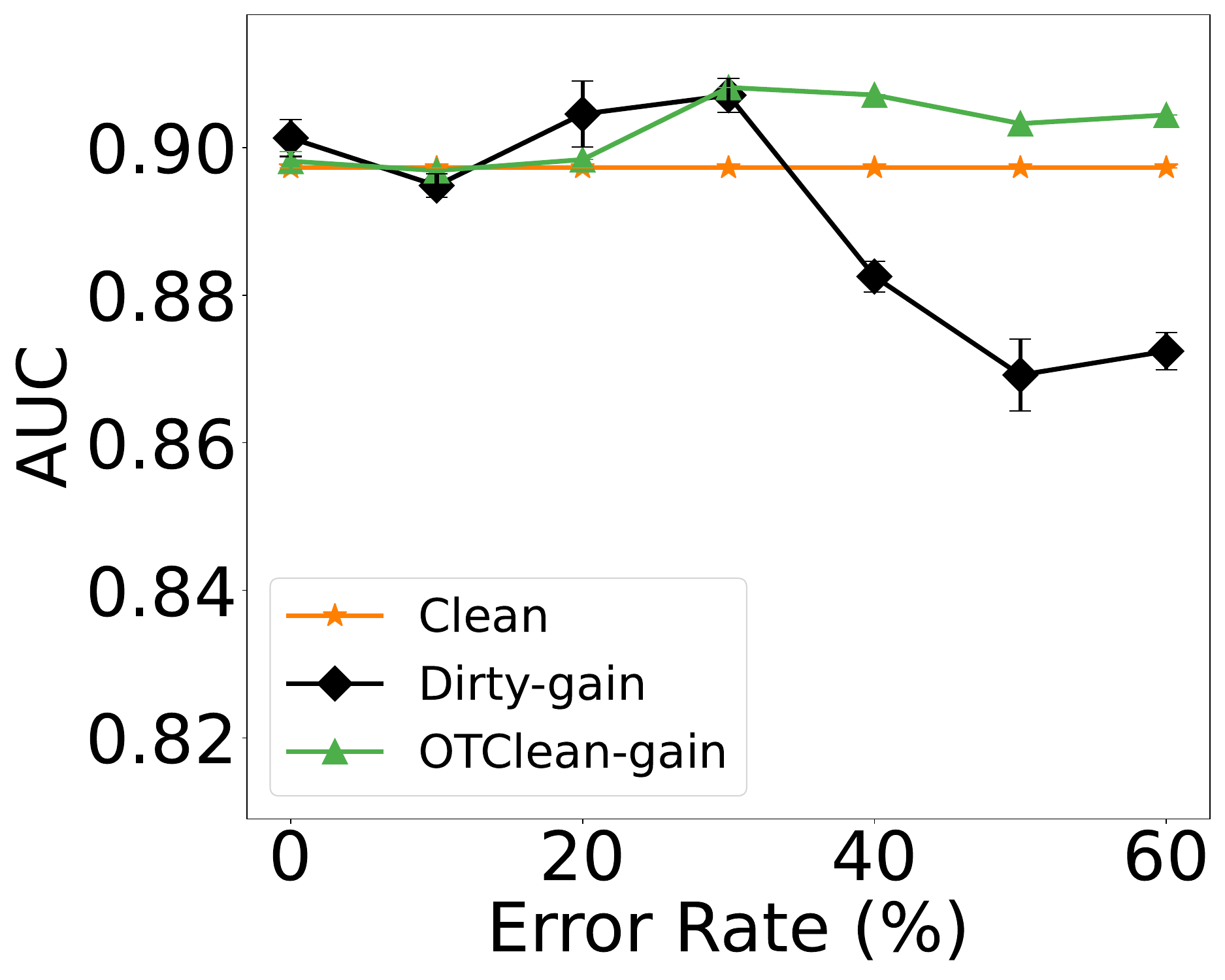}
        \caption{GAIN}
        \label{fig:gain-mar-boston}
    \end{subfigure}
    \hfill
    \begin{subfigure}{0.22\textwidth}
        \centering
        \includegraphics[width=\linewidth]{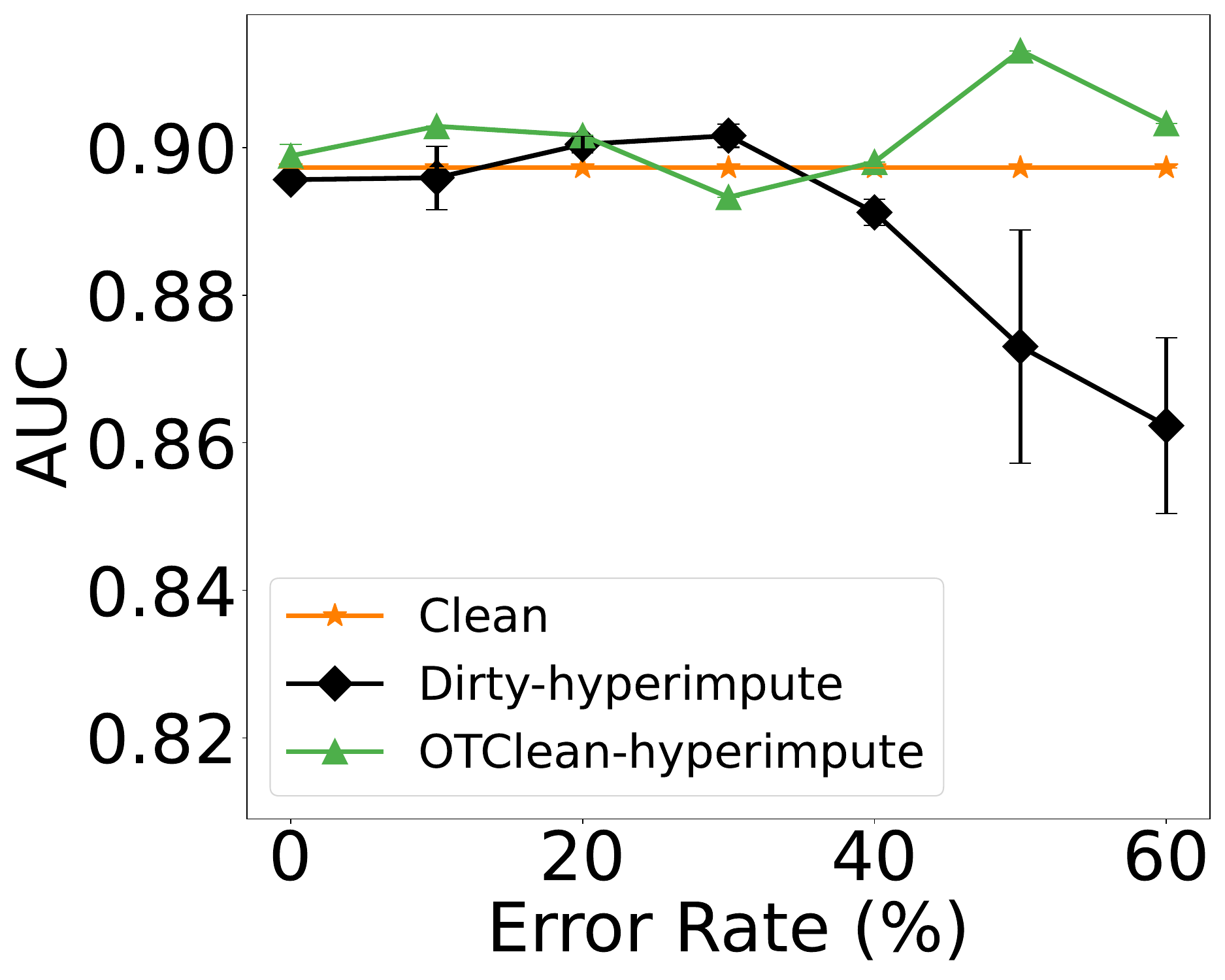}
        \caption{Hyperimpute}
        \label{fig:hyper-mar-boston}
    \end{subfigure}
    \vspace{-3mm}
    \caption{Missing at random (MAR) in \boston dataset}
    \label{fig:mar}
\end{figure*}
\begin{figure*}[h]
    \centering
    \begin{subfigure}{0.22\textwidth}
        \centering
        \includegraphics[width=\linewidth]{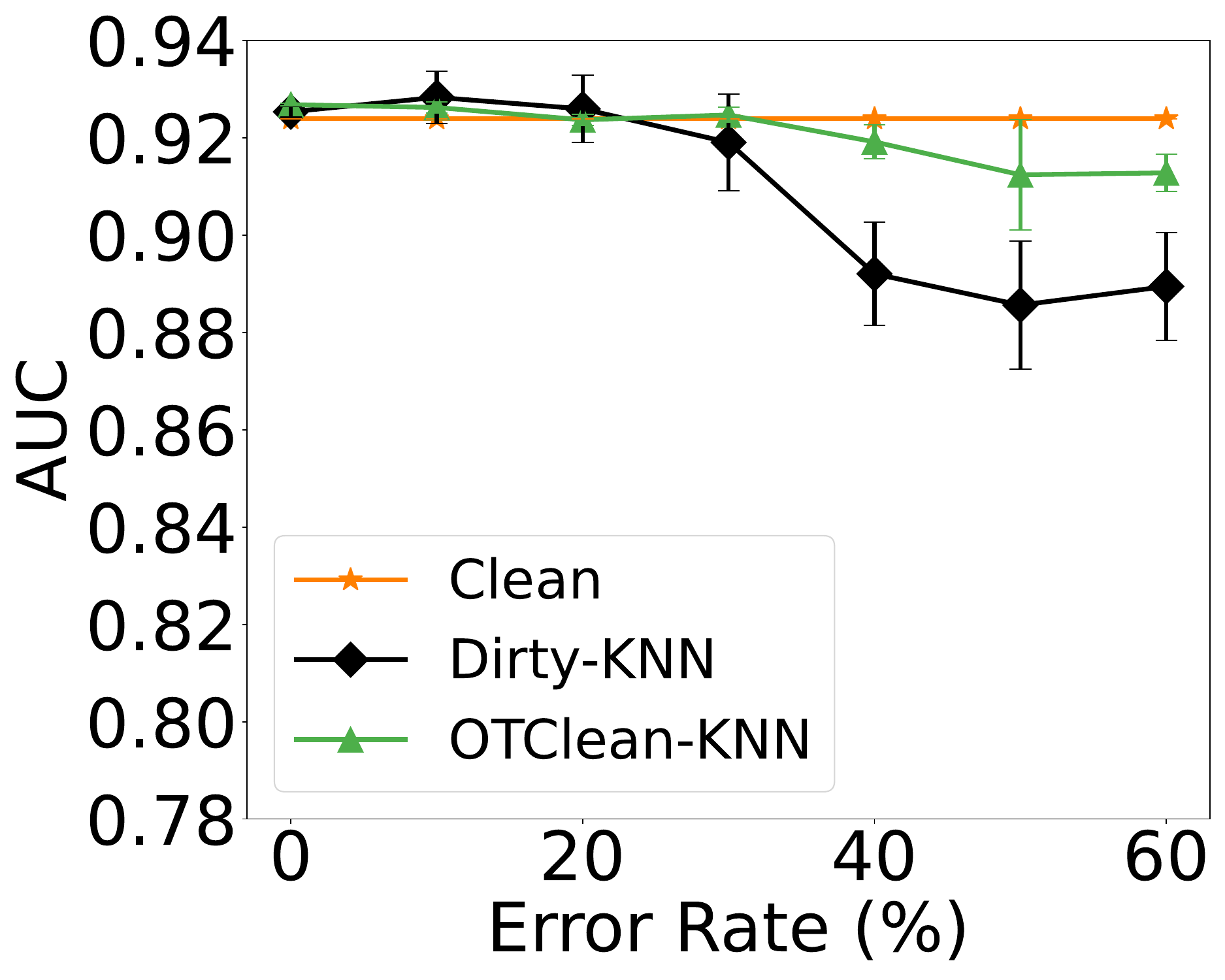}
        \caption{kNN Imputation}
        \label{fig:knn-mnar-car}
    \end{subfigure}
    \hfill
    \begin{subfigure}{0.22\textwidth}
        \centering
        \includegraphics[width=\linewidth]{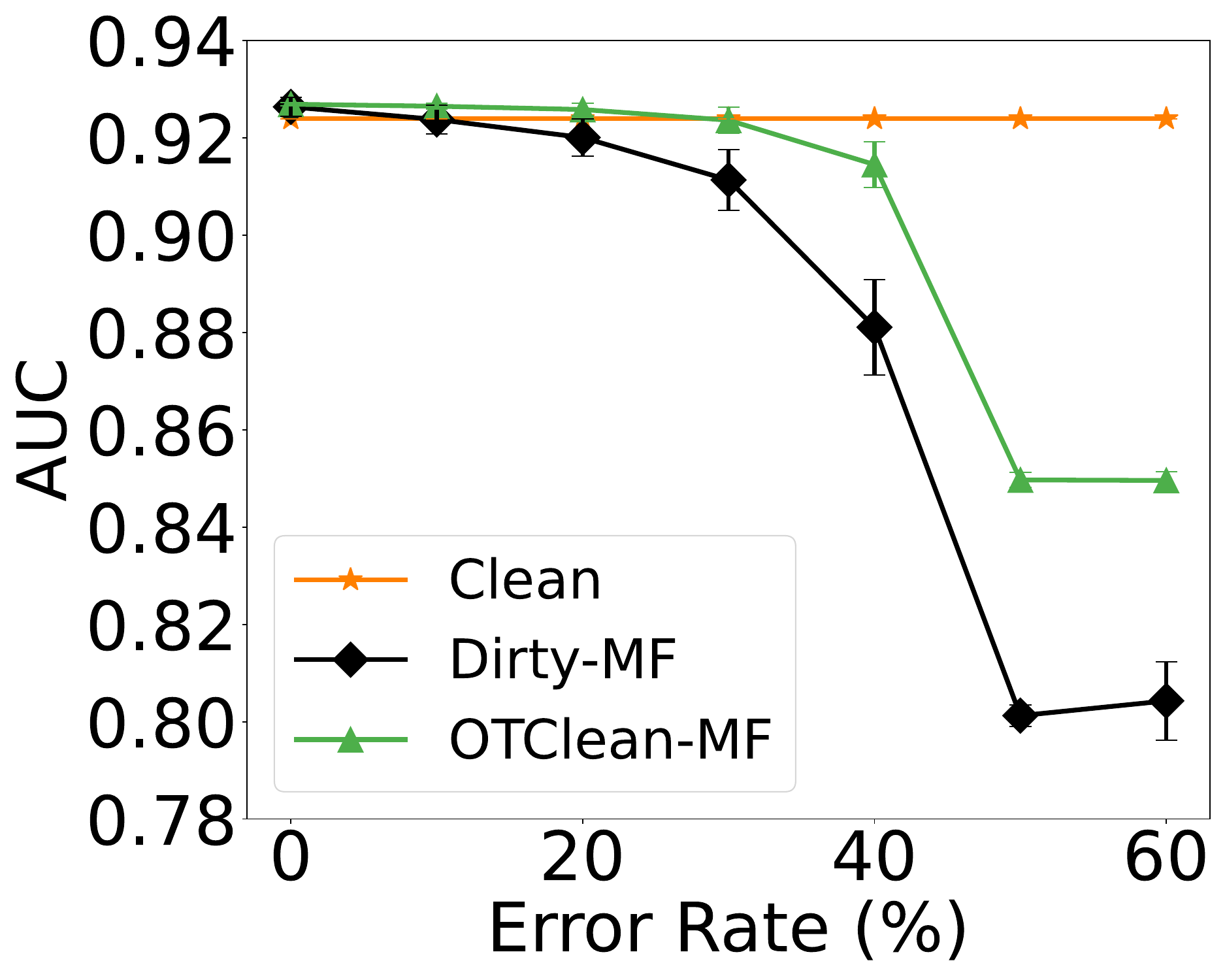}
        \caption{Most Frequent}
        \label{fig:mf-mnar-car}
    \end{subfigure}
    \hfill
    \begin{subfigure}{0.22\textwidth}
        \centering
        \includegraphics[width=\linewidth]{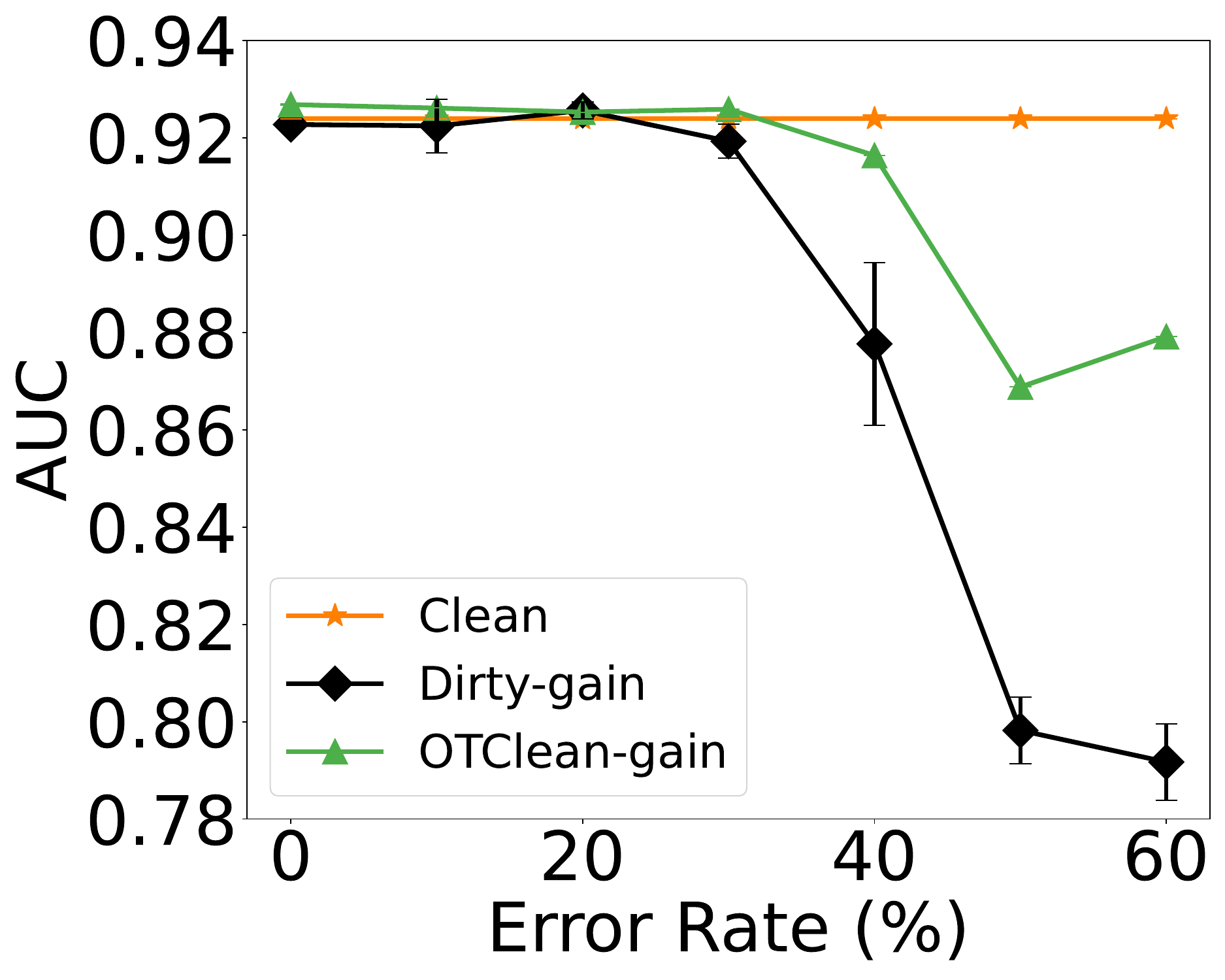}
        \caption{GAIN}
        \label{fig:gain-mnar-car}
    \end{subfigure}
    \hfill
    \begin{subfigure}{0.22\textwidth}
        \centering
        \includegraphics[width=\linewidth]{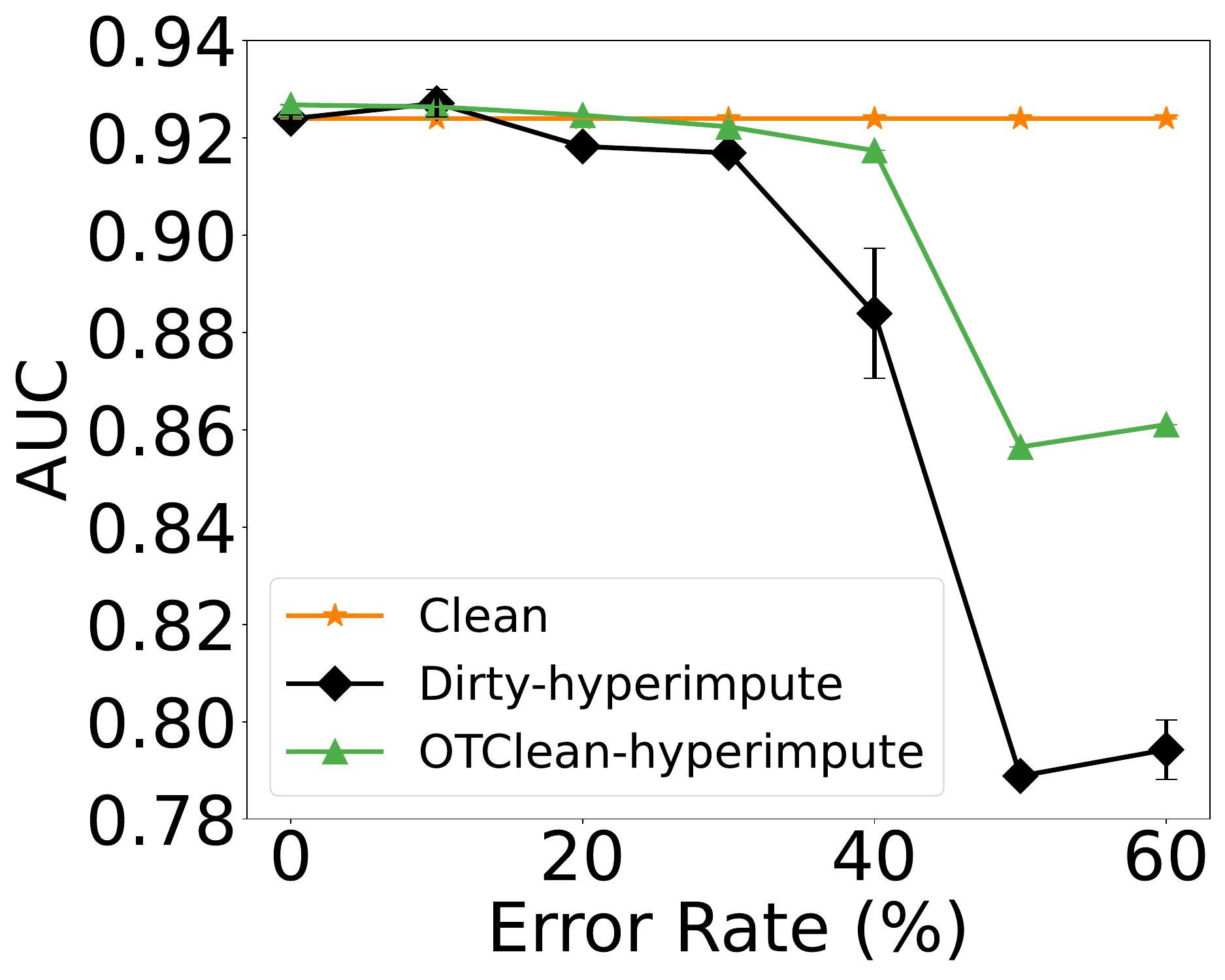}
        \caption{Hyperimpute}
        \label{fig:hyper-mnar-car}
    \end{subfigure}
    \vspace{-3mm}
    \caption{Missing Not at Random (MNAR) in \car dataset}
    \label{fig:mnar}
\end{figure*}

 \vspace{-0.3cm}
\paragraph{\bf Missing Values}  \reviewerone{In our missing value experiments (Figures~\ref{fig:mar} for MAR and \ref{fig:mnar} for MNAR), we tested model performance at different missing data levels. We compared ``Dirty'' models (trained with missing values filled using methods like MF, kNN, GAIN, and Hyperimpute) against \sys-enhanced models (\sys-MF, \sys-KNN, \sys-GAIN, and \sys-Hyperimpute). For MAR, all imputation methods struggled with high missing data rates, affecting performance. However, combining them with \sys improved results, closely matching the ground truth regardless of missing data amount. The slight advantage over ground truth models in Figure~\ref{fig:mar} is due to limited data size. For MNAR, as shown in Figure~\ref{fig:mnar}, our approach performed better than the baseline but declined as missing data increased. This is because MNAR issues are generally harder to address. While using \sys helps reduce false correlations, differences in training and test data distributions can still affect performance.}

%\subsubsection{Label Noise}

\reviewerone{\subsection{Evaluation using Statistical Distortion}\label{sec:distortion}}

\reviewerone{Dasu et al.~\cite{dasu2012statistical} proposed a way to evaluate data cleaning methods, focusing on how they statistically distort data. They used measurements like the Earth Mover Distance (EMD) to see how much a method changes the original data distribution; less change is better. Their approach starts with a dirty dataset and its cleaned version. Using sampling, they generate pairs of these datasets, called replications, and clean the dirty ones. Using several replications instead of a single dataset pair ensures a more comprehensive and robust evaluation, avoiding biases that might arise from the unique characteristics of a single dataset. They then measure how much these strategies alter the data and improve error correction.}

\reviewerone{In our experiments, we applied this framework to test \sys as a data cleaning method. We compared its effect on data distortion to other methods. Instead of looking at repaired errors, we focused on the accuracy (AUC) using the cleaned data. We ran 100 replications with attribute noise. The results are in Figure~\ref{fig:distortion}, where each cluster represents a cleaning method (the black point shows the original dirty data). Each point shows the balance between data distortion and AUC improvement for a replication. The figure indicates that \sys generally improves performance more than Baran in most cases and is closer to the clean datasets, though with a bit more distortion. This increased distortion is due to moving the data closer to the ideal clean dataset, leading to better accuracy.}

\begin{figure}[h]
    \centering
    \includegraphics[width=0.50\linewidth]{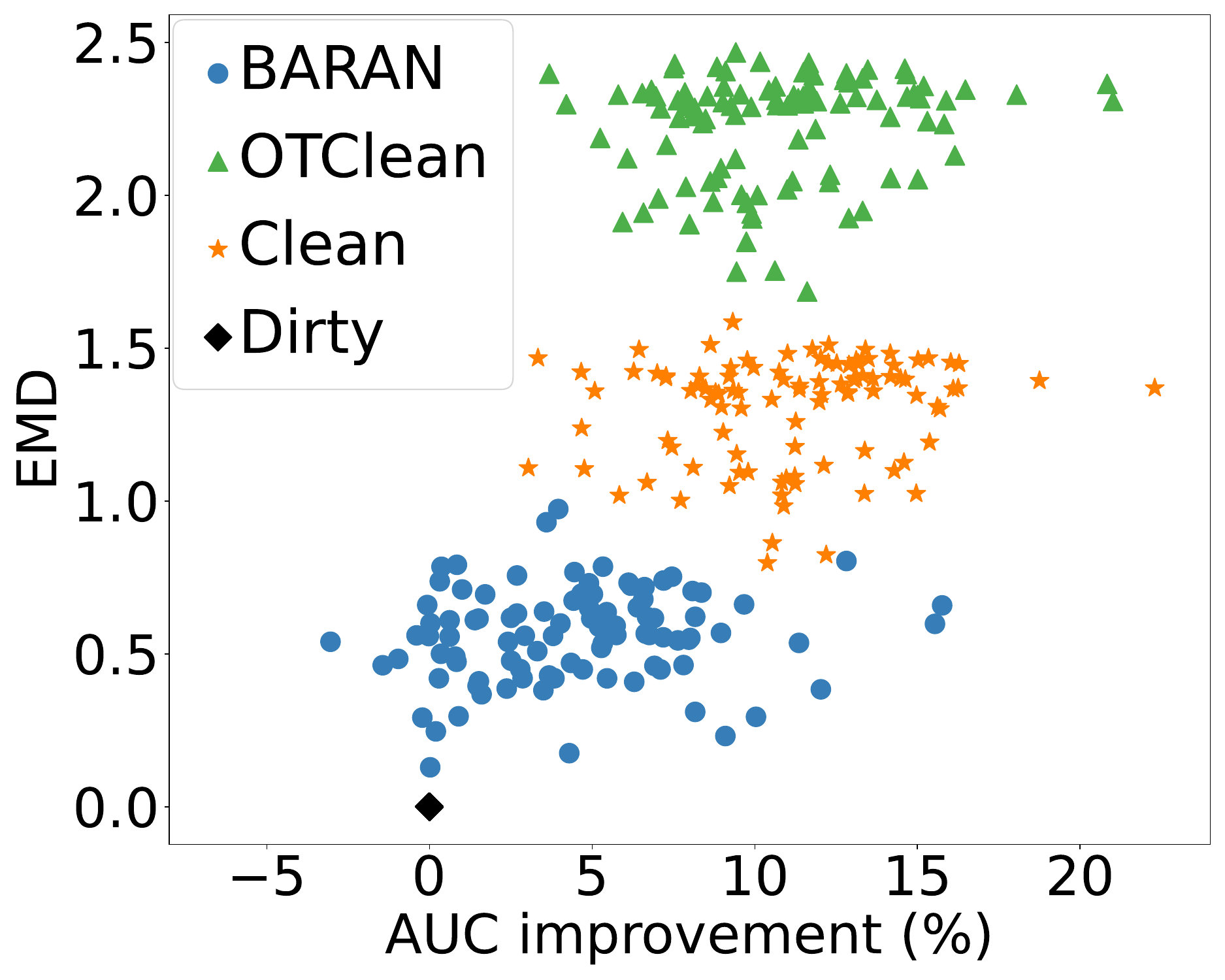}
 
    \caption{Comparing \sys and the competing cleaning methods based on their statistical distortion}
    \label{fig:distortion}
\end{figure}

\begin{table}[h]
    \centering
     \resizebox{0.45\textwidth}{!}{
    \begin{tabular}{*{7}{l}} 
    \toprule
        \textbf{Dataset} &
        \textbf{\mainAlg-C1} & 
        \textbf{\mainAlg-C2} & 
        \textbf{MF} & 
        \textbf{IC} & 
        \textbf{MS} & 
        \textbf{QCLP} \\
    \midrule
        \adult &
        1229   & 
        1279   & 
        66 & 66 & 700 & NA \\
        \compas &
        848   & 
        337  & 
        7 & 6 & 1227 & 2 \\
    \bottomrule
    \end{tabular}
     }
    \caption{Runtime \reviewerone{(sec)} for the fairness application}

    \label{tab:runtimes}
\end{table}

\subsection{\sys's Runtime and Performance}\label{sec:sys-exp}

\paragraph{\bf Runtime} 
In Table~\ref{tab:runtimes}, we provide the runtime results of \mainAlg for \adult and \compas datasets, comparing them with the baselines. While our algorithm's runtime is somewhat higher due to the complex nature of optimal transport, it remains reasonably fast and offers a practical means for employing optimal transport in data cleaning for CI constraints. \reviewerthree{Our algorithm's runtime is mainly influenced by the number of attributes in the CI constraints rather than the data size. This is because the size of the transport plan we use stays the same no matter how large the data is; it only changes based on the number of attributes. In our experiments, the key factor is the domain size, determined by the number of attributes in the CI constraints and the range of values these attributes can take. Figure~\ref{fig:perf-rt-mem} illustrates \mainAlg's runtime and memory usage for methods on the \adult for increasing domain size, showing that it can scale effectively to high-dimensional CIs. Furthermore, memory requirements and runtime can be further reduced by implementing a sparse representation for the transport plan.} % since the transport plan is inherently very sparse. This is an area we plan to explore in future work.

\paragraph{\bf Convergence and Optimization} Figure~\ref{fig:converg} demonstrates the convergence behavior of our main \mainAlg, affirming the result presented in Theorem~\ref{th:correct}. It shows the monotonic decrease of the objective function, which represents the cost of the transport plan with the number of iterations. Additionally, the graph compares the convergence properties of \mainAlg with two different initializations: one with a random initialization of $\vt{q}$ and another using NMF. Notably, initializing with NMF reduces the total convergence iterations by nearly 30\%.  We also highlight optimizations aimed at reducing runtime. The first optimization involves updating 
$q$ slices in parallel, achieving a significant speedup of \ $\times7$ in our \adult data. Another optimization focuses on unsaturated CIs. Figure~\ref{fig:unsaturated-time} illustrates the substantial runtime improvement achieved by employing the proposed optimization for unsaturated CI constraints while maintaining the same outcome. In this scenario, we initiate with a CI constraint and construct \(\st{W}\) using attributes with varying domain sizes. We then evaluate the runtime of both the naive and saturation approaches. The saturation approach consistently solves the same problem, optimizing \(\pi_s\), regardless of growing \(\pi\)'s size, contributing to its stable performance.  In our final experiment, we investigate the impact of warm start optimization on Sinkhorn iteration numbers. Figure~\ref{fig:warm-start} shows warm start reduces the number of iterations by more than sevenfold.

\begin{figure}[h]
    \centering
    \begin{subfigure}{0.22\textwidth}
        \centering
        \includegraphics[width=\linewidth]{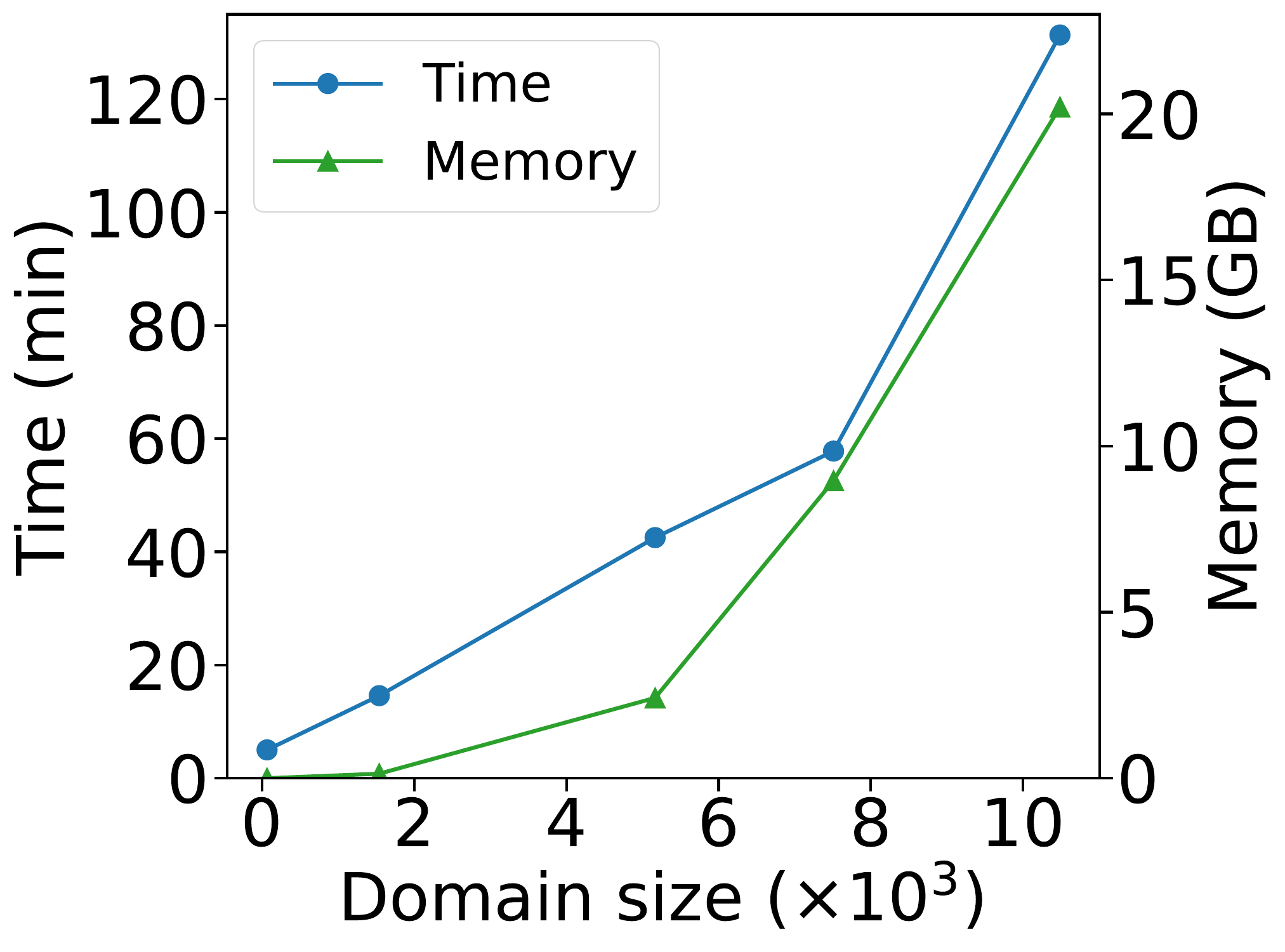}
        \caption{Runtime and Memory Usage}
        \label{fig:perf-rt-mem}
    \end{subfigure}
    \hfill
    \begin{subfigure}{0.22\textwidth}
        \centering
        \includegraphics[width=\linewidth]{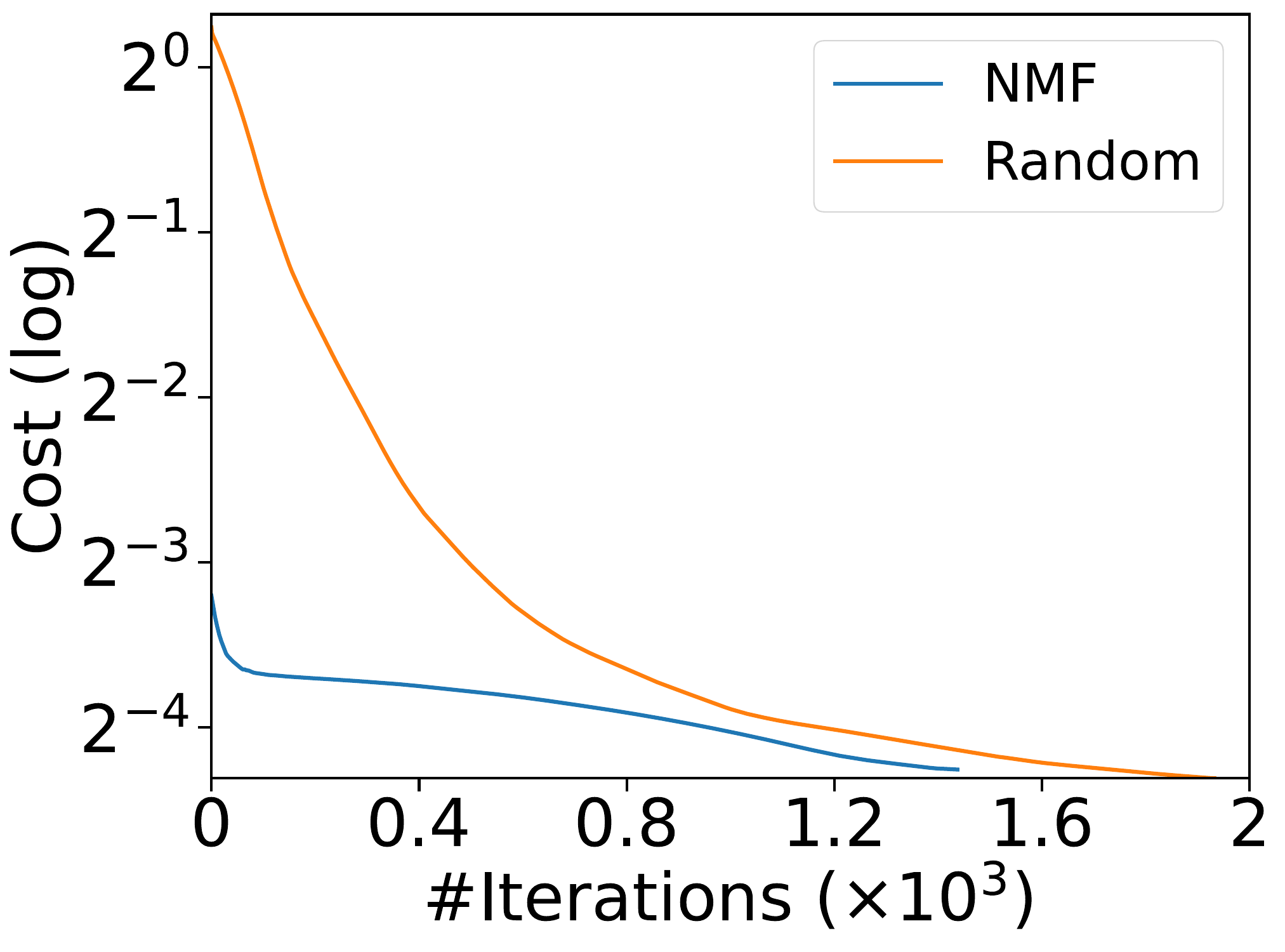}
        \caption{Convergence}
        \label{fig:converg}
    \end{subfigure}

    \caption{\sys' performance}
    \label{fig:perf}
\end{figure}

%\subsection{Optimization Impact on \mainAlg}\label{sec:sys-exp-opt}
%\paragraph{Optimizations}

\begin{figure}[h]
    \centering
    \begin{subfigure}{0.22\textwidth}
        \centering
        \includegraphics[width=\linewidth]{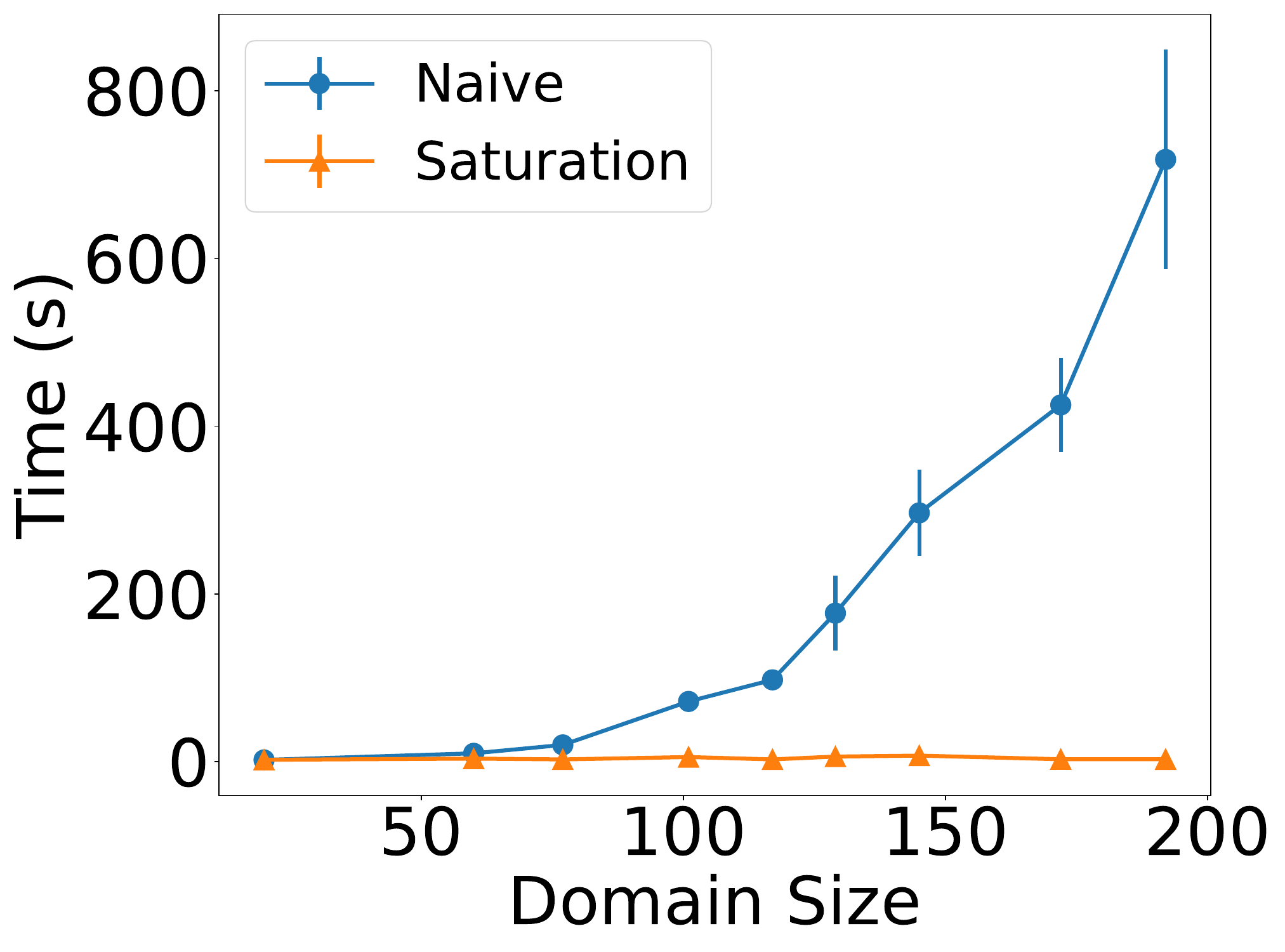}
        \caption{Saturation}
        \label{fig:unsaturated-time}
    \end{subfigure}
    \hfill
    \begin{subfigure}{0.22\textwidth}
        \centering
        %\includegraphics[width=\linewidth]{figs/unsaturated_QCLP_max_memory.pdf}
        %\caption{}
        %\label{fig:unsaturated-memory}
        \includegraphics[width=\linewidth]{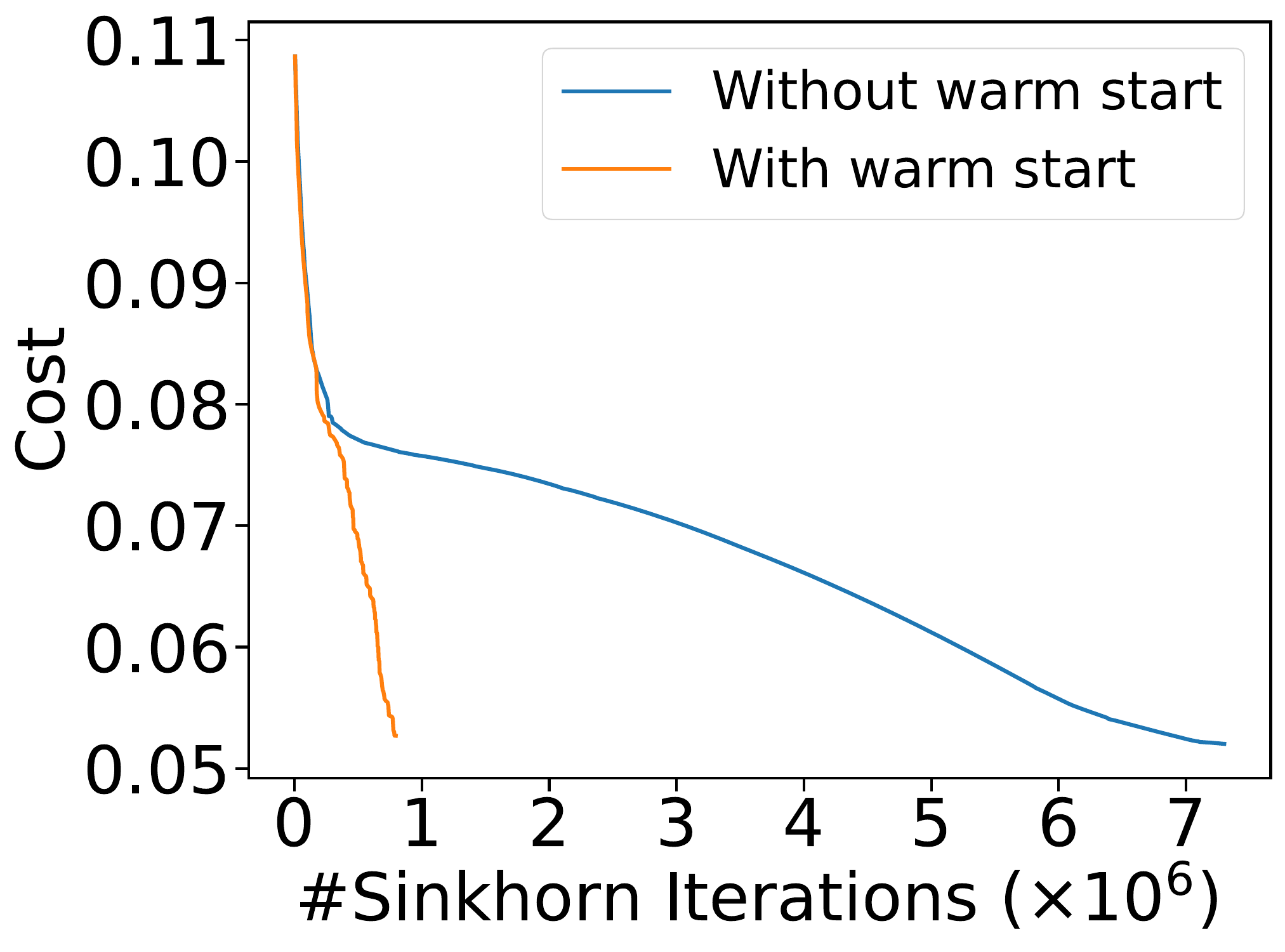}
        \caption{Warm start}
        \label{fig:warm-start}
    \end{subfigure}

    \caption{\sys's Optimizations}
    \label{fig:unsaturated}
\end{figure}

%% file: related_work.tex
\section{Related Work}

Our research connects with two main areas of study.
\vspace{-2mm}
\paragraph{\bf Data Cleaning for Conditional Independence}
Data cleaning in the database domain traditionally revolves around enforcing integrity constraints, such as functional dependencies and conditional functional dependencies~\cite{livshits2020computing, kolahi2009approximating, bohannon2006conditional, DBLP:conf/icdt/LivshitsK21}. Nonetheless, the domain of data cleaning for conditional independence has only recently gained attention. Notable works in this emerging field include~\cite{salimi2019interventional} and \cite{yan2020scoded}. SCODED~\cite{yan2020scoded} employs statistical constraints to detect errors within datasets but primarily focuses on ranking individual data tuples based on their relevance to conditional independence violations, differing from our data-centric approach. On the other hand, \cite{salimi2019interventional} aims to find optimal repairs for conditional independence violations, involving the addition or removal of tuples to satisfy the constraint. However, their method lacks the application of specific statistical divergence or distance measures to assess the quality of the repaired data. In a somewhat distinct vein,~\cite{ahuja2021conditionally} utilizes generative adversarial networks (GANs) to generate data adhering to conditional independence constraints. Their primary objective is to train these generative models effectively, particularly emphasizing the minimization of Jensen–Shannon divergence in continuous data. However, their focus is on training generative models rather than cleaning existing data.
\vspace{-2mm}
\paragraph{\bf Fairness and Optimal Transport}
Algorithmic fairness research has primarily focused on detecting and mitigating biases in machine learning models, utilizing pre-, post-, and in-processing techniques. Pre-processing methods~\cite{Caton2020FairnessIM}, aim to eliminate bias from training data before model training. While model-agnostic approaches such as~\cite{NIPS2017_6988, feldman2015certifying, salimi2019interventional} exist, they often lack insights into the root causes of biases. These strategies typically address basic fairness criteria and may not delve into enforcing conditional independence tests or incorporating optimal transport methods. Notably,~\cite{gordaliza2019obtaining} employs the Wasserstein barycenter for pre-processing training data to achieve statistical parity but doesn't specifically address the complexities of conditional statistical inference in high-dimensional datasets, distinguishing it from our approach. \cite{silvia2020general} employs optimal transport as a regularizer during ML model training, focusing on a different aspect than our data cleaning objective. Additionally, studies like~\cite{si2021testing, black2020fliptest} use optimal transport to quantify unfairness, making them less aligned with our core research goal.

%% file: conclusion.tex
%\vspace{-0.2cm}
\section{Conclusion}\label{sec:fw} In this paper, we introduced a principled approach for data cleaning under conditional independence constraints, harnessing optimal transport theory. Our results underscore the importance of prioritizing conditional independence in data pipelines for enhancing ML model robustness, reliability, accuracy, and fairness. Our techniques have demonstrated potential with discrete data, and we aim to further optimize and extend their applicability to continuous and relational data. Additionally, we plan to explore methods for enforcing multiple conditional independence constraints and capturing interactions between CIs and other database dependencies.

%% file: appx.tex
\section{Additional Experimental Results}

In this section, we present additional experimental findings that couldn't fit into the main paper due to space limitations.

\subsection{Impact of Cost Functions}

One of the key promises of \sys is that by using OT, we can incorporate a suitable cost function that allows our solution to tailor its data repair based on the types of errors. To validate this concept, we conducted experiments to examine the impact of different cost functions on the results of our data-cleaning application.

\begin{figure}[h]
    \centering
    \begin{subfigure}{0.22\textwidth}
        \centering
        \includegraphics[width=\linewidth]{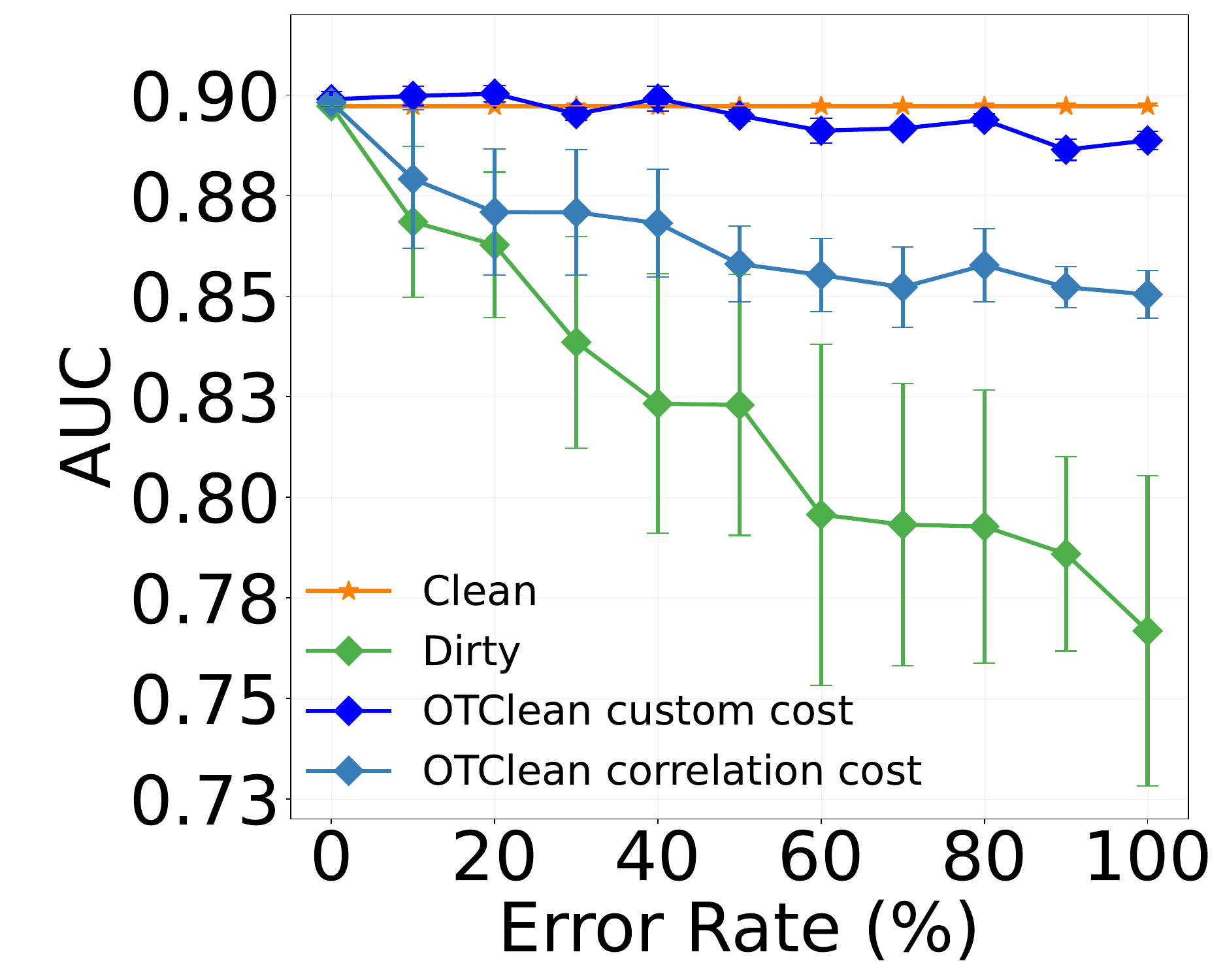}
        % correlation cost
        \caption{\boston}
        \label{fig:cost-boston}
    \end{subfigure}
    \hfill
    \begin{subfigure}{0.22\textwidth}
        \centering
        \includegraphics[width=\linewidth]{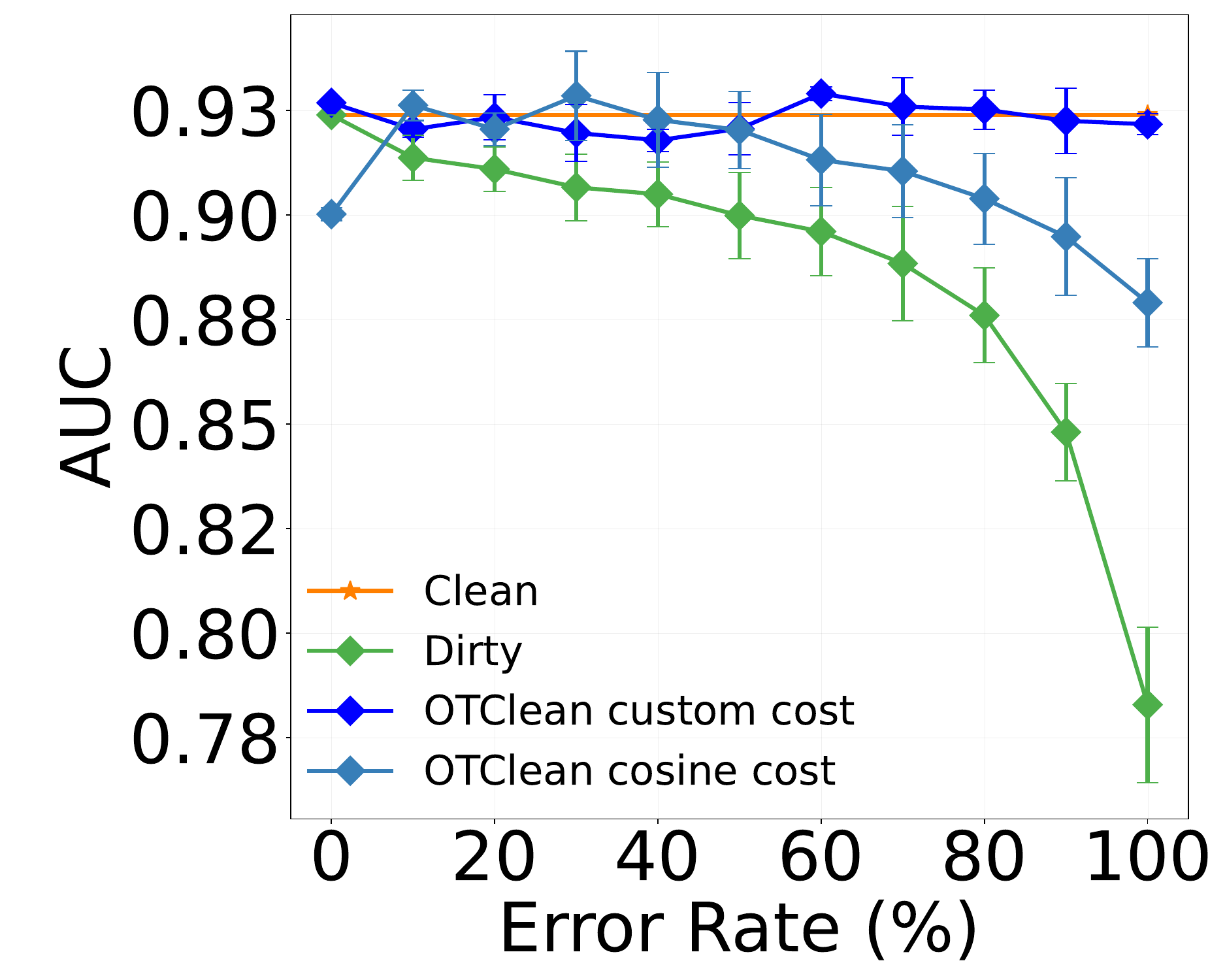}
        % cosin cost
        \caption{\car}
        \label{fig:cost-car}
    \end{subfigure}
    \centering
    \caption{Impact of cost on cleaning}
    \label{fig:cost-impact}
\end{figure}

In particular, we focused on attribute noises, and we illustrate the outcomes in Figures~\ref{fig:cost-impact}. We carried out the experiments for attribute noise, as discussed in Section~\ref{sec:exp-cleaning}, but using various cost functions. The outcomes are presented in Figures~\ref{fig:cost-boston} and~\ref{fig:cost-car}, for datasets \boston and \car. In these experiments, we explored a user-defined or ``custom'' cost function that aligns with the noise introduction process. This custom cost function assigns lower repair costs to noisy values when they are more likely to be corrected to their true values. This indicates that fixing values with higher probabilities of being correct is easier and incurs lower costs compared to other potential corrections. We compare this custom cost function, denoted as ``\sys custom cost'', with two alternative cost functions. One is based on cosine similarity, denoted by ``\sys cosine cost'', in the \boston dataset, while the other uses Pearson correlation, denoted by ``\sys correlation cost'', in the \car dataset. 

The results align with our expectations, demonstrating that a suitable custom cost function can come close to the clean data by effectively repairing noisy values to their true states. As shown in the figures, repairs made using the customer cost function result in a model that significantly outperforms the performance of the models using the other two general-purpose cost functions.

\begin{figure*}[h]
    \centering
    \begin{subfigure}{0.22\textwidth}
        \centering
        \includegraphics[width=\linewidth]{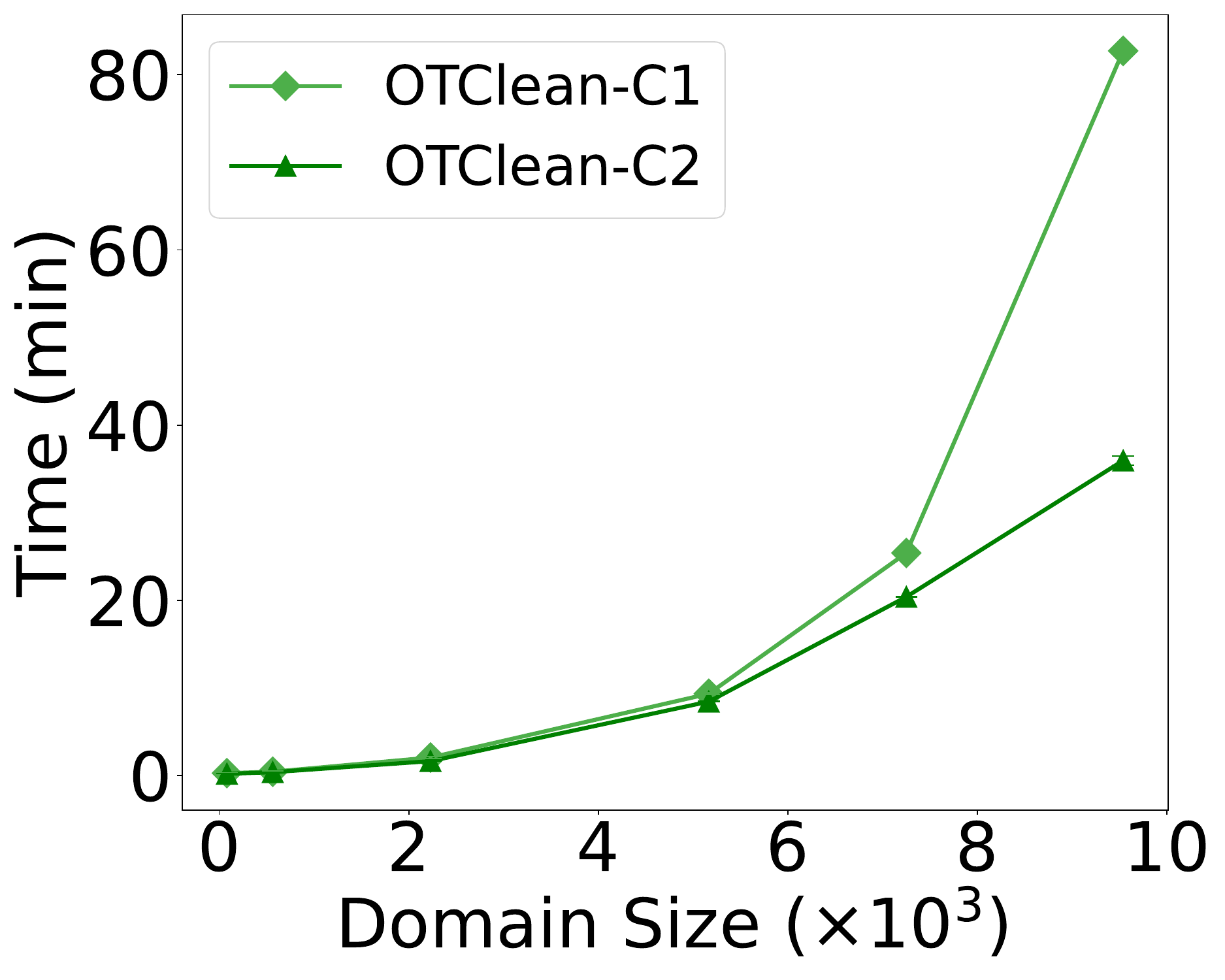}
        \caption{\adult}
        \label{fig:runtime-adult}
    \end{subfigure}
    \hspace{3mm}
    \begin{subfigure}{0.22\textwidth}
        \centering
        \includegraphics[width=\linewidth]{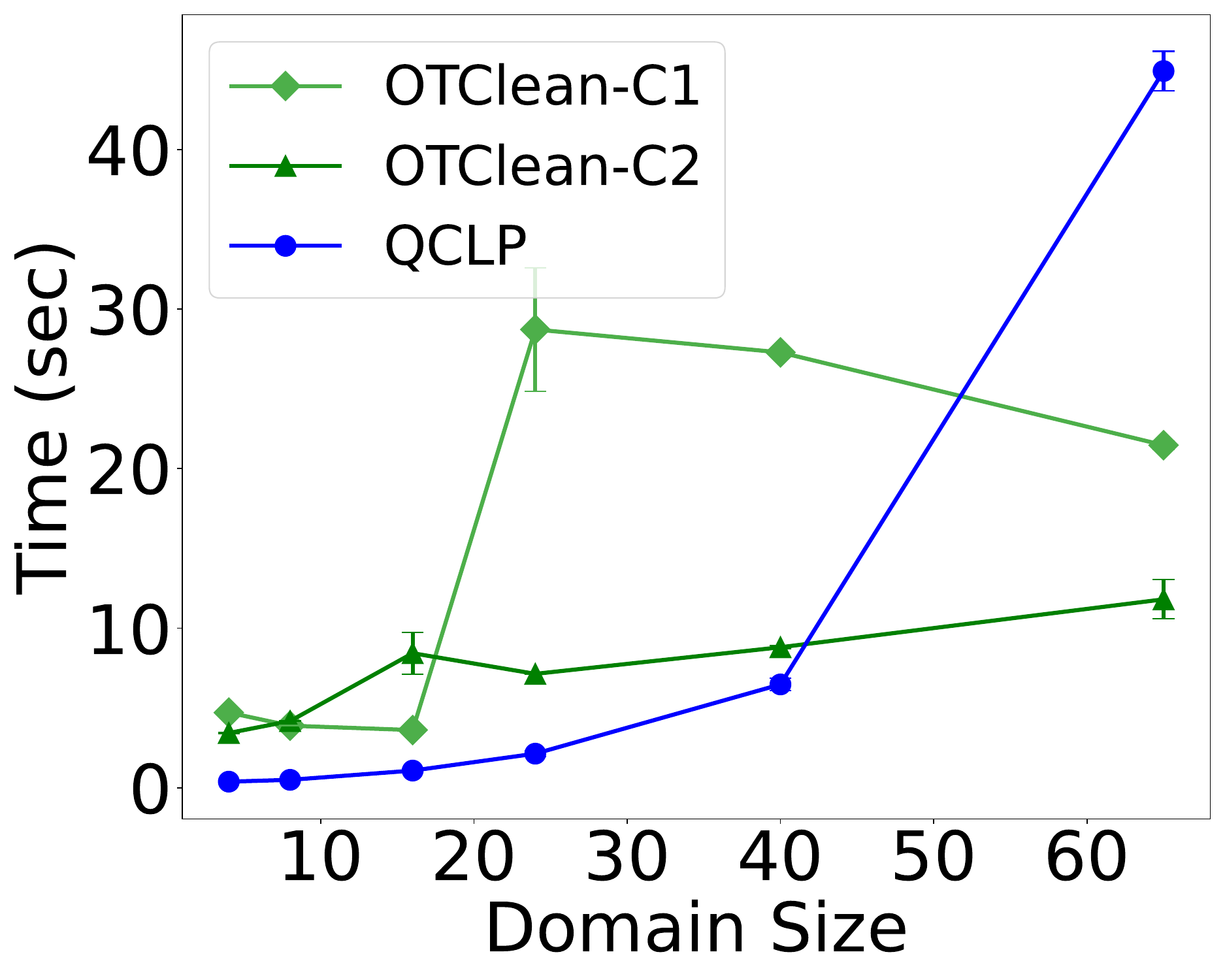}
        \caption{\compas with QCLP}
        \label{fig:runtime-compas-qclp}
    \end{subfigure}
    \hspace{3mm}
    \begin{subfigure}{0.22\textwidth}
        \centering
        \includegraphics[width=\linewidth]{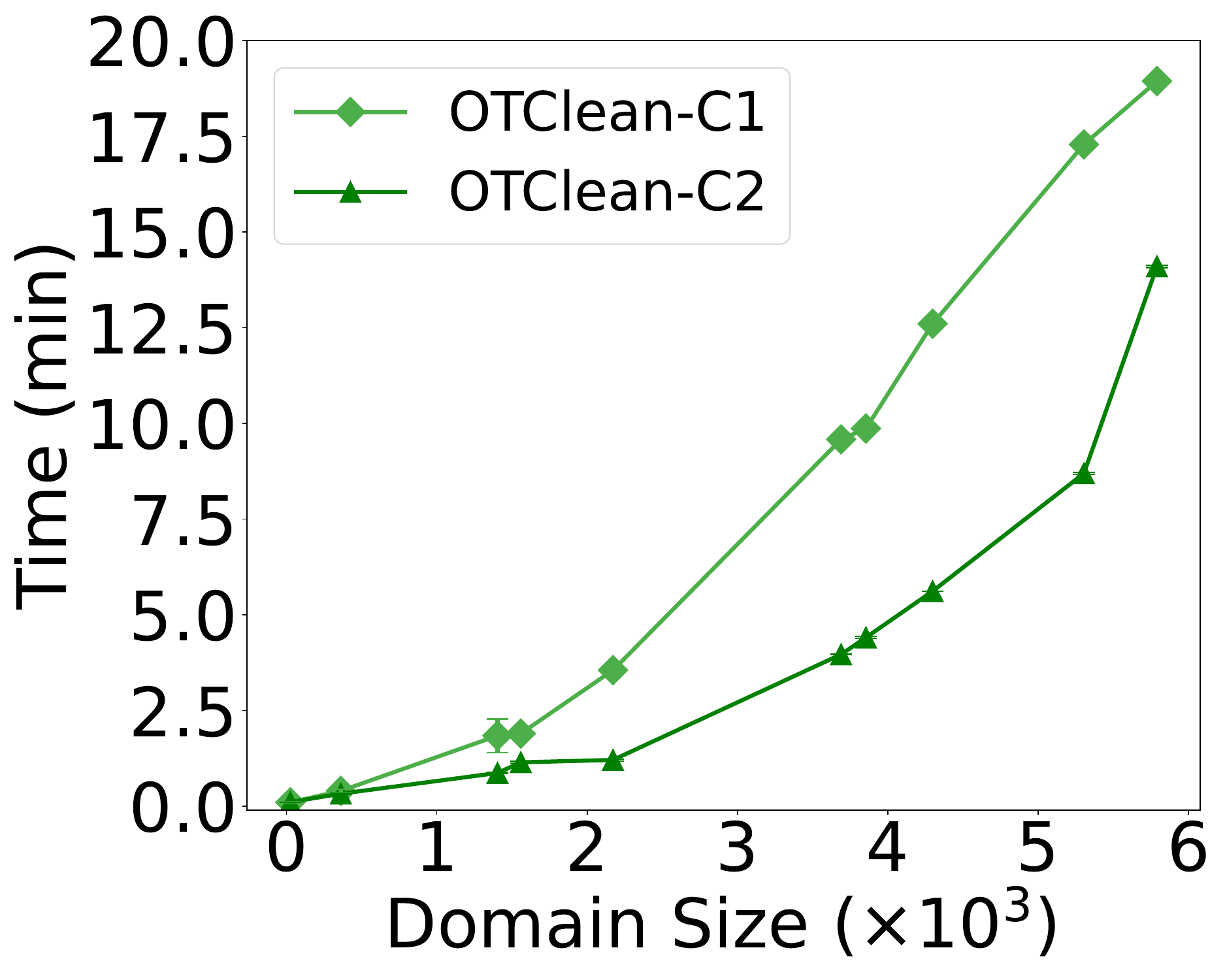}
        \caption{\compas}
        \label{fig:runtime-compas}
    \end{subfigure}
    \vspace{-3mm}
    \caption{Runtime}
    \label{fig:runtime}
\end{figure*}

\begin{figure*}[h]
    \centering
    \begin{subfigure}{0.22\textwidth}
        \centering
        \includegraphics[width=\linewidth]{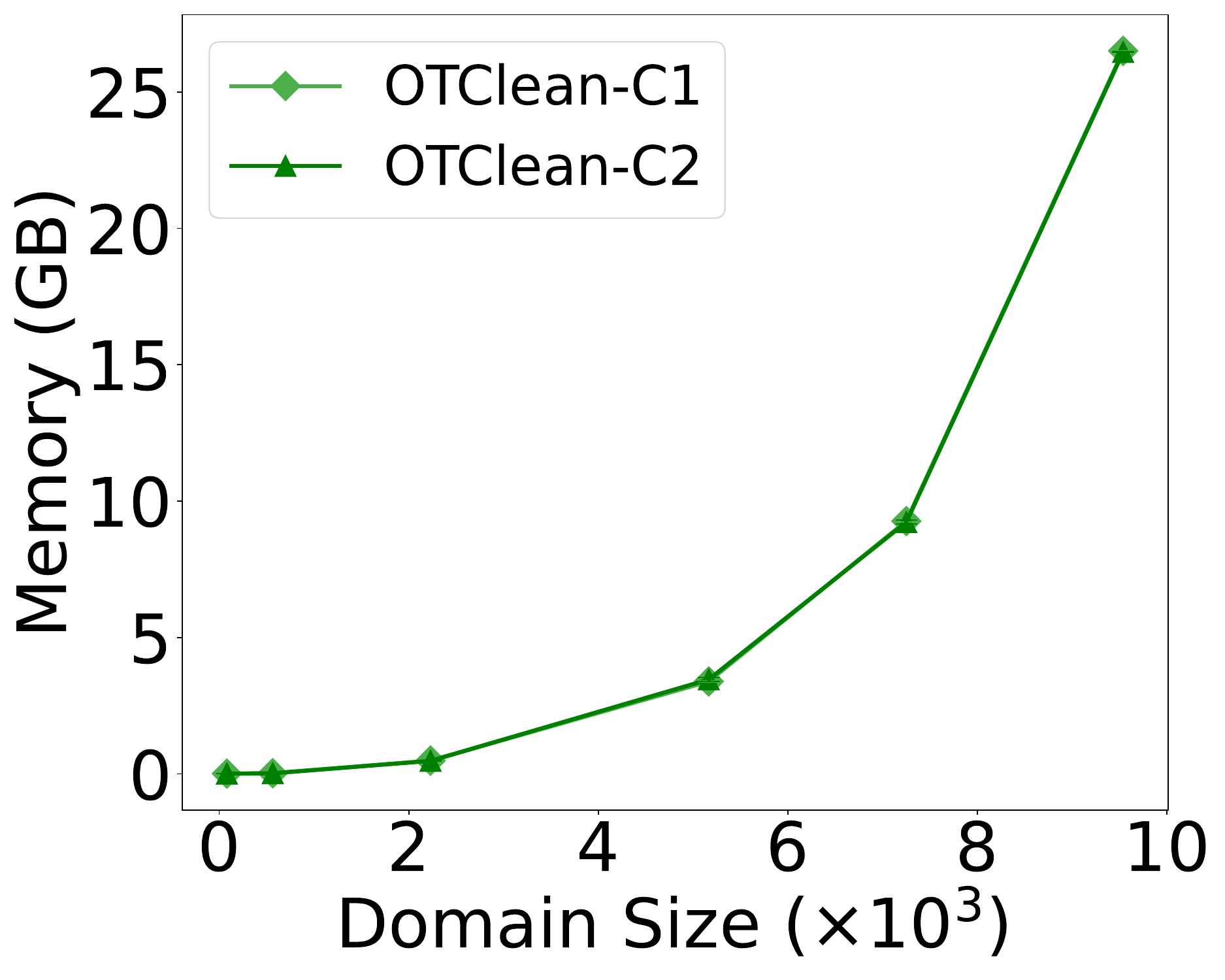}
        \caption{\adult}
        \label{fig:memory-adult}
    \end{subfigure}
    \hspace{3mm}
    \begin{subfigure}{0.22\textwidth}
        \centering
        \includegraphics[width=\linewidth]{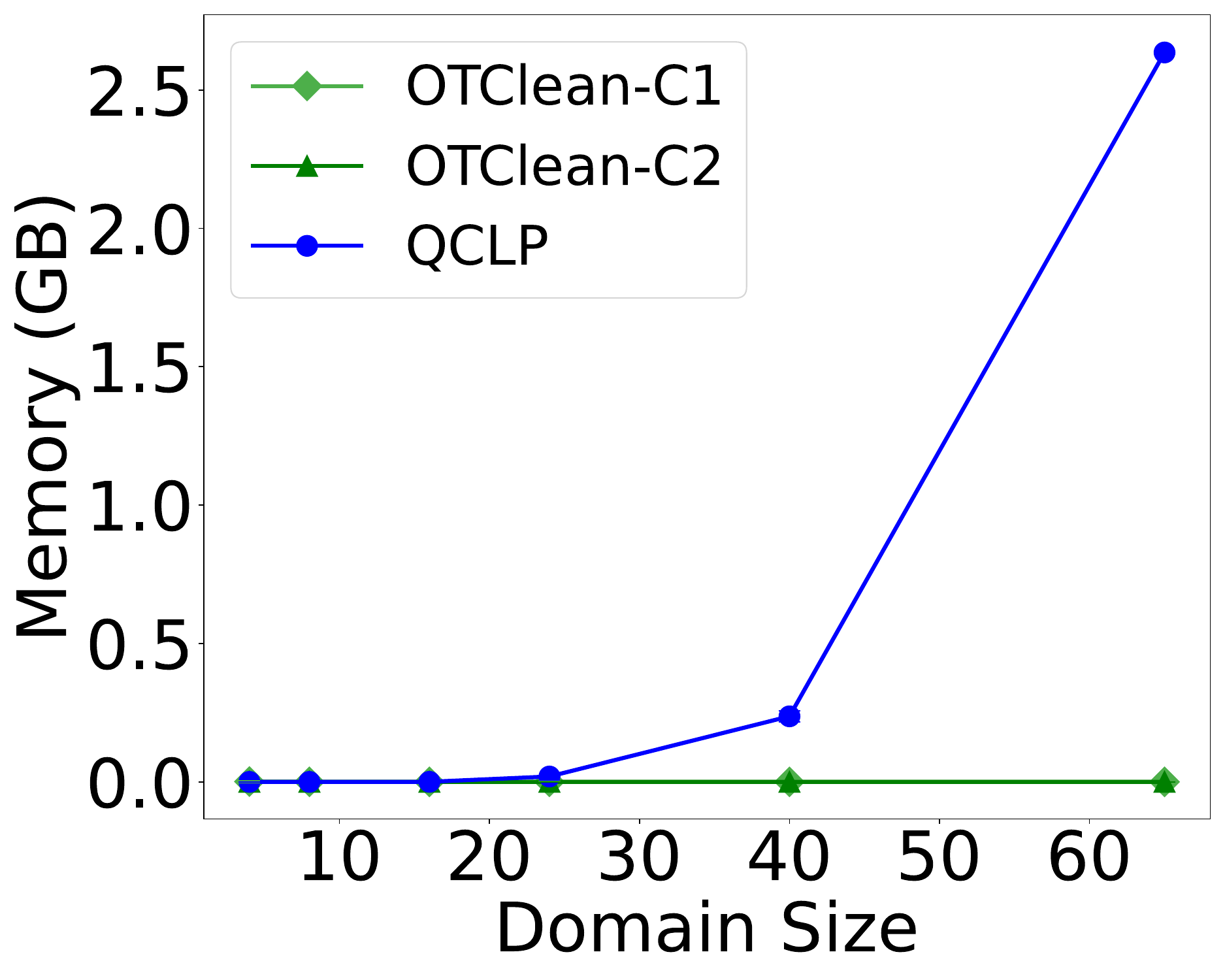}
        \caption{\compas with QCLP}
        \label{fig:memory-compas-qclp}
    \end{subfigure}
    \hspace{3mm}
    \begin{subfigure}{0.22\textwidth}
        \centering
        \includegraphics[width=\linewidth]{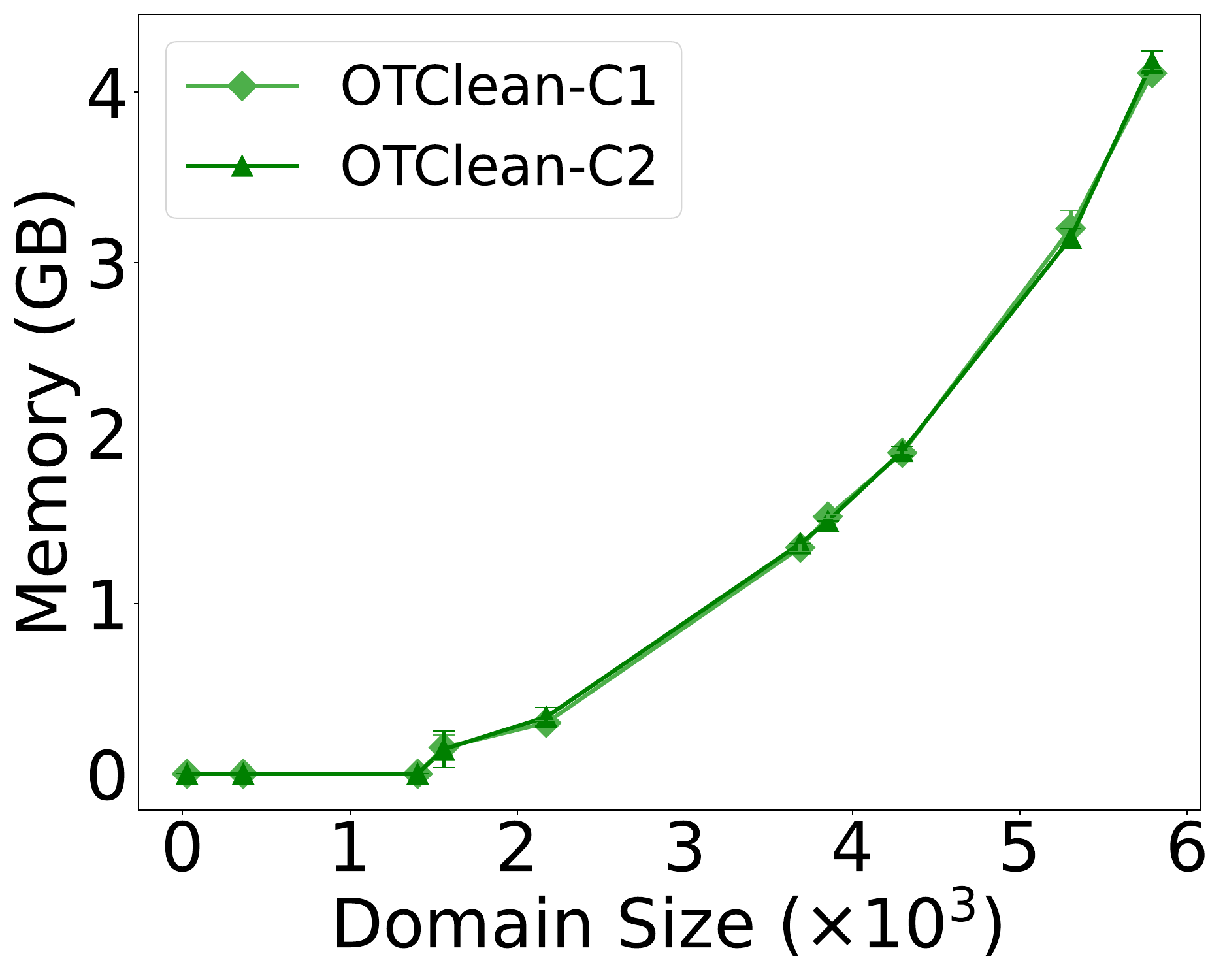}
        \caption{\compas}
        \label{fig:memory-compas}
    \end{subfigure}
    \vspace{-3mm}
    \caption{Memory Usage}
    \label{fig:memory}
\end{figure*}

\subsection{Runtime and Memory Analysis: \mainAlg vs QCLP} \label{sec:runtime-cmp} 

We conducted additional experiments to assess the performance of \mainAlg and compare it with QCLP. To assess the runtime performance and memory usage of \mainAlg and QCLP, we used the \adult and \compas datasets. The experiments are similar to those described in Section~\ref{sec:sys-exp}, which involves incrementally adding attributes to the CI constraints from these datasets to enlarge the domain size. This process allows us to analyze how increasing domain sizes affect both solutions' runtime and memory usage. The outcomes of this analysis are illustrated in Figures~\ref{fig:runtime} and~\ref{fig:memory}.

In Figure~\ref{fig:runtime-adult}, we present the runtime performance of \mainAlg on the \adult dataset across a range of domain sizes, using the two cost functions detailed in Section~\ref{sec:tunning}. We omitted QCLP from this figure due to its prohibitively high memory demands, which resulted in failure at even the smallest domain sizes for \adult. The figure demonstrates that \mainAlg efficiently manages CI constraints involving multiple attributes, even in scenarios with extensive domain sizes.

Figure~\ref{fig:runtime-compas-qclp} contrasts the runtime performance of \mainAlg and QCLP in the \compas dataset, focusing on smaller domain sizes where QCLP does not fail. For larger domain sizes, Figure~\ref{fig:runtime-compas} specifically examines the runtime changes in \mainAlg, as QCLP fails in these conditions. These findings corroborate the data presented in Table~\ref{tab:runtimes}, offering a broader perspective on how different domain sizes impact performance. Notably, the runtimes recorded in Table~\ref{tab:runtimes} for small domain sizes in \compas indicate QCLP's superior performance in these specific conditions, although \mainAlg exhibits better performance across other domain sizes, as shown in Figure~\ref{fig:runtime-compas-qclp}.

Finally, Figure~\ref{fig:memory} presents the memory usage of both \mainAlg and QCLP. Like the runtime analysis, the memory consumption was examined for both \adult and \compas datasets. In the case of QCLP, memory usage data is available only for smaller domain sizes in \compas. The key insight from these observations is that \mainAlg consistently requires less memory than QCLP, especially as the domain size increases.

\subsection{Integrating Background Knowledge}\label{sec:background}

Another set of experiments focused on understanding how considering prior background knowledge about erroneous attributes affects our data repair process, especially when dealing with attribute noise. Figure~\ref{fig:BGK} illustrates the difference in performance between models trained without knowledge of erroneous attributes (OTClean-Blind) and those trained with background knowledge (OTClean-BG). The results reveal that knowing which attributes to repair significantly improves our solution, almost matching the performance of a clean dataset. We provide results for the \boston dataset since the \car dataset already exhibited high performance with blind repair, leaving little room for improvement.

\begin{figure}[h]
    \centering
        \centering
        \includegraphics[width=0.5\linewidth]{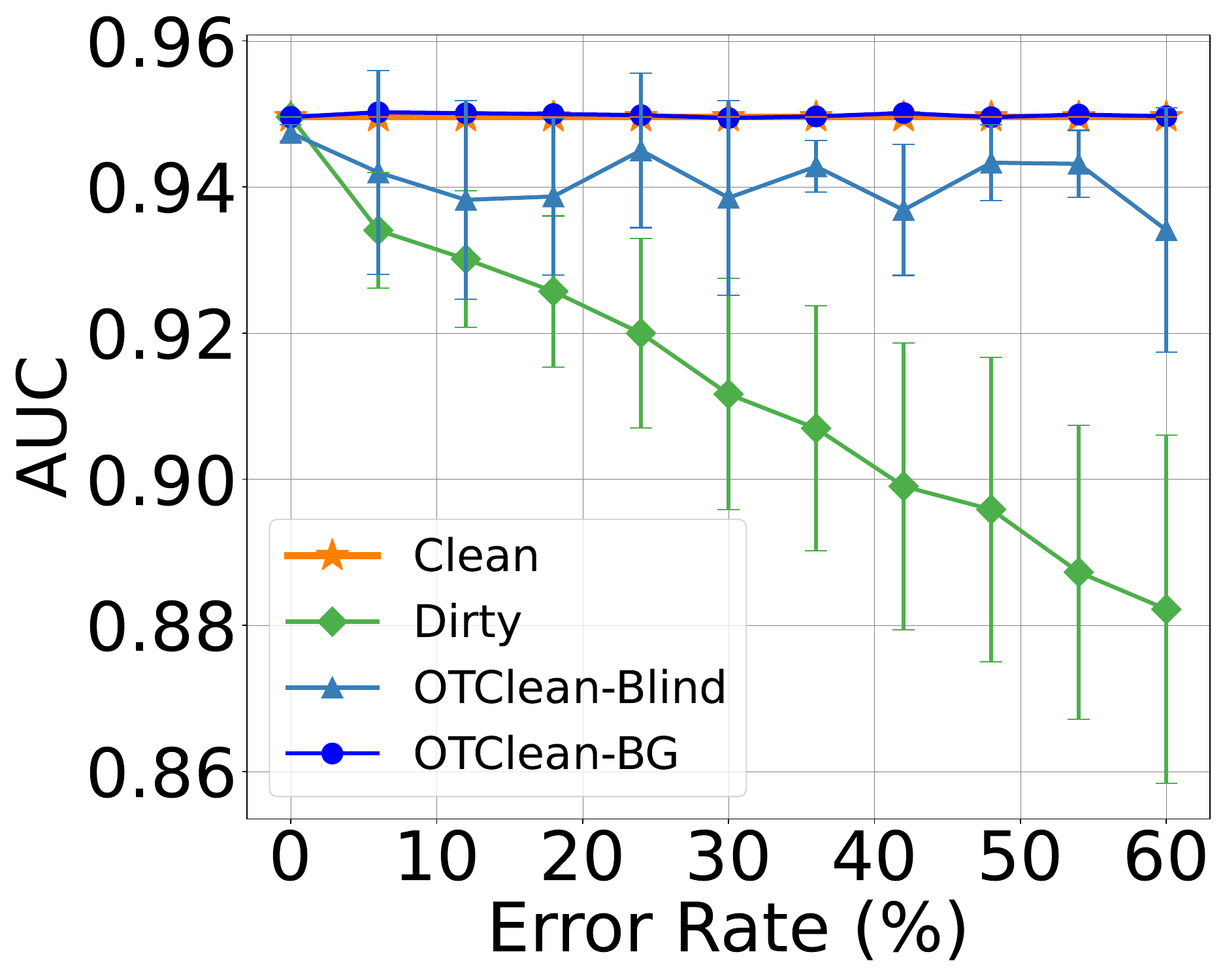}
        \label{fig:BGK}
    \caption{Blind Repair vs Repair with Background Knowledge}
    \label{fig:BGK}
\end{figure}

\subsection{Additional Results for Missing Value}

\begin{figure}[h]
    \centering
    \begin{subfigure}{0.22\textwidth}
        \centering
        \includegraphics[width=\linewidth]{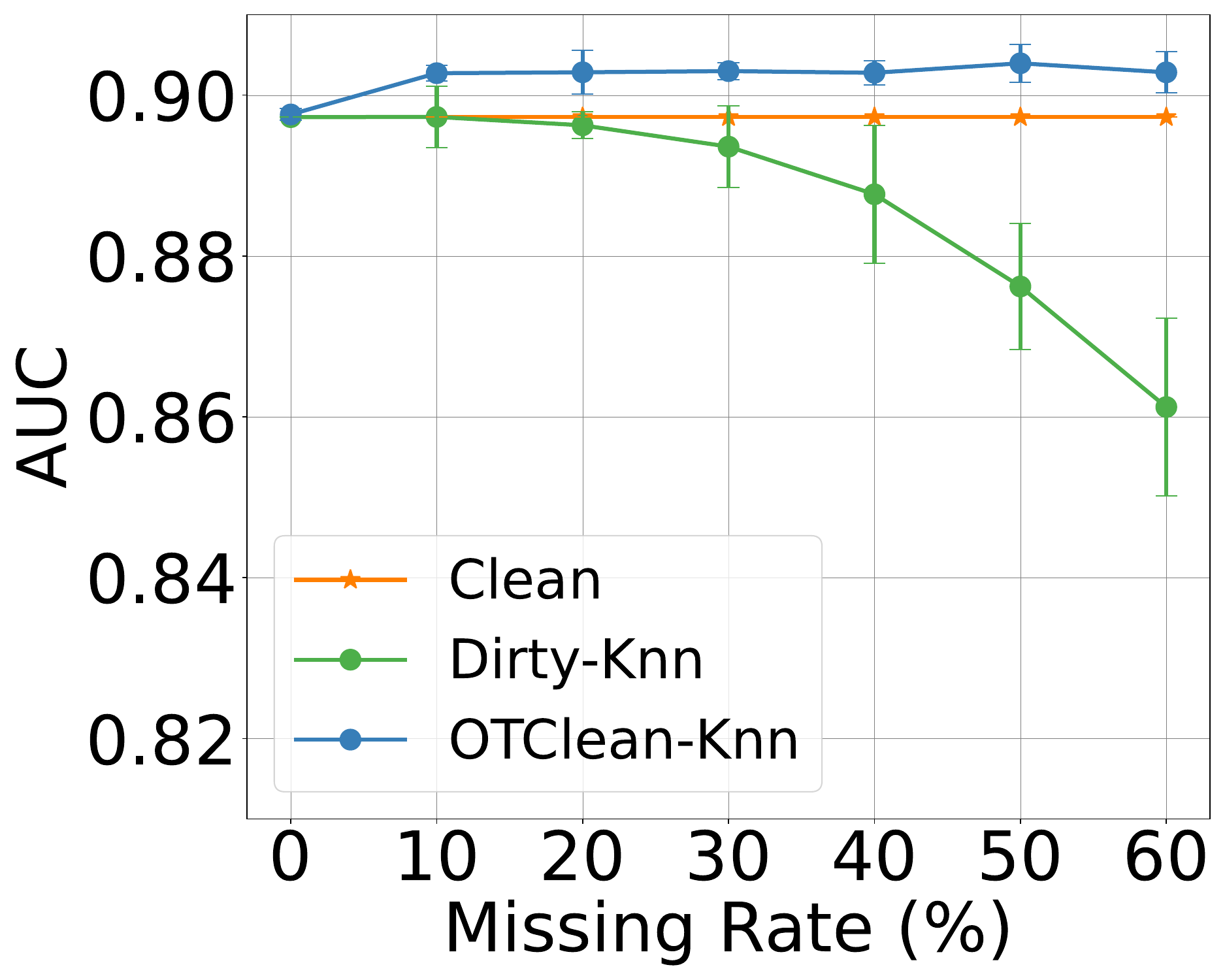}
        \caption{kNN Imputation (\boston)}
        \label{fig:knn-mnar-boston}
    \end{subfigure}
    \hfill
    \begin{subfigure}{0.22\textwidth}
        \centering
        \includegraphics[width=\linewidth]{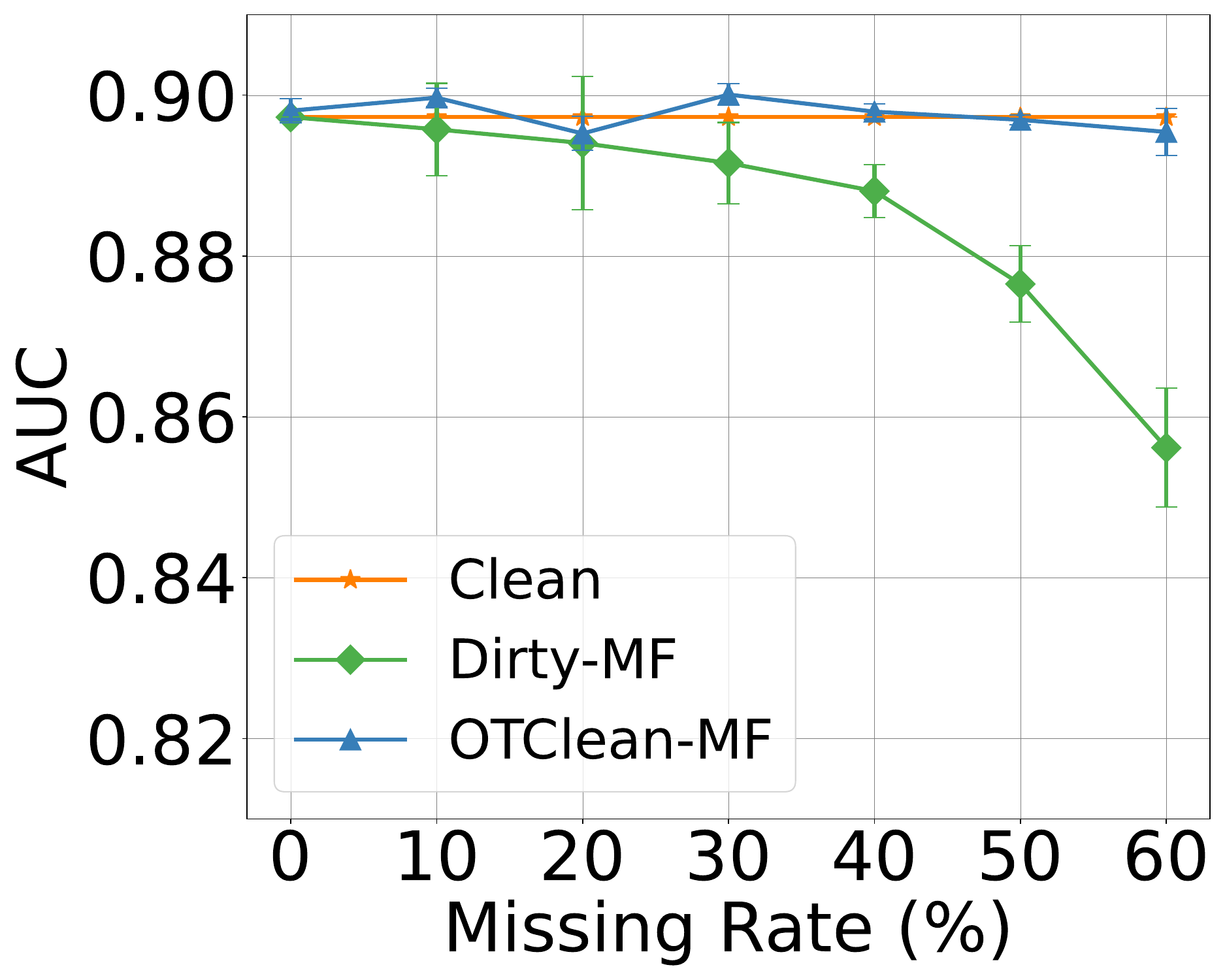}
        \caption{Most Frequent (\boston)}
        \label{fig:mf-mnar-boston}
    \end{subfigure}
    \vspace{-3mm}
    \caption{Missing Not at Random (MNAR).}
    \label{fig:mnar-extra}
\end{figure}

\begin{figure}[h]
    \centering
    \begin{subfigure}{0.22\textwidth}
        \centering
        \includegraphics[width=\linewidth]{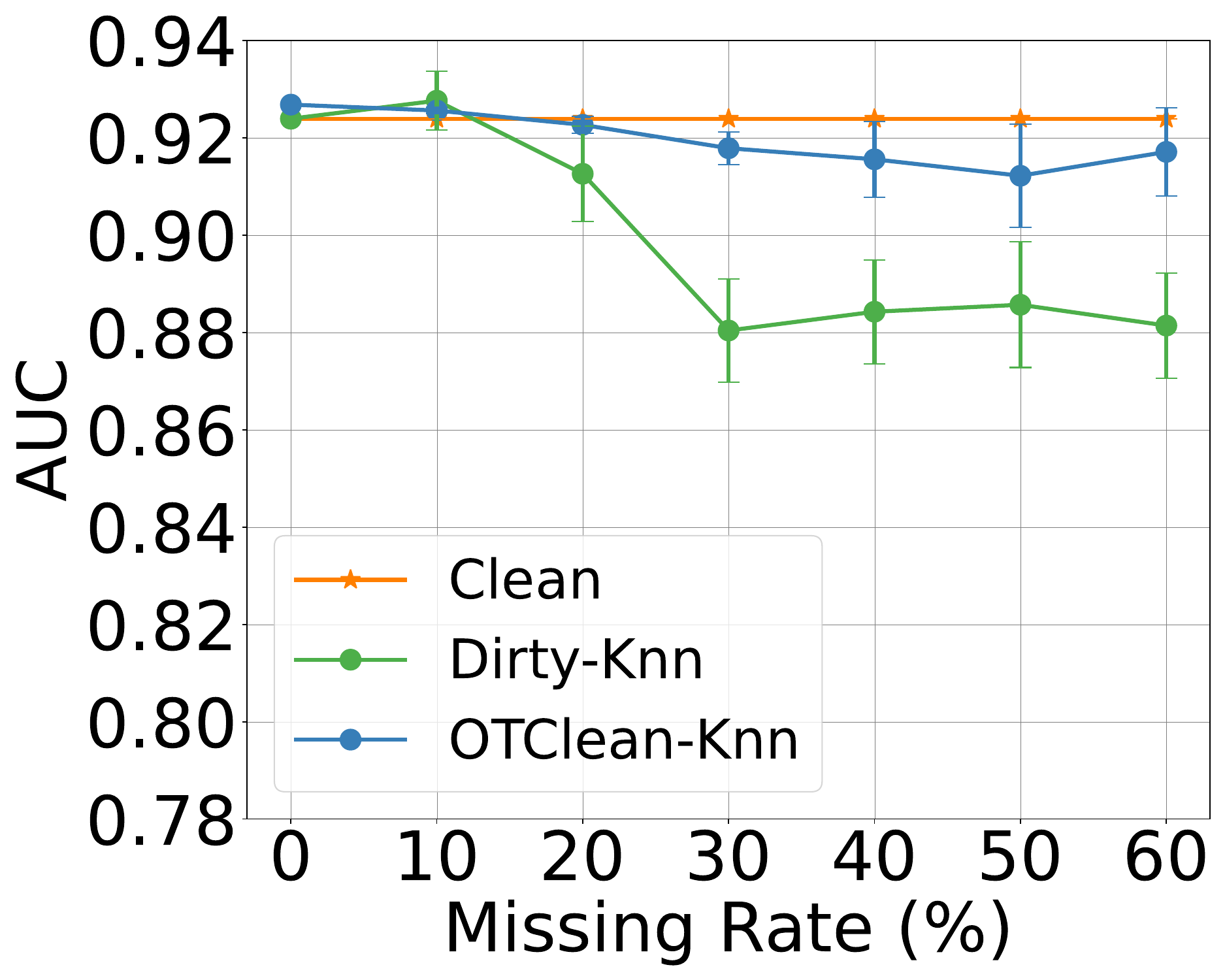}
        \caption{kNN Imputation (\car)}
        \label{fig:knn-mar-car}
    \end{subfigure}
    \hfill
    \begin{subfigure}{0.22\textwidth}
        \centering
        \includegraphics[width=\linewidth]{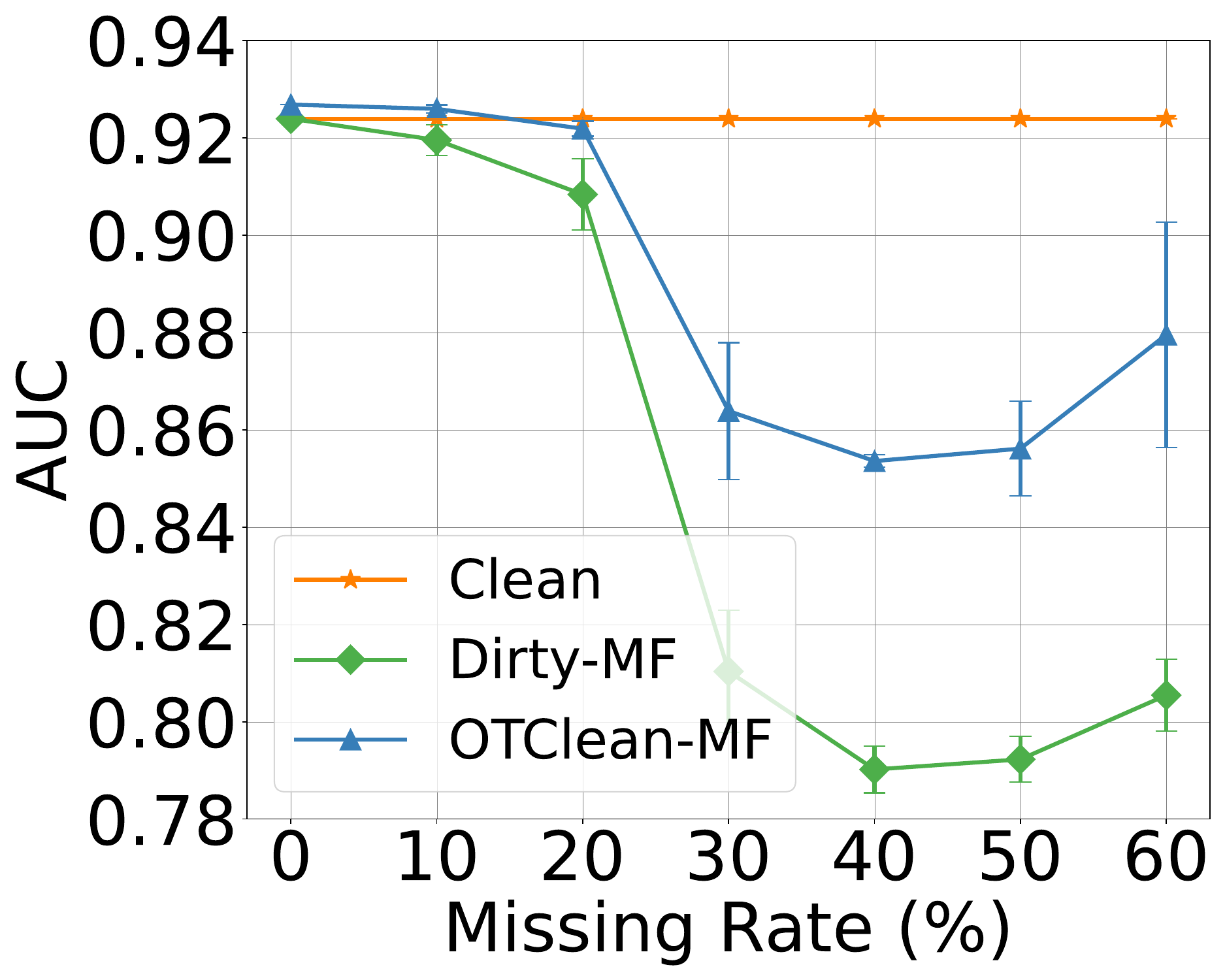}
        \caption{Most Frequent (\car)}
        \label{fig:mf-mar-car}
    \end{subfigure}
    \centering
    \caption{Missing at random (MAR)}
    \label{fig:mar-extra}
\end{figure}

We conducted extensive experiments on missing value imputation for both the \boston and \car datasets, including scenarios of MAR and MNAR. These additional experiments were not included in the main body of the paper but are presented in this section, as shown in Figures~\ref{fig:mnar-extra} and \ref{fig:mar-extra}.

The results reaffirm our earlier conclusions regarding the effectiveness of \sys in repairing data to mitigate the spurious correlations introduced by imputation methods employed to handle missing values. This is particularly more evident in Figures~\ref{fig:knn-mnar-boston}-\ref{fig:knn-mar-car}. As illustrated in Figure~\ref{fig:mf-mar-car}, the efficacy of \sys is still contingent on both the initial imputation technique used and the missing data rate. Higher rates of missing values and the use of a simplistic imputation method like naive most-frequent (MF) imputation can lead to reduced performance, even after applying \sys. Nevertheless, it is worth noting that \sys consistently delivers significant improvements compared to imputation methods without subsequent repair.